\documentclass[journal]{IEEEtran}
\ifCLASSINFOpdf
\else
\fi
\usepackage[square, comma, sort&compress, numbers]{natbib}
\usepackage{ragged2e}
\usepackage{blindtext}
\usepackage{amsmath}
\usepackage{graphicx}
\usepackage{placeins}
\usepackage{float}
\usepackage{subfigure}
\usepackage{amsfonts}
\usepackage{amssymb}
\usepackage{mathtools}
\usepackage{amsthm}
\usepackage{soul}
\usepackage[normalem]{ulem}
\usepackage{booktabs}
\usepackage{xcolor}
\usepackage{algorithm}
\usepackage{algorithmicx}
\usepackage{algpseudocode}
\usepackage{bm}
\usepackage{url}

\algnewcommand\algorithmicinput{\textbf{INPUT: }}
\algnewcommand\Input{\item[\algorithmicinput]}
\algnewcommand\algorithmicoutput{\textbf{OUTPUT: }}
\algnewcommand\Output{\item[\algorithmicoutput]}

\usepackage{multicol}

\newtheorem{proposition}{Proposition}[section]
\usepackage[numbers,sort&compress]{natbib}
\newtheorem{theorem}{Theorem}[section]
\newtheorem{definition}{Definition}
\newtheorem{lemma}[theorem]{Lemma}
\newtheorem*{result}{Result}
 
\begin{document}
%
\title{Probabilistic Learning of Multivariate Time Series with Temporal Irregularity}
%
%
%

\author{Yijun LI,
    Cheuk Hang LEUNG,
    Qi WU
    \IEEEcompsocitemizethanks{
        \IEEEcompsocthanksitem The first two authors contributed equally.
        \IEEEcompsocthanksitem Corresponding author: Qi WU (qiwu55@cityu.edu.hk)
        \IEEEcompsocthanksitem The authors are with the Department of Data Science, City University of Hong Kong, Hong Kong.}
    } 
   
%
%

\markboth{Journal of \LaTeX  CLASS FILES, VOL. 14, NO. 8, AUGUST 2015}
{Shell \MakeLowercase{\textit{et al.}}: Bare Demo of IEEEtran.cls for IEEE Journals}
%



\maketitle

\begin{abstract}
   Probabilistic forecasting of multivariate time series is essential for various downstream tasks. Most existing approaches rely on the sequences being uniformly spaced and aligned across all variables. However, real-world multivariate time series often suffer from temporal irregularities, including nonuniform intervals and misaligned variables, which pose significant challenges for accurate forecasting. To address these challenges, we propose an end-to-end framework that models temporal irregularities while capturing the joint distribution of variables at arbitrary continuous-time points. Specifically, we introduce a dynamic conditional continuous normalizing flow to model data distributions in a non-parametric manner, accommodating the complex, non-Gaussian characteristics commonly found in real-world datasets. Then, by leveraging a carefully factorized log-likelihood objective, our approach captures both temporal and cross-sectional dependencies efficiently. Extensive experiments on a range of real-world datasets demonstrate the superiority and adaptability of our method compared to existing approaches. The data and code supporting this work are available at \url{https://github.com/lyjsilence/RFN}.
\end{abstract}


%
\IEEEpeerreviewmaketitle

\section{Introduction}
%
%
%
%

\IEEEPARstart{M}{ultivariate} time series (MTS) data, where multiple variables are recorded and evolve over time, are essential across fields such as healthcare, finance, and climate science. For instance, in healthcare, MTS data can track a patient’s vital signs, enabling early detection of potential health issues. In finance, it supports forecasting market behaviors and asset correlations, which is essential for risk management and investment strategies. Probabilistic forecasting of MTS is indispensable, as it supports downstream tasks like anomaly detection, risk assessment, and decision-making. This involves predicting not only central measures like the mean or median but also quantiles and confidence intervals that are essential in high-stakes decision-making scenarios \citep{ban2018machine}.

At the variable level, MTS data exhibit serial dependence, meaning the marginal distributions of individual variables at different time points are not independent \citep{zhang2021memory}. At the group level, the dependence structure encompasses relationships like the copula function of the joint distribution, where components may be strongly interdependent. Both variable-level serial dependencies and group-level interdependencies can vary over time. Probabilistic forecasting entails specifying the joint distribution of variables in MTS, learning its representation from data, and predicting its evolution over time \citep{benidis2018deep}. 

However, real-world MTS data rarely follow regular sampling intervals. Instead, these datasets often exhibit temporal irregularities, creating multivariate irregular time series where observations are separated by uneven time intervals. Temporal irregularities occur frequently in various applications. For example, climate data might have gaps if monitoring equipment fails to record at scheduled times. Patient measurements may be taken at inconsistent times, resulting in irregular sampling. 

These irregularities, which may be random or due to specific factors, present significant challenges for probabilistic forecasting in MTS \citep{kremer2010detecting, gao2022explainable}. Irregular sampling disrupts the modeling of serial dependencies, as traditional methods rely on evenly spaced data points to capture temporal patterns accurately. Additionally, it complicates the modeling of interdependencies among variables, as unaligned time points make it harder to reveal important relationships within the data. This irregularity increases the difficulty of accurately characterizing the joint distribution, adding uncertainty to the forecasting process \citep{shukla2020survey}. Before presenting our approach, we review existing methods for handling temporal irregularity and modeling joint distributions.

\subsection{On Handling Temporal Irregularity} \label{subsec: On Handling Irregular Sampling}
There are three primary approaches to handling irregular sampling. The first approach involves converting an irregularly sampled time series into one with evenly spaced time intervals before making predictions. The discretization method selects a larger uniformly spaced time interval, with each interval containing several observations, and computes the mean of these observations to represent the interval's value. However, this method loses local information due to averaging. In contrast, the imputation method interpolates the missing values of the lower-frequency variables instead of averaging the higher-frequency variables. It keeps the local information of the higher-frequency variables intact and uses models such as the Gaussian Process regression model \cite{williams2006gaussian}, the Recurrent Neural Networks (RNNs) \cite{cao2018brits}, and the Generative Adversarial Networks (GANs) \cite{luo2018multivariate} to impute the missing values of the lower-frequency components.

The next approach proposes using end-to-end models to avoid the ``interpolate first and predict later'' idea. This approach modifies classical recurrent architectures to encode the information embedded in irregular temporal patterns. For example, Che et al. \cite{che2018recurrent} added an exponential decay mechanism in the hidden state. Neil et al. \cite{neil2016phased} extended the Long short-term memory (LSTM) unit by adding a new time gate controlled by a parametrized oscillation with a frequency range. Additionally, Mozer et al. \cite{mozer2017discrete} incorporated multiple time scales of the hidden state and made a context-dependent selection of time scales for information storage and retrieval. 

At last, Chen et al. \cite{chen2018neural} introduced the Neural ODE framework by extending discrete neural networks into continuous-time networks, which makes it a natural candidate for handling data with arbitrarily small time intervals. For instance,  Rubanova et al. \cite{rubanova2019latent} proposed the Latent ODE and were the first to embed Neural ODEs in a Variational Autoencoder \cite{DBLP:journals/corr/KingmaW13} to address the problem of irregularly sampled time series. De Brouwer et al. \cite{de2019gru} integrated the Neural ODEs in the classical Gated Recurrent Unit (GRU) cell and derived the dynamics of the hidden state. Unlike the classical GRU cell, which keeps the hidden state constant in the absence of observations, the continuous-time GRU cell learns to evolve the hidden state using Neural ODEs. Building upon similar ideas, Lechner et al. \cite{lechner2020learning} transformed the standard LSTM into a continuous version to address the issues of gradient vanishing and exploding. In section \ref{subsec: neural ode}, we will provide more background on Neural ODEs and how to utilize them to model irregularly sampled data.

\subsection{On Modeling Joint Data Distribution} \label{subsec: On Modeling the Joint Distribution}
For point estimation tasks, vanilla recurrent architectures, including RNN, GRU, and LSTM, can capture different aspects of the aforementioned properties of the MTS data. However, they are not directly applicable to distribution prediction due to the deterministic nature of the transition functions of their hidden states, which do not account for modeling uncertainties \citep{chung2015recurrent}. 

One class of models modifies the output function of neural networks to model the joint distribution or quantile function. For example, the models in \citep{salinas2020deepar, de2019gru, salinas2019high} assume the data-generating process follows parametric distribution, such as the multivariate Gaussian (for continuous variables) and multivariate Negative Binomial (for discrete variables).  Alternatively, researchers use quantile regression to fit the quantile function of the joint distribution. They use the quantile loss \citep{wen2017multi,yan2019cross} or the continuous ranked probability score \citep{laio2007verification} as the objective function to train the model and predict multiple quantile points simultaneously conditional on the hidden states. 

Recently, unsupervised deep learning models have been utilized to learn the joint distribution of data, including integrating variational autoencoders  \citep{chung2015recurrent}, normalizing flows \citep{rasul2020multivariate, feng2023multi}, or diffusion models \citep{rasul2021autoregressive} into RNNs. Among these, the flow models are flexible in capturing intricate and evolving dependence structures and impose no assumptions about the functional form of the joint data distribution. These characteristics make them attractive for dealing with complex data, although they do not specifically address the structural aspects of irregular sampling. In section \ref{subsec:flow}, we shall further detail the background of representing data distribution using the normalizing flow approach.

\subsection{Our Approach and Contributions}
Discussions in \ref{subsec: On Handling Irregular Sampling} and \ref{subsec: On Modeling the Joint Distribution} unveil dislocations and disparities among ideas of handling temporal irregularities in the data and ideas to model its joint distribution. This paper bridges this gap by introducing a deep learning solution called the \textit{Recurrent Flow Network} (RFN). It can seamlessly integrate the treatment of temporal irregularities with the learning of joint data distribution. Its novelties are as follows.

{\textbf{(i)}} The proposed RFN framework formulates a two-layer representation that distinguishes marginal learning of variable dynamics from multivariate learning of joint data distribution. It is a versatile methodology that can be trained end-to-end and accommodates synchronous and asynchronous data structures. It also broadly applies to underlying recurrent architectures. Once the joint data distribution is learned, it is ready for sampling despite the distribution being non-parametrically represented via neural networks.

{\textbf{(ii)}} The joint learning layer resolves the struggle faced by existing models \citep{salinas2020deepar, de2019gru, salinas2019high} in achieving a non-parametric representation of non-Gaussian data distribution, simultaneously with a flexible choice of information to generate time variation. The conditional CNF (Continuous Normalizing Flow) representation we developed enables one to choose what information to use to drive the time variation of the base distribution and the flow map without compromising any non-parametric capacity to represent the non-Gaussian data distribution.

{\textbf{(iii)}} Building upon (ii), we strategically condition the log-likelihood objective on the observation times. This conditioning structure enables the optimizer to fully acknowledge and account for both the uneven spacing aspect and the asynchrony aspect of temporal irregularity in the MTS data. By conditioning on the observation times, the RFN ensures that the model incorporates the specific time points at which the data is observed. 

We validate the novelties mentioned above through synthetic experiments and demonstrate the overall performance of the RFNs on three real-world datasets. The synthetic studies simulate sample paths of a multivariate correlated Geometric Brownian Motion process to verify the ability of our approach to capture the conditional joint distribution. Meanwhile, the experimental datasets include the physical activities of the human body from the MuJoCo module \cite{rubanova2019latent}, the climate records of weather from the USHCN dataset \cite{menne2010long}, and the minute-level transaction records of eight stocks in the biotechnology sector of NASDAQ market \cite{qin2017dual}. We compare four baseline models in terms of their performance in vanilla forms and the performances utilizing the RFN specification. The results show that the RFN framework has broad applicability and significantly improves existing approaches.

\section{Backgrounds} \label{sec: background}
Understanding the working mechanism of the proposed RFN model requires the knowledge of Neural ODEs, their applications to model irregularly sampled data, and the flow representation of distributions. This section summarizes these subjects to make the paper self-contained. Throughout the paper, we denote random variables as follows: $X$ for scalar, $\mathrm{X}$ for vector, and $\mathbf{X}$ for matrix. Their corresponding sample values are denoted as $x$, $\mathrm{x}$, and $\mathbf{x}$ accordingly.

\subsection{Neural Ordinary Differential Equations}\label{subsec: neural ode}
Neural ODEs were developed as the continuous limit of the ResNet model. The ResNet model solves the degradation problem of neural networks where researchers noticed that, as the network layers go deeper, the training loss begins to increase steadily once the network depth crosses a certain threshold \cite{he2016deep}. Consider a $L$-layer network, with $\mathrm{x}_0$ being the input, $\mathrm{x}_l$ being the output of each layer $l\in \{1, \cdots, L\}$, and $f_{\theta_{l}}(\cdot)$ being the learning functions of layer $l$. Instead of learning the mapping from $\mathrm{x}_0$ to $\mathrm{x}_L$ directly, ResNet learns the difference between the input and output of each layer:
\begin{equation}\label{Eq1}
\mathrm{x}_l=\mathrm{x}_{l-1}+\mathrm{f}_{\theta_l}(\mathrm{x}_{l-1}). 
\end{equation} 

Chen et al. \cite{chen2018neural} proposed that taking the limit of the number of layers to infinity shall turn discrete layers into continuous layers. The resulting continuous limit of the recursive equation \eqref{Eq1} is an ODE:
\begin{equation}
\frac{d\mathrm{x}(l)}{dl}=\mathrm{f}_{\theta}(\mathrm{x}(l), l). \label{Eq2}
\end{equation} 
Solving \eqref{Eq2} with initial condition $\mathrm{x}(0)$ is equivalent to the forward pass of ResNet. One can use numerical methods such as the Euler and the Runge–Kutta methods to solve \eqref{Eq2}. 

\subsection{Unconditional Normalizing Flow}\label{subsec:flow}
For tasks related to probabilistic forecasts, one needs a representation of the data distribution. Let $p(\mathrm{x})$, $\mathrm{x}\in \mathbb{R}^{D}$ be the probability density of the data-generating distribution and $p(\mathrm{z})$, $\mathrm{z}$ $\in \mathbb{R}^{D}$ be the probability density of the base distribution which is typically set as the standard normal, i.e., $\mathrm{Z}\sim \mathcal{N}(\mathrm{0}, \mathbb{I}_D)$. The idea of the normalizing flow model is to find a differentiable bijective function $\mathrm{f}=[f^{1},\cdots,f^{D}]^{\top}$ which can map samples from $\mathrm{Z}$ to $\mathrm{X}$ \cite{dinh2014nice, rezende2015variational}:
\begin{align*}
\mathrm{f}: \mathbb{R}^D \rightarrow \mathbb{R}^D; \quad \mathrm{f}(\mathrm{z})=\mathrm{x}.
\end{align*} 
In the discrete formulation, $\mathrm{f}$ is typically specified as a sequence of neural networks, $\mathrm{f} = \mathrm{f}_1 \circ \cdots \circ \mathrm{f}_{M-1} \circ \mathrm{f}_{M}$. However, designing the architectures of $\mathrm{f}_1,\cdots,\mathrm{f}_M$ is challenging because they need to satisfy three conditions: being bijective, differentiable, and facilitating the computation of the determinant of the Jacobian of the function $\mathrm{f}$. 

The continuous normalizing flow model (CNF) \cite{grathwohl2018ffjord} offers a solution to this challenge by extending the composition of discrete maps into a continuous map, whose differential form reads as follows:
\begin{subequations}\label{Eq5}
\begin{equation}
\begin{aligned}\label{eqt:ffjord1}
&\frac{\partial \mathrm{z}(s)}{\partial s}=\mathrm{f}(\mathrm{z}(s), s ; \theta),\quad s \in [s_0, s_1], \\
&\text{where} \quad \mathrm{z}(s)|_{s=s_0} = \mathrm{z} , \,\, \mathrm{z}(s)|_{s=s_1} = \mathrm{x}.
\end{aligned}
\end{equation}
Unlike the physical time $t$, $s$ is called the \textit{flow time} of the dynamics \eqref{eqt:ffjord1}. At the initial flow time $s_0$, the value of the flow $\mathrm{z}(s_0)$ is set as $\mathrm{z}$, which samples from the base distribution $p(\mathrm{z})$ of the base random variable $\mathrm{Z}$. At the terminal flow time $s_1$, the value of the flow $\mathrm{z}(s_1)$ is set to equal $\mathrm{x}$, which is the observed sample from the distribution of the true data-generating distribution $p(\mathrm{x})$. 

The discrete formulation requires careful design of the weight matrices of $\mathrm{f}_{j},$ {\small$1\le j\le M$}, to be triangular to facilitate computing the Jacobian's determinant easily. However, the computation of the Jacobian determinant is replaced with relatively cheap trace operations thanks to the Instantaneous Change of Variables theorem \cite{chen2018neural} in the continuous formulation. Consequently, the log-density of the continuous flow follows the following equation:
\begin{equation}
\begin{aligned}\label{eqt:ffjord2}
&\frac{\partial \log p(\mathrm{z}(s))}{\partial s} =-\operatorname{Tr}\left[\partial_{\mathrm{z}(s)} \mathrm{f}\right],\quad s \in [s_0, s_1], \\
& \begin{aligned}
\text{where}\quad p(\mathrm{z}(s))|_{s=s_0} = p(\mathrm{z}), \ 
p(\mathrm{z}(s))|_{s=s_1}  = p(\mathrm{x}).
\end{aligned}
\end{aligned}
\end{equation}
Solving equations \eqref{eqt:ffjord1} and \eqref{eqt:ffjord2} together, we have
\begin{equation}\tag{\ref{Eq5}}\label{eqt:ffjord}
\begin{bmatrix}
\mathrm{x} \\
\log p(\mathrm{x})
\end{bmatrix}=
\begin{bmatrix}
\mathrm{z}\\
\log p(\mathrm{z})
\end{bmatrix}+\int_{s_0}^{s_1}
\begin{bmatrix}
\mathrm{f}(\mathrm{z}(s), s ; \theta) \\
-\operatorname{Tr}\left[\partial_{\mathrm{z}(s)} \mathrm{f}\right]
\end{bmatrix}ds.
\end{equation}
\end{subequations}

\section{Data Structure \& Problem Statement}
MTS are sequences of data where multiple variables are observed over time. Each variable may exhibit dependencies on both its own past values and the past values of other variables. Consider a MTS dataset containing {\small$N$} instances. Each instance is a {\small$D$}-dimensional sample path. All instances span the same $[0, T]$ period. For example, this dataset could represent climate recordings, consisting of $N$ daily records, each spanning 24 ($T$) hours. During each day, multiple indices ($D$), such as temperature and precipitation, are observed. However, some recorded values may be missing due to various factors, such as equipment failure.

To account for the presence of temporal irregularity in a given instance $i\in \{1,\cdots, N\}$, we first collect all time points at which at least one variable has an observation and define this collection as the time vector of observations:
\begin{equation*}
\begin{aligned}
\mathrm{t}^i :=[t_{1}^i,\cdots, t_{K_i}^i], \quad 0 \leq t_1^i \leq \cdots \leq t_{K_i}^i \leq T.
\end{aligned}
\end{equation*}

\begin{figure}[t]
	\centering
	\subfigure[Syn-MTS]{
		\includegraphics[width=.45\linewidth]{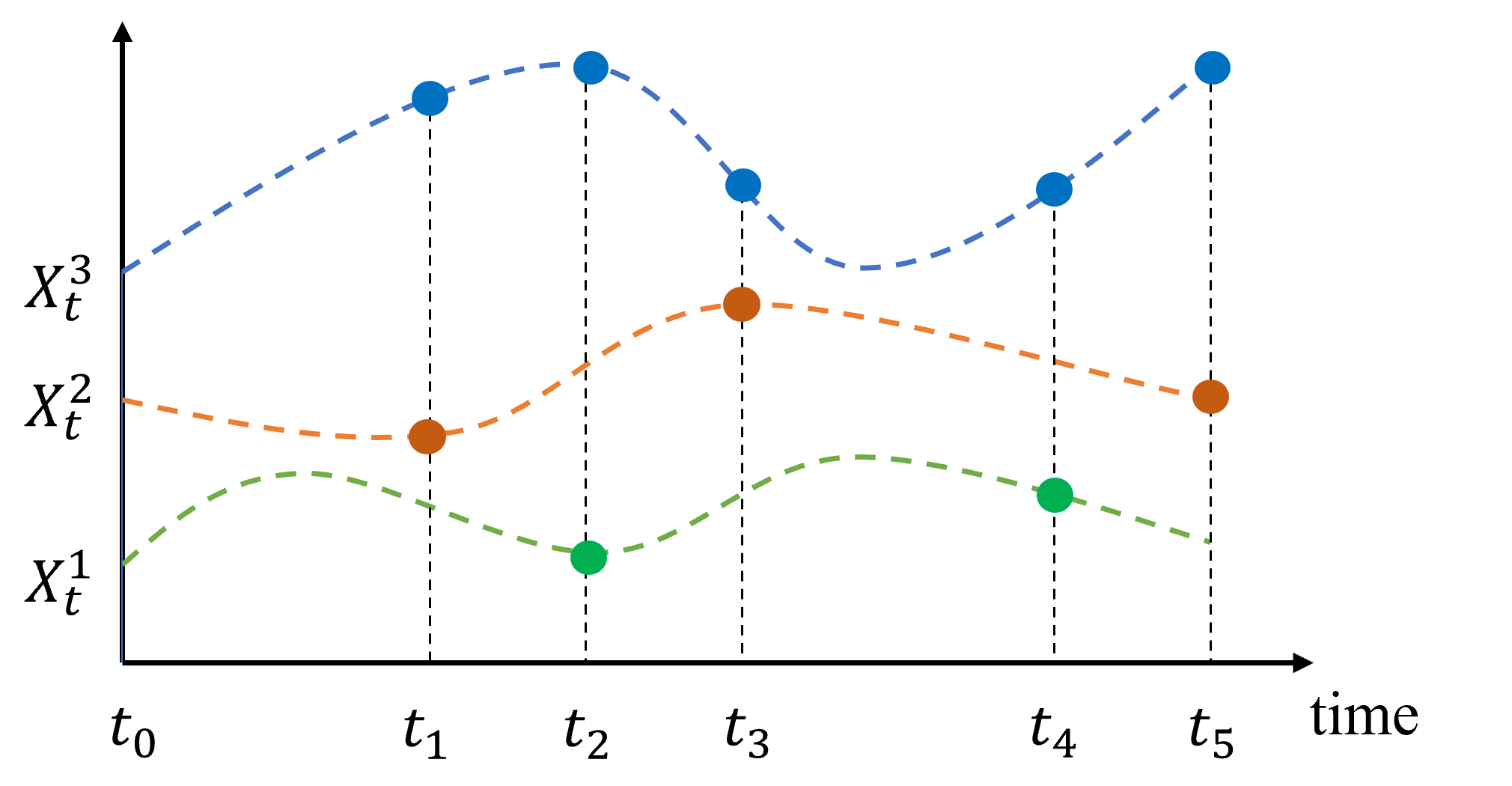} 
		\label{sync}
	}%
	\subfigure[Asyn-MTS]{
		\includegraphics[width=.45\linewidth]{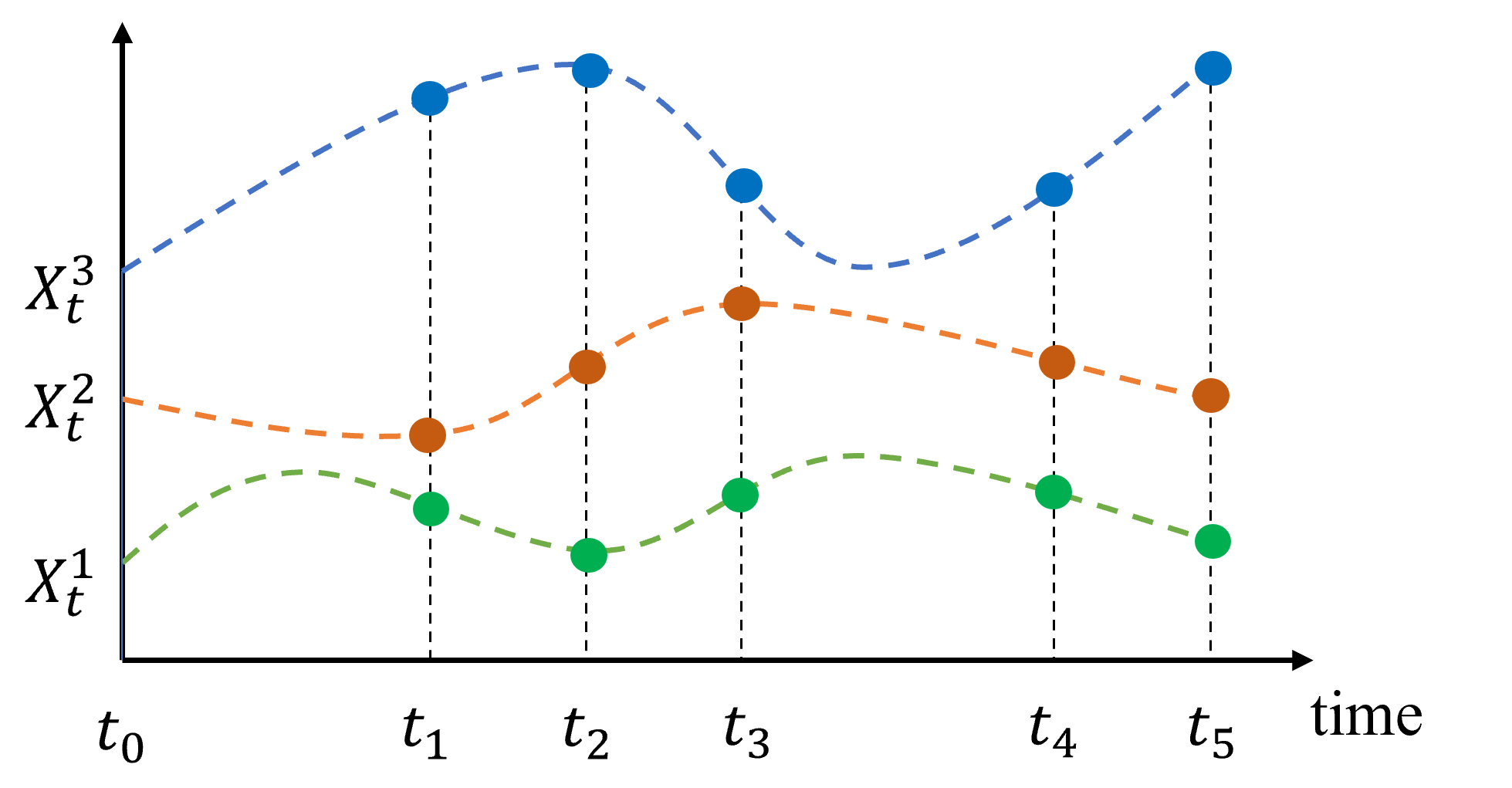}
		\label{async}
	}%
	\caption{(a) and (b) are examples of Syn-MTS and Asyn-MTS where observed data points are marked as solid circle dots. While the time intervals between consecutive observation times are unevenly spaced in both cases, component observations of the Syn-MTS sample path are always aligned. In contrast, in the Asyn-MTS case, no observation time has complete observations. This demonstrates that uneven spacing originates at the univariate level, while asynchrony arises exclusively in the multivariate context.}
\end{figure}

At a particular observation time $t\in \mathrm{t}^i$, we use $\mathrm{x}_{i,t}\in \mathbb{R}^{D\times 1}$ to denote the time-$t$ sample values of the random vector $\mathrm{X}_{t}$ of the $i^{\textrm{th}}$ instance, within which we use $x_{i,t}^d\in \mathbb{R}$ to denote the $d^{\textrm{th}}$ component, which is the sample value of the $d^{\textrm{th}}$ random scalar $X^d_{t}$, We also set $x_{i,t}^{d}=0$ if no observation for the $d^{\textrm{th}}$ at time $t$.  Thus, we have
\begin{equation*}
\begin{aligned}
\mathrm{x}_{i,t} := [x_{i,t}^1, \cdots, x_{i,t}^d, \cdots, x_{i,t}^D]^{\top}, \quad \text{where}\;t\in \mathrm{t}^i.
\end{aligned}
\end{equation*}
We then aggregate $\mathrm{x}_{i,t}$ from all observation times $t=t^i_1,t^i_2,\cdots, t^i_{K_i}$ of the instance $i$ to form the instance-level data matrix $\mathbf{x}_i\in \mathbb{R}^{D\times K_i}$, 
where
\begin{equation*}
\begin{aligned}
\mathbf{x}_{i} := [\mathrm{x}_{i,t_1^i}; \mathrm{x}_{i,t_{2}^i}; \cdots ;\mathrm{x}_{i,t_{K_i}^i}]^{\top}, \quad \;i\in [1,\cdots,N].
\end{aligned}
\end{equation*}
Finally, the entire MTS dataset is the collection of all $N$ instances $\{ \{\mathbf{x}_1\}; \{\mathbf{x}_2\}; \cdots; \{\mathbf{x}_N\}\}$. 

An instance $\mathbf{x}_i$ can be synchronous or asynchronous depending on whether $x^d_{i,t}$ is observed or not for any combination of the variable dimension $d \in \{1,\cdots, D\}$ and the observation time $t\in \mathrm{t}^i$. 
\begin{definition}
{\textbf{Synchronous multivariate time series (Syn-MTS)}}: An instance $\mathbf{x}_i$ where all of its $D$ component series have observations at each and every time points $t \in \mathrm{t}^i$ (see Fig. \ref{sync}). 
\end{definition}
\begin{definition}
{\textbf{Asynchronous multivariate time series (Asyn-MTS)}}: An instance $\mathbf{x}_i$ where at least one of its $D$ component series does not have all observations at all time point $t \in \mathrm{t}^i$ (see Fig. \ref{async}). 
\end{definition}

To precisely distinguish between Asyn-MTS and Syn-MTS, we can use the mask matrix. For each instance $i$ with corresponding data matrix $\mathbf{x}_i$, its mask matrix is $\mathbf{m}_{i}$ such that
\begin{equation*}
\begin{aligned}
\mathbf{m}_{i} := [\mathrm{m}_{i,t_1^i}; \mathrm{m}_{i,t_2^i}; \cdots; \mathrm{m}_{i,t_{K_i}^i}]^{\top}.
\end{aligned}
\end{equation*}
Here, $\mathrm{m}_{i,t}$ is a vector that denotes whether the constituent component variables are observed at time $t \in \mathrm{t}^i$, i.e., 
\begin{equation*}
\begin{aligned}
\mathrm{m}_{i,t} :&=[m^1_{i,t},\cdots,m^d_{i,t}, \cdots, m^D_{i,t}]^{\top},\quad \text{where}\\
m^d_{i,t}&=\begin{cases}
1,  \quad &\text{if } x^d_{i,t} \text{ is observed}; \\
0,  \quad &\text{if } x^d_{i,t} \text{ is unobserved}.
\end{cases}
\end{aligned}
\end{equation*}
In the sequel, we shall drop the instance script $i$ to lighten notations, e.g., $t^i_k$, $\mathrm{x}_{i,t}$ and $\mathrm{m}_{i,t}$ shall become $t_k$, $\mathrm{x}_{t}$ and $\mathrm{m}_{t}$, whenever the context is clear. In Fig. \ref{fig:data structure}, we give a plot explanation of an instance with two variables.

\begin{figure}[t]
    \centering
    \includegraphics[width=\linewidth]{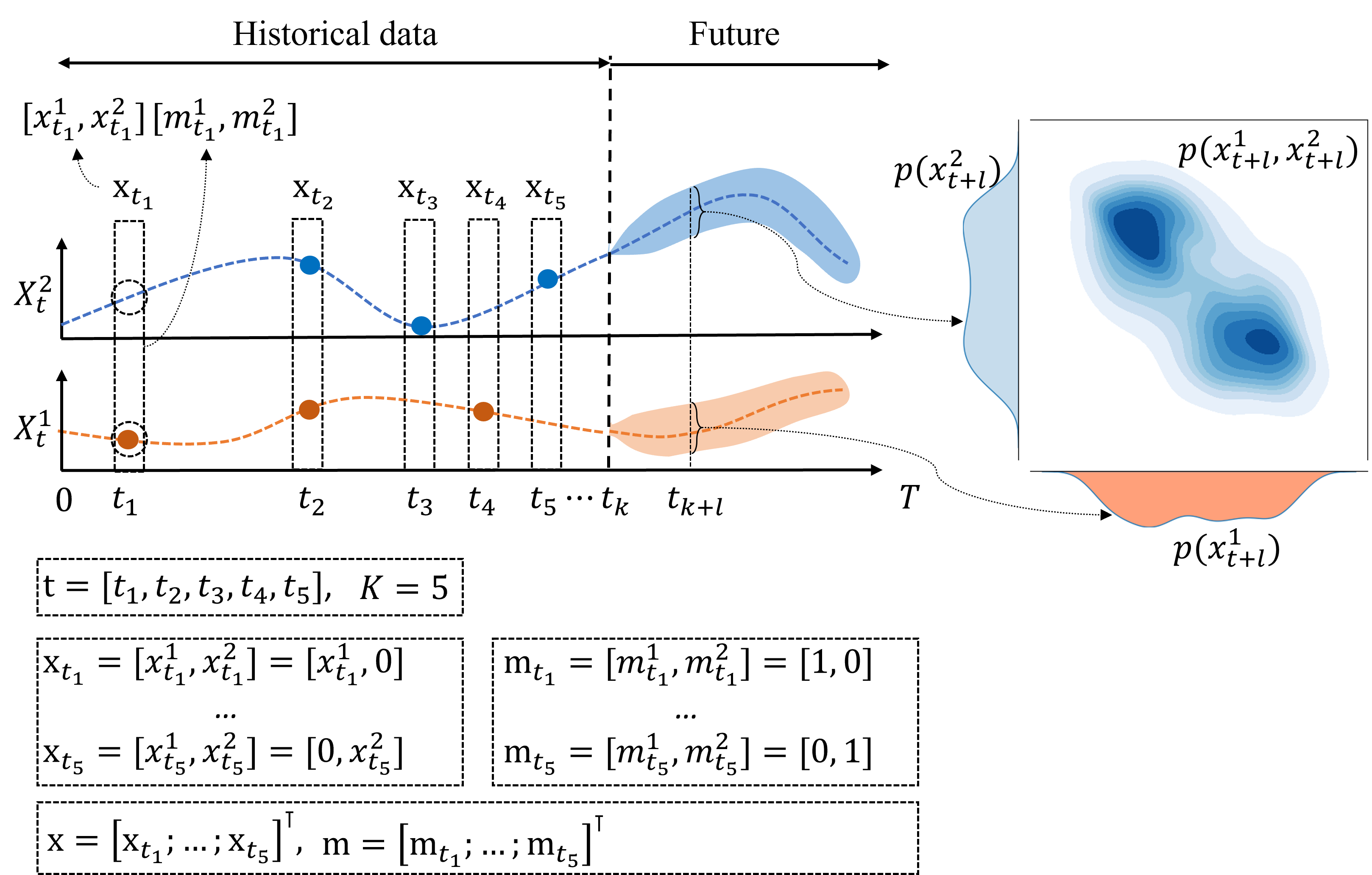}
    \caption{The data structure and notations for one instance.}
    \label{fig:data structure}
\end{figure}
\textbf{Problem Statement.} Given the historical value of Syn-MTS or Asyn-MTS observations $\mathrm{x}_{t_1}, \cdots, \mathrm{x}_{t_k}$, the objective of our work is to learn a non-linear mapping function to estimate the joint distribution \(p(\mathrm{x}_{t_{k+l}})\) at arbitrary future time $t_{k+l}$ in a continuous-time manner. In particular, the joint distribution can be decomposed into marginal distributions \(p(x^1_{t_{k+l}}), \cdots, p(x^D_{t_{k+l}})\), which allows for interval estimation of each individual variable. In Fig. \ref{fig:data structure}, we provide a graphical representation of the goal we aim to achieve, which demonstrates the probabilistic forecasting of two irregularly sampled and non-Gaussian variables over time.


\section{Model Framework}
In this section, we introduce our RFN model framework, designed for probabilistic forecasting of multivariate irregular time series in both synchronous and asynchronous cases. A visual representation is provided in Figure \ref{fig:RFN-GRUODE}. 

The RFN comprises two main layers: the marginal learning layer and the joint learning layer. In the marginal learning layer, the multivariate irregular time series is processed using any advanced sequential model suitable for such data. The resulting representation, the hidden state of sequential models, is then passed to the joint learning layer.

The joint learning layer aims to learn the unknown data distribution using the conditional CNF model, leveraging the change of variable theorem. Specifically, conditioned on the hidden state, we learn a differentiable bijective function that maps the unknown distribution to a simple base distribution—such as a multivariate normal distribution—for which the likelihood is easy to compute. The likelihood of a real data point in the unknown distribution can then be calculated by combining the likelihood of the transformed data point under the base distribution with the transformation loss.

Furthermore, the joint learning layer is tailored to handle both Syn-MTS and Asyn-MTS scenarios. The key difference between these cases is that in Asyn-MTS, some variables may be missing at certain time points, making it impossible to compute the likelihood. To address this issue, we force each variable in the base distribution to be independent and compute the likelihood of the observed variables only.

Below, we first describe the marginal learning layer, followed by a detailed explanation of the joint learning layer for both synchronous and asynchronous cases.

\begin{figure*}[htb]
	\centering
	\subfigure[Synchronous case]{
		\includegraphics[width=0.97\columnwidth]{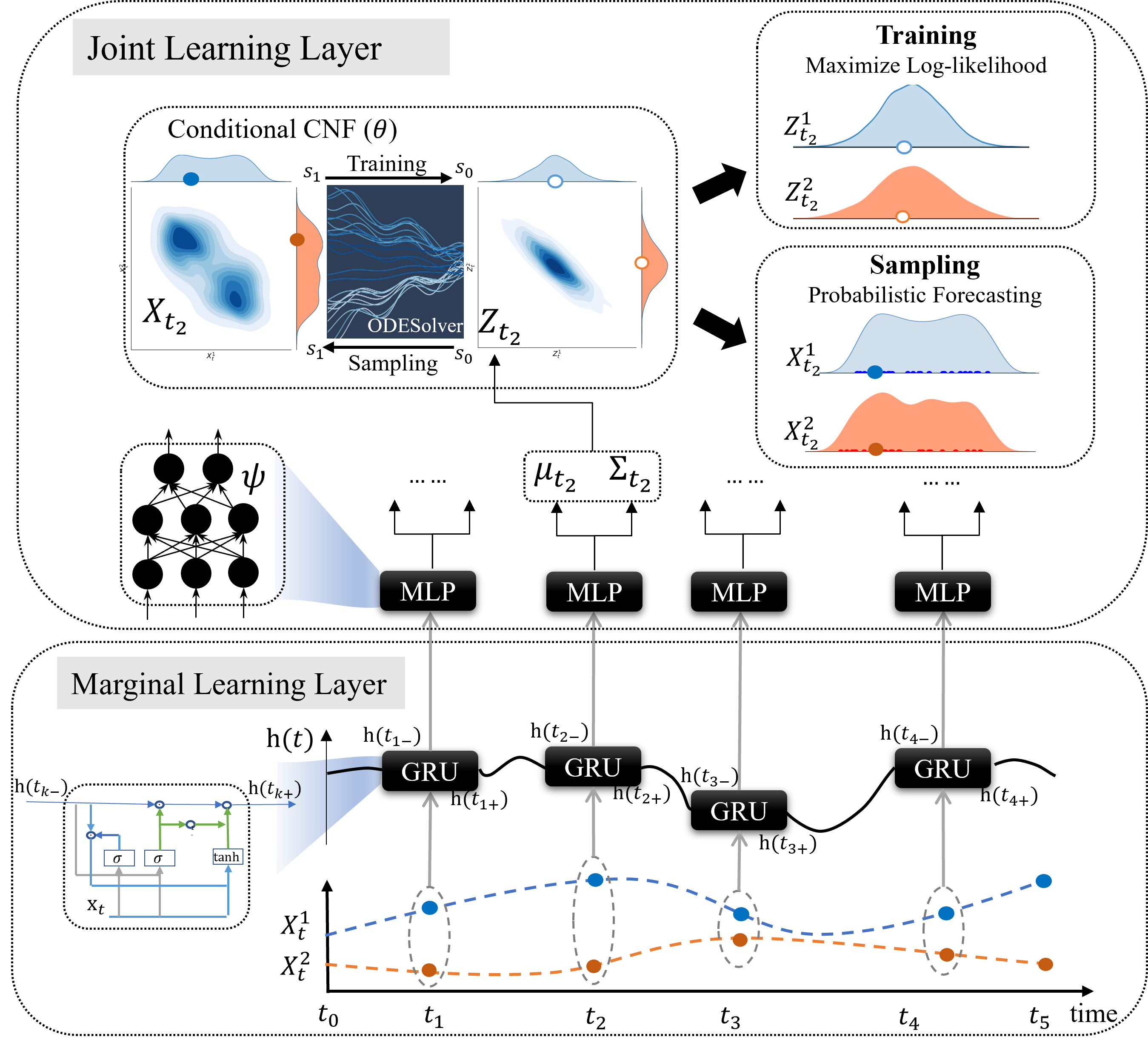}
		\label{RFN-GRUODE-Syn-MTS}
	}%
    \quad
	\subfigure[Asynchronous case]{
		\includegraphics[width=0.97\columnwidth]{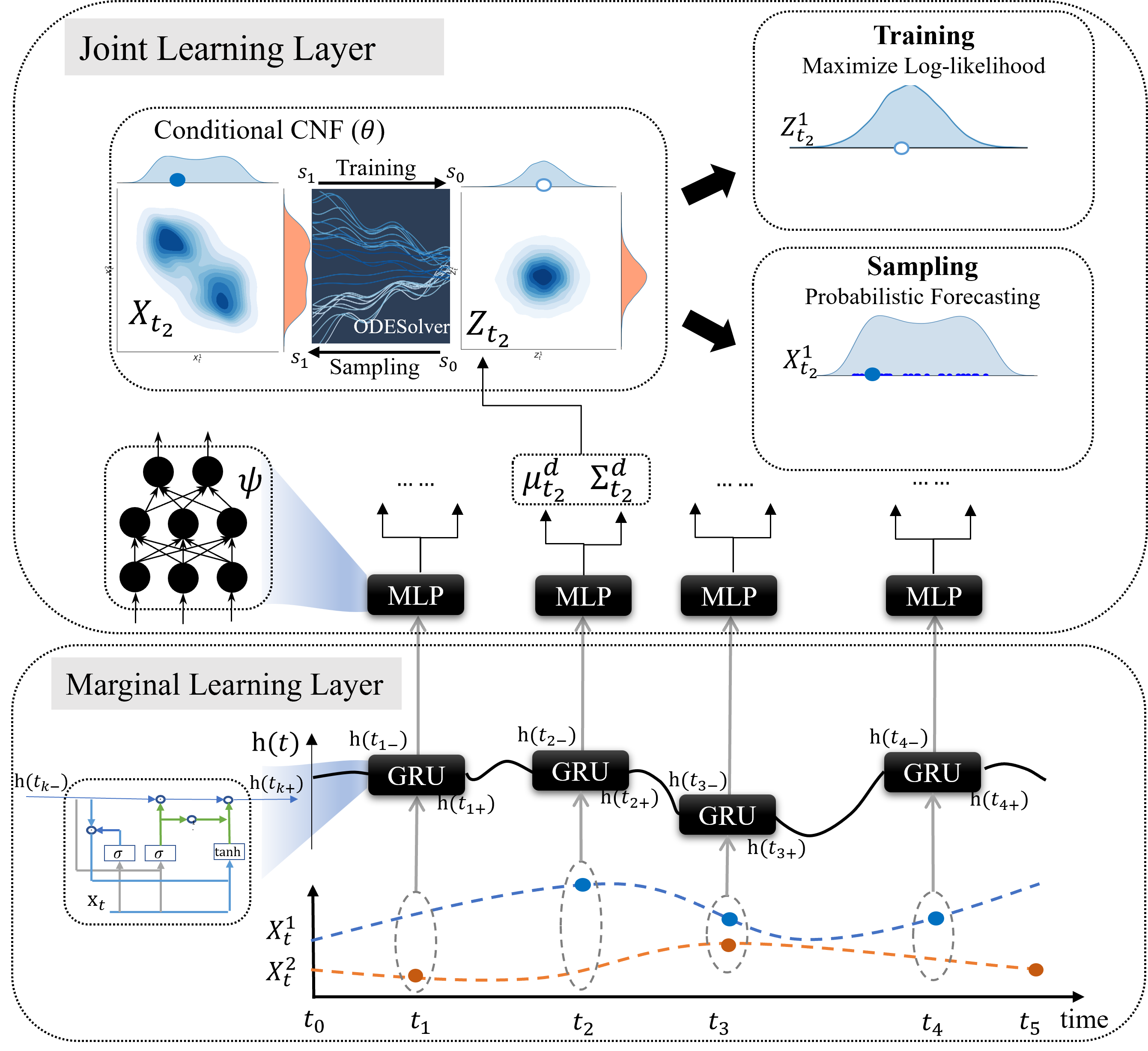} 
		\label{RFN-GRUODE-Asyn-MTS}
	}%
	\caption{The framework of RFNs for (a) Syn-MTS and (b) Asyn-MTS. In both cases, there are two component variables $X^1_t, X^2_t$. The solid points in different colors indicate they are observations of different variables. In the marginal learning layer, the hidden states will be updated only when at least one variable has an observation, e.g., from $\mathrm{h}(t_{1-})$ to $\mathrm{h}(t_{1+})$. In the joint learning layer, the base distribution parameters at each time are $\mu_t$ and $\Sigma_t$ (for the Syn-MTS case) or $\mu_t^d$ and $\Sigma_t^d$ (for the Asyn-MTS case), which is learned from hidden state $\mathrm{h}(t_{1-})$. The conditional CNF transforms the data points following the unknown distribution to the base distribution, for which likelihoods are easy to compute.   \label{fig:RFN-GRUODE}}
\end{figure*}

\subsection{Marginal Learning Layer}
The marginal learning layer aims to acknowledge temporal irregularities in the MTS data. Various existing recurrent architectures targeting handling irregularly sampled time series can be employed, such as those discussed in Section \ref{subsec: On Handling Irregular Sampling}. In this section, we dedicate to providing one example, GRU-ODE-Bayes \cite{de2019gru}, to elucidate the functionality of the marginal learning layer and its interactions with the joint learning layer. 

The vanilla GRU has the following updating formulas:
\begin{align}
	\mathrm{h}_t&=(1-\mathrm{z}_t) \odot \tilde{\mathrm{h}}_t+\mathrm{z}_t \odot \mathrm{h}_{t-1},\label{Eq3}
\end{align}
where $\mathrm{z}_t$, $\tilde{\mathrm{h}}_t$, $\mathrm{h}_t\in \mathbb{R}^H$ are vectors denoting the update gates, the candidate update gates, and the hidden states; the operator $\odot$ denotes the element-wise multiplication. 

Subtracting $\mathrm{h}_{t-1}$ on both sides of \eqref{Eq3} leads to
\begin{align*}
\mathrm{h}_{t}-\mathrm{h}_{t-1} &=(1-\mathrm{z}_{t}) \odot (\tilde{\mathrm{h}}_{t}-\mathrm{h}_{t-1}).
\end{align*}
Then, by taking the limit as the time difference between $t-1$ and $t$ tends to zero, one arrives at the continuous-time dynamics of the hidden states: 
\begin{align}
	\frac{d \mathrm{h}(t)}{d t}&=(1-\mathrm{z}(t)) \odot (\tilde{\mathrm{h}}(t)-\mathrm{h}(t)). \label{Eq4}
\end{align}
 We thus specify equations governing the evolution of the hidden states in two regimes, continuous-updating and discrete-updating.

The {\textit{continuous-updating}} regime corresponds to continuous time intervals from  $t_{0+}$ to $t_{1-}$, $t_{1+}$ to $t_{2-}$, $t_{2+}$ to $t_{3-}$, $t_{3+}$ to $t_{4-}$, and $t_{4+}$ to $t_{5-}$ in Fig. \ref{fig:RFN-GRUODE} when no variables have observations. Here, $t_{-}$ denotes the time point right before $t$ and $t_{+}$ means the time point immediately after $t$. In this regime, we assume that the hidden states for all variables evolve according to equation \eqref{Eq4} between observation intervals. To be concrete, the details of evolution are given in equation \eqref{Eq6}.

The {\textit{discrete-updating}} regime corresponds to the set of discrete time points $t=t_1,t_2,t_3,t_4$ at which at least one variable is observed. In this case, the hidden state will be updated according to the observations $\mathrm{x}_{t_k}$ via vanilla GRU cell, i.e., equation \eqref{Eq7}.
\\[5pt]
\textit{\textbf{Continuous-updating}}:
\begin{align}
& \forall t\in [0,T] \backslash [t_1,\cdots,t_K], \nonumber \\
&\frac{d \mathrm{h}(t)}{d t}=(1-\mathrm{z}^c(t)) \odot(\mathrm{\tilde{h}}^c(t)-\mathrm{h}(t)),\quad \textit{s.t.}\label{Eq6}\\
&\quad 
\begin{cases}\nonumber
\mathrm{r}^c(t)=\sigma(\mathbf{w}^c_{r} \mathrm{x}(t)+\mathbf{u}^c_{r} \mathrm{h}(t)+\mathrm{b}^c_{r}),\\
\mathrm{z}^c(t)=\sigma(\mathbf{w}^c_{z} \mathrm{x}(t)+\mathbf{u}^c_{z} \mathrm{h}(t)+\mathrm{b}^c_{z}),\\
\mathrm{\tilde{h}}^c(t)=\tanh (\mathbf{w}^c_{h} \mathrm{x}(t)+\mathbf{u}^c_{h} (\mathrm{r}^c(t) \odot \mathrm{h}(t))+\mathrm{b}^c_{h}).
\end{cases}
\end{align}
\\[5pt]
\textit{\textbf{Discrete-updating}}: 
\begin{align}
& \forall t \in [t_1,\cdots,t_K], \nonumber \\
&\mathrm{h}(t_{+})=(1-\mathrm{z}^u(t_{-})) \odot \mathrm{\tilde{h}}^u(t_{-})+\mathrm{z}^u(t_{-})\odot\mathrm{h}(t_{-})\;\; \textit{s.t.} \label{Eq7} \\
& \quad 
\begin{cases}\nonumber
\mathrm{r}^u(t_{-})=\sigma\left(\mathbf{w}^u_{r} \mathrm{x}_t+\mathbf{u}^u_{r} \mathrm{h}(t_{-})+\mathrm{b}^u_{r}\right),\\
\mathrm{z}^u(t_{-})=\sigma\left(\mathbf{w}^u_{z} \mathrm{x}_t+\mathbf{u}^u_{z}  \mathrm{h}(t_{-})+\mathrm{b}^u_{z}\right),\\
\mathrm{\tilde{h}}^u(t_{-})=\tanh \left(\mathbf{w}^u_{h} \mathrm{x}_t+\mathbf{u}^u_{h} (\mathrm{r}^u(t_{-}) \odot \mathrm{h}(t_{-}))+\mathrm{b}^u_{h}\right).
\end{cases}
\end{align}

The architecture of the marginal learning layer can be widely adaptable and can be substituted with other sequential models. These alternative formulations are detailed in Appendix \ref{Appendix:Marginal learning blocks}.

\subsection{Joint Synchronous Learning Layer} \label{sec:Multivariate Synchronous Learning}
Once the marginal learning layer is set up, we seek to maximize the log-likelihood of all observations,
\begin{align*}
\max_{\Phi} \log p(\mathrm{x}_{t_1}, \cdots, \mathrm{x}_{t_{K}}; \Phi),
\end{align*} 
where $p$ is the joint density to be estimated from the observed data at all observation times $[t_1,\cdots,t_K]$ and $\Phi=\{\phi_{1},\cdots,\phi_{K}\}$ is the parameter set. 

\subsubsection{Factorizing Log-likelihood Objective}
If the multivariate sequence is synchronous, observing one variable implies the observation of all other variables, despite potentially unevenly-spaced time intervals. The synchronous property ensures alignments, enabling the use of the chain rule of probability. Thus, we can decompose $p(\mathrm{x}_{t_1}, \cdots, \mathrm{x}_{t_{K}};\Phi)$ as
\begin{align}
&p(\mathrm{x}_{t_1}, \cdots, \mathrm{x}_{t_{K}};\Phi) \nonumber\\
=&p(\mathrm{x}_{t_1};\phi_1) 
p(\mathrm{x}_{t_2}|\mathrm{x}_{t_1};\phi_2) \cdots  p(\mathrm{x}_{t_K}|\mathrm{x}_{t_1},\cdots,\mathrm{x}_{t_{K-1}};\phi_K) \nonumber \\
=&p(\mathrm{x}_{t_1};\phi_1) 
\prod{_{k=2}^K} p(\mathrm{x}_{t_k}|\mathrm{x}_{t_1},\cdots,\mathrm{x}_{t_{k-1}};\phi_k). \label{Eq9}
\end{align}
Such decomposition can be used to factorize the maximization of the joint log-likelihood at all observation times into optimizing conditional log-likelihoods individually:
\begin{align*}
&\max_{\Phi} \log p(\mathrm{x}_{t_1}, \cdots, \mathrm{x}_{t_{K}};\Phi)\\
=&\max_{\phi_1} \log p (\mathrm{x}_{t_1};\phi_1) + 
\sum_{k=2}^{K}  \max_{\phi_k}\log  p(\mathrm{x}_{t_{k}}|\mathrm{x}_{t_{1:k-1}};\phi_k).
\end{align*}

Note that the amount of exogenous information in the hidden state up to $t_{k-1}$ will not increase until the next observation arrives at $t_k$. Therefore, conditioning on the sample path $\mathrm{x}_{t_1}, \cdots, \mathrm{x}_{t_{k-1}}$ is equivalent to conditioning on the hidden state at time $t_{k-}$
under mild assumptions.  As a result, we have
\begin{equation}
\begin{aligned}\label{Eq10}
\max_{\Phi} \log p(\mathrm{x}_{t_1}, \cdots, \mathrm{x}_{t_{K}};\Phi)
=\sum_{k=1}^K \max_{\phi_k}\log p(\mathrm{x}_{t_{k}}|\mathrm{h}_{t_{k-}};\phi_k).
\end{aligned}
\end{equation}

\subsubsection{Conditional Flow Representation}
Optimizing \eqref{Eq10} requires a conditional representation of the data distribution, wherein the log-densities can be tailored to suit one's specific conditional choices. For this purpose, we develop the following conditional formulation of the continuous normalizing flow. Fig. \ref{fig:cond_CNF} displays the framework of conditional CNF. 

Given three random variables $\mathrm{X}$, $\mathrm{Y}$, and $\mathrm{Z}$, our objective is constructing conditional CNF, i.e., expressing conditional log-density $\log p(\mathrm{x}|\mathrm{y})$ in terms of $\log p(\mathrm{z}|\mathrm{y})$ when there is a non-parametric map $\mathrm{f}$ that maps $\mathrm{z}$ to $\mathrm{x}$ under the control of $\mathrm{y}$. Lemma \ref{thm:ffjord joint density} is useful to achieve our goal.

\begin{lemma}\label{thm:ffjord joint density}
Let $\tilde{\mathrm{z}}(s)=[\mathrm{z}(s),\mathrm{y}(s)]^{\top}$ be a finite continuous random variable, and the probability density function of $\tilde{\mathrm{z}}(s)$ is $p(\tilde{\mathrm{z}}(s))=p(\mathrm{z}(s),\mathrm{y}(s))$ which depends on flow time $s$, where $s_{0}\leq s\leq s_{1}$. Given the governing dynamics of $\tilde{\mathrm{z}}(s)$ as
\begin{equation*}
\begin{aligned}
\frac{\partial \tilde{\mathrm{z}}(s)}{\partial s} =\left[\begin{array}{c}
\frac{\partial\mathrm{z}(s)}{\partial s}  \\
\frac{\partial\mathrm{y}(s)}{\partial s} 
\end{array}\right]=\left[\begin{array}{c}
\mathrm{f}(\mathrm{z}(s), s, \mathrm{y}(s); \theta) \\
0
\end{array}\right], 
\end{aligned}
\end{equation*}
\text{where} $s_0 \leq s \leq s_1$,	$\mathrm{z}(s_1)=\mathrm{x}$, $\mathrm{y}(s_1)=\mathrm{y}$, and $\mathrm{f}$ is Lipschitz continuous in $\mathrm{z}$ and continuous in $s$ for any $\mathrm{y}$. We have
	\begin{equation} \label{Eq11}
			\begin{aligned}
				&\quad \, \log p(\mathrm{x},\mathrm{y}) \\
				&= \log p(\mathrm{z}(s_{0}), \mathrm{y})+\int_{s_1}^{s_0}\operatorname{Tr}[\partial_{\mathrm{z}(s)} \mathrm{f}(\mathrm{z}(s), s, \mathrm{y}; \theta)]ds. 
			\end{aligned}
	\end{equation} 
\end{lemma}
\begin{proof}
See the Appendix \ref{Appendix:Proof of Theorem}. 
\end{proof}

Subtracting $\log p(\mathrm{y})$ on both sides of equation \eqref{Eq11} in Lemma \ref{thm:ffjord joint density} and using the fact that $\log p(\mathrm{x}, \mathrm{y}) - \log p(\mathrm{y}) = \log \frac{p(\mathrm{x},\mathrm{y})}{p(\mathrm{y})}= \log p(\mathrm{x}|\mathrm{y})$, we obtain a result of conditional CNF that is summarized in Proposition \ref{thm:ffjord conditional density}.


\begin{proposition} \label{thm:ffjord conditional density}
Let the assumptions in Lemma \ref{thm:ffjord joint density} hold. The conditional log-density $p(\mathrm{x}|\mathrm{y})$ is given by:
\begin{equation}
\begin{aligned}\label{Eq12}
&\log p(\mathrm{x}|\mathrm{y}) \\ 
=&\log p(\mathrm{z}(s_{0})|\mathrm{y})+\int_{s_1}^{s_0}\operatorname{Tr}[\partial_{\mathrm{z}(s)} \mathrm{f}(\mathrm{z}(s), s, \mathrm{y}; \theta)]ds. 
\end{aligned}
\end{equation} 
\end{proposition}

By extending the unconditional CNF formulation (see Section \ref{subsec:flow}) to the conditional formulation, equation (\ref{Eq12}) allows the data distribution to be both non-parametric via the flow map $\mathrm{f}$ and time-varying, e.g., by choosing the conditional information $\mathrm{y}$ properly. 

\begin{figure}[ht]
    \centering
    \includegraphics[width=0.9\linewidth]{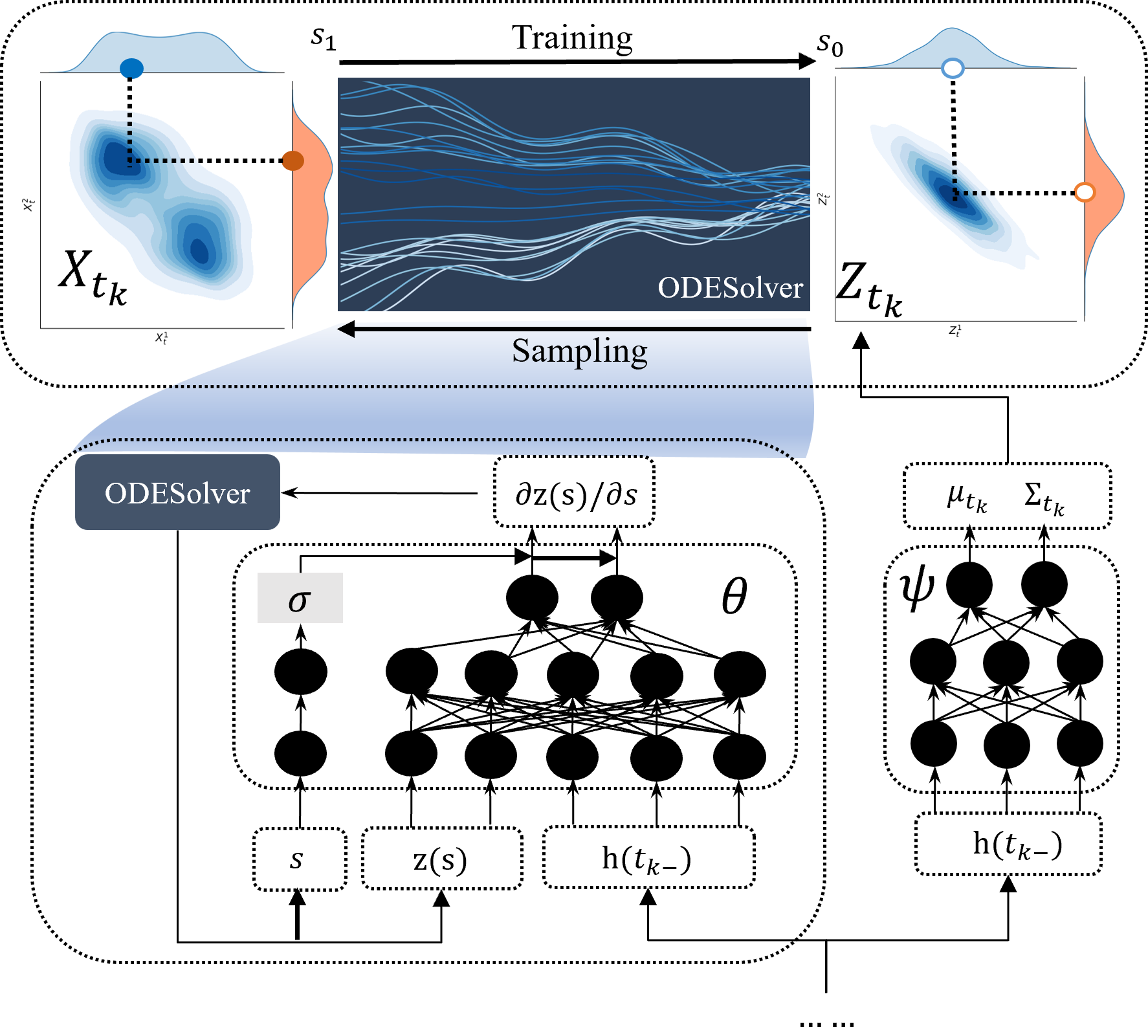}
    \vspace{-0.2cm}
    \caption{The framework of conditional CNF.}
    \label{fig:cond_CNF}
\end{figure}
\subsubsection{Time-varying Specification}
We now discuss how to obtain $p(\mathrm{x}_{t_{k}}|\mathrm{h}_{t_{k-}};\phi_k)$ in equation \eqref{Eq10} via conditional CNF. First, the conditional information $\mathrm{y}$ in equation \eqref{Eq12} is the hidden states $\mathrm{h}_{t_{k-}}$. We assume that there is a bijective map between $\mathrm{x}_{t_{k}}$ and $\mathrm{z}_{t_{k}}$ such that $\mathrm{z}_{t_{k}}|\mathrm{h}_{t_{k-}}$ is Gaussian distributed. Further, we shall let the hidden states determine the parameter values of the Gaussian base distribution. Because hidden states contain information on historical data and vary at different observation times, this choice of conditioning leads to path-dependent parameters of the base distribution. 

Specifically, we assume that the mean vector $\mu_{t_{k}}$ and covariance matrix $\Sigma_{t_{k}}$ of the Gaussian base are functions (learned by standard MLP, denoted by $\mathrm{g}$) of the hidden states $\mathrm{h}_{t_{k-}}$ at time-${t_{k-}}$ learned from the marginal learning layer, i.e.,
\begin{equation}
\begin{aligned}\label{Eq13}
&\mathrm{Z}_{t_k}|\mathrm{h}_{t_{k-}} \sim \mathcal{N}(\mu_{t_{k}} ,\Sigma_{t_{k}}), \quad \text{where}\\
&\{\mu_{t_k}, \Sigma_{t_k}\}=\mathrm{g}(\mathrm{h}_{t_{k-}}; \psi).
\end{aligned}
\end{equation}
This is in contrast with the unconditional flow models where the base distribution parameters are constants as it is a standard Normal $\mathcal{N}(\mathrm{0},\mathbb{I}_D)$. 


It is worth noting that previous studies \citep{salinas2020deepar, de2019gru, salinas2019high} also use hidden states to learn Gaussian parameters. However, those Gaussian parameters are associated with the data distribution. Our representation shifts the Gaussian assumption from the data distribution to the base distribution, and it is the base distribution parameters that depend on hidden states. This distinction is crucial because our representation enables the data distribution to be non-parametric, non-Gaussian, and time-varying simultaneously.

Meanwhile, we use the following gated mechanism to incorporate the hidden state dependence into the flow map. Mathematically, we have
\begin{equation}
\begin{aligned}\label{Eq14}
&\mathrm{f}(\mathrm{z}(s),s,\mathrm{h}_{t_{k-}};\theta)\\
=&(\mathbf{w}_z \mathrm{z}(s)+\mathbf{w}_h \mathrm{h}_{t_{k-}}+\mathrm{b}_z)\sigma(w_{s}s + b_{s}).
\end{aligned}
\end{equation}
Here, $\sigma$ is the sigmoid activation function. $\theta$ is the set of trainable parameters which includes $\mathbf{w}_h\in \mathbb{R}^{D\times H}$, $\mathrm{b}_z\in \mathbb{R}^{D}$, and $\{w_{s}, b_{s}\}\in \mathbb{R}$. The resulting flow dynamics is


\begin{align*}
	&\frac{\partial \mathrm{z}(s)}{\partial s}=\mathrm{f}(\mathrm{z}(s), s, \mathrm{h}_{t_{k-}}; \theta), s \in [s_0, s_1], \\
	&\text{where} \quad \mathrm{z}(s)|_{s=s_0} = \mathrm{z}_{t_k} , \,\, \mathrm{z}(s)|_{s=s_1} = \mathrm{x}_{t_k}.
\end{align*}
The initial value of the flow $\mathrm{z}(s)|_{s=s_0}$ is set as the Gaussian base sample $\mathrm{z}_{t_k}$, sampled from $\mathrm{Z}_{t_k}|\mathrm{h}_{t_{k-}}$ per (\ref{Eq13}), and the terminal flow value $\mathrm{z}(s)|_{s=s_1}$ is set to equal the time-$t_k$ observed data $\mathrm{x}_{t_k}$. Its solution maps samples from base distribution to data distribution concurrently, given as
\begin{align}
\mathrm{x}_{t_k} = \mathrm{z}_{t_k} + \int_{s_0}^{s_1}\mathrm{f}(\mathrm{z}(s),s, \mathrm{h}_{t_{k-}};\theta)ds. \label{Eq15}
\end{align} 

We summarize the time-varying specification of the conditional CNF representation for Syn-MTS data as follows. For $t_k\in [t_1,\cdots, t_K]$, we have
{\small
\begin{equation}
		\begin{aligned}
			&\quad \,
			\begin{bmatrix}
				\mathrm{x}_{t_k}\\
				\log p(\mathrm{x}_{t_k}|\mathrm{h}_{t_{k-}}; \phi_{t_k})
			\end{bmatrix}\\
			&=
			\begin{bmatrix}
				\mathrm{z}_{t_k} \\
				\log p(\mathrm{z}_{t_k}|\mathrm{h}_{t_{k-}};\mu_{t_k}, \Sigma_{t_k})
			\end{bmatrix}+\int_{s_0}^{s_1}
			\begin{bmatrix}
				\mathrm{f}\big(\mathrm{z}(s), s, \mathrm{h}_{t_{k-}}; \theta\big) \\
				-\operatorname{Tr} [\partial_{\mathrm{z}(s)} \mathrm{f}]
			\end{bmatrix}ds. \label{Eq16}
		\end{aligned}
\end{equation}
}\noindent
It has two essential components, the time-varying base distribution driven by hidden states where $\psi$ in (\ref{Eq13}) is the parameter to be trained, and the time-varying flow map indexed by observation arrival times with $\theta$ in \eqref{Eq14} being the set of trainable parameters.

\subsubsection{Training and Sampling}	
To train the NFR model, we maximize the conditional log-likelihoods \eqref{Eq10} across all sample instances at all observation times. We set the dataset level log-likelihood $ \mathcal{L_{\text{Syn-MTS}}}$ as the sample average of $(\mathcal{L}^{i}_{\text{Syn-MTS}})_{i=1}^{N}$, where each $\mathcal{L}^{i}_{\text{Syn-MTS}}$ is weighted by the number of observation times $K_i$. To be concrete, we have
{\small
	\begin{align}
		&\mathcal{L_{\text{Syn-MTS}}}\nonumber\\
		=&\frac{1}{N}\underset{i=1}{\overset{N}{\sum}}\mathcal{L}^{i}_{\text{Syn-MTS}} 
		= \frac{1}{N} \sum_{i=1}^{N}\frac{1}{K_{i}} \underset{k=1}{\overset{K_{i}}{\sum}} \log p(\mathrm{x}_{i,t_{k}^{i}}| \mathrm{h}_{i,t_{k-}^{i}}; \phi_{i,t_k^{i}} )\nonumber\\
		=&\frac{1}{N} \sum_{i=1}^{N} \frac{1}{K_{i}} \underset{k=1}{\overset{K_{i}}{\sum}} 
		\begin{pmatrix}
			\log p(\mathrm{z}_{i,t_k^{i}}|\mathrm{h}_{i,t_{k-}^{i}} ;\mu_{i,t_k^{i}}, \Sigma_{i,t_k^{i}}) + \qquad\\
			\quad\int_{s_1}^{s_0}\operatorname{Tr}\big[\partial_{\mathrm{z}(s)} \mathrm{f}(\mathrm{z}(s),s, \mathrm{h}_{i,t_{k-}^{i}};\theta)\big] ds
		\end{pmatrix}.\label{Eq17}
	\end{align}
}

Once the model is trained, we can forecast the joint data distribution at time point $t_{k}$ given the observations $\{\mathrm{x}_{t_1}, \cdots, \mathrm{x}_{t_{k-1}}\}$ of the $i^{\text{th}}$ instance. The confidence interval of predictions can be obtained by sampling from the learned joint distribution. 

For sampling, we first obtain the hidden states $\mathrm{h}_{t_{k-}}$ for the $i^{\mathrm{th}}$ instance based on the observations $\{\mathrm{x}_{t_1}, \cdots, \mathrm{x}_{t_{k-1}}\}$ per \eqref{Eq6} and \eqref{Eq7}, and use $\mathrm{h}_{t_{k-}}$ to predict the base distribution parameters $\{\mu_{t_k},\Sigma_{t_k}\}$ per \eqref{Eq13}. We then sample a given number of points $\mathrm{z}_{t_k}$ from the base distribution with the predicted parameters. Finally, we transform these base sample points using the conditional CNF map \eqref{Eq15} to obtain concurrent samples that follow the data distribution.

\subsection{Joint Asynchronous Learning Layer}\label{sec:Multivariate Asynchronous Learning}
If the multivariate sequence is asynchronous, it implies the possibility that not all entries of the vector $\mathrm{x}_{t_k}$ are observable at any given time $t_k\in \mathrm{t}$ (as defined in Definition 2). 

\subsubsection{Masked Independent Log-likelihoods Objective}
In the scenario of Syn-MTS, it is observed that all variables simultaneously record observations $\mathrm{x}_{t_k} \in \mathbb{R}^D$ at a specific observational time point $t_k$. Under these circumstances, one can convert it into a sample $\mathrm{z}_{t_k} \in \mathbb{R}^D$ following multivariate Gaussian distribution conditioned on the hidden state $\mathrm{h}_{t_{k-}}$ via a conditional CNF. The mean and covariance of this multivariate Gaussian distribution are represented as $\mu_{t_k}$ and $\Sigma_{t_k}$, respectively. Consequently, the log-likelihood of $\mathrm{x}_{t_k}$ can be computed as
\begin{equation*}
\begin{aligned}
&\log p(\mathrm{x}_{t_k}|\mathrm{h}_{t_{k-}}; \phi_{t_k})\\
=&\log p(x_{t_k}^1, \cdots, x_{t_k}^D|\mathrm{h}_{t_{k-}}; \phi_{t_k}) \\
=&\log p(z_{t_k}^1, \cdots, z_{t_k}^D|\mathrm{h}_{t_{k-}}; \mu_{t_k}, \Sigma_{t_k}) + \int_{s_0}^{s_1}-\operatorname{Tr}[\partial_{z(s)}\mathrm{f}]ds.
\end{aligned}
\end{equation*}

In the Asyn-MTS cases, where some variables do not have values at given times, the approach used in the synchronous case cannot be applied directly. 
To resolve this additional complication, we restrict each dimension of conditional multivariate Gaussian distribution to be independent. This allows for the decomposition of the joint log-likelihood $\log p(z_{t_k}^1, \cdots, z_{t_k}^D|\mathrm{h}_{t_{k-}}; \mu_{t_k}, \Sigma_{t_k})$ into individual components, as follows
\begin{align*}
    p(z_{t_k}^1, \cdots, z_{t_k}^D|\mathrm{h}_{t_{k-}}; \mu_{t_k}, \Sigma_{t_k})=\underset{d=1}{\overset{D}{\prod}} p(z_{t_k}^{d}|\mathrm{h}_{t_{k-}}; \mu_{t_k}^{d}, \Sigma_{t_k}^{d}).
\end{align*}
Thus, we have
\begin{align}
&\log p(x_{t_k}^1, \cdots, x_{t_k}^D|\mathrm{h}_{t_{k-}}; \phi_{t_k}) \nonumber\\
=&\log\underset{d=1}{\overset{D}{\prod}} p(z_{t_k}^{d}|\mathrm{h}_{t_{k-}}; \mu_{t_k}^{d}, \Sigma_{t_k}^{d})+\int_{s_0}^{s_1}-\operatorname{Tr}[\partial_{z(s)}\mathrm{f}]ds \nonumber\\
=&\sum_{d=1}^{D} \log p(z_{t_k}^d|\mathrm{h}_{t_{k-}}; \mu_{t_k}^d, \Sigma_{t_k}^d)+ \int_{s_0}^{s_1}-\operatorname{Tr}[\partial_{z(s)}\mathrm{f}]ds.\label{eqt:obj_fun_asyn_train}
\end{align}
Note that $(\mu_{t_k}^{d}, \Sigma_{t_k}^{d})\in\mathbb{R}\times\mathbb{R}$ are learned from MLPs. For each $d$, $\mu_{t_k}^{d}$ and $\Sigma_{t_k}^{d}$ represent the mean and variance of the variable $Z_{t_{k}}^{d}$ conditioned on the hidden state $\mathrm{h}_{t_{k-}}$, which is assumed to follow a Gaussian distribution.

Nevertheless, it is noteworthy that at a specific observation time $t_k$, not all variables $x_{t_{k}}^{d}$ are necessarily observed. For any unobserved variable $x_{t_{k}}^{d}$, its corresponding log-density should be omitted from \eqref{eqt:obj_fun_asyn_train}. To address this scenario in practice, we employ a masking mechanism denoted as $m_{t_{k}}^{d}$ to mitigate the impact of the unobserved quantity $\log p(z_{t_k}^d|\mathrm{h}_{t_{k-}}; \mu_{t_k}^d, \Sigma_{t_k}^d)$. As a result, equation \eqref{eqt:obj_fun_asyn_train} is transformed into the form:
\begin{align}\label{eqt:resultant asyn obj}
&\log p(z_{t_k}^d|\mathrm{h}_{t_{k-}}; \mathrm{m}_{t_{k}}, \phi_{t_k})\nonumber\\
=&\sum_{d=1}^{D} m_{t_k}^d \log p(z_{t_k}^d|\mathrm{h}_{t_{k-}}; \mu_{t_k}^d, \Sigma_{t_k}^d)+ \int_{s_0}^{s_1}-\operatorname{Tr}[\partial_{z(s)}\mathrm{f}]ds.
\end{align}
This decomposition along the component dimension facilitates the optimization process by focusing solely on the conditional distribution of variables with observations at time $t_k$, rather than optimizing across all $D$ variables. Essentially, any component lacking an observation at $t_k$ is effectively masked out, thereby ensuring its exclusion from influencing the overall loss function.

In addition to the decomposition of the base Gaussian distribution, it is equally crucial to address how unobserved components are handled during the transformation process $\mathrm{x}_{t_k}=\mathrm{f}(\mathrm{z}_{t_k})$. Allowing transformations of missing values in parallel with observed data can significantly skew the learning outcomes. This is attributed to the fact that the transformation applied to these unobserved values inadvertently influences the cumulative transformation loss, specifically {\small $\int_{s_0}^{s_1}-\operatorname{Tr}[\partial_{z(s)}\mathrm{f}]ds$}. 
To mitigate the problem, we allow the transformation process for those observed components, and keep the process unchanged for those unobserved components at each $t_{k}\in\mathrm{t}$. Thus, the evolution equation of $z^{d}(s)$ for each $d$ along the flow time can be described by the following ODE:
\begin{equation}
\begin{aligned}\label{Eq21}
\frac{\partial z^d(s)}{\partial s} = \mathrm{f}(z^d(s),s,\mathrm{h}_{t_{k-}};  m_{t_k}^d, \theta) \;,\;\; s \in [s_0, s_1],
\end{aligned}
\end{equation}
where
{\small
\begin{align*}\label{Eq21}
\mathrm{f}(z^d(s),s,\mathrm{h}_{t_{k-}}; m_{t_k}^d, \theta)=
\begin{cases}
\begin{aligned}
\mathrm{f}(z^d(s),&s,\mathrm{h}_{t_{k-}}; \theta) && \text{if $m_{t_k}^d=1$,}\\
& 0 && \text{if $m_{t_k}^d=0$.}
\end{aligned}
\end{cases}
\end{align*}
}

We summarize the conditional CNF representation for Asyn-MTS data as
{\small
\begin{equation}
    \begin{aligned}\label{Eq22}
        &\quad \,\begin{bmatrix}
            \mathrm{x}_{t_k} \\
            \log p(\mathrm{x}_{t_k} | \mathrm{h}_{t_{k-}}; \mathrm{m}_{t_k}, \phi_{t_k})
        \end{bmatrix} \\
        &=
        \resizebox{0.5\columnwidth}{!}{$
            \begin{bmatrix}
                \mathrm{z}_{t_k}  \\
                \sum_{d=1}^{D} m^d_{t_k}\log p\big(z_{t_k}^d| \mathrm{h}_{t_{k-}};\Sigma^d_{t_k}, \mu_{t_k}^d\big)
            \end{bmatrix}+
            $}
        \int_{s_0}^{s_1}
        \resizebox{0.36\columnwidth}{!}{$			
            \begin{bmatrix}
                \mathrm{f}(\mathrm{z}(s), s, \mathrm{h}_{t_{k-}},  \mathrm{m}_{t_k}; \theta) \\
                -\operatorname{Tr}\big[ \partial_{z^d(s)} \mathrm{f} \big]
            \end{bmatrix}ds
            $}.
    \end{aligned}
\end{equation}
}

\subsubsection{Training and Sampling}
The target in the asynchronous case also maximizes the log-likelihood of observations. However, different from the objective function \eqref{Eq16} of the Syn-MTS model, the unobserved components are masked. Mathematically, we have

{\small
\begin{align}\label{Eq23}
&\mathcal{L}_{\text{Asyn-MTS}}=\frac{1}{N}\underset{i=1}{\overset{N}{\sum}}\mathcal{L}^{i}_{\text{Asyn-MTS}}, \quad \text{where} \nonumber\\
&\mathcal{L}^{i}_{\text{Asyn-MTS}} \nonumber \\ &=\sum_{\substack{1\leq k\leq K_{i}\\1\leq d\leq D}}\frac{1}{K_{i}D}
		\begin{pmatrix}
			m^{d}_{i,t_{k}^{i}}\log p(z_{i,t_{k}^{i}}^d|\mathrm{h}_{i,t_{k-}^{i}};\mu_{i,t_k^{i}}^d, \Sigma_{i,t_k^{i}}^{d})+\\
			\int_{s_1}^{s_0}\operatorname{Tr}[\partial_{z^d(s)}\mathrm{f}(z^d(s),s,\mathrm{h}_{t_{k-}}; m_{t_k}^d, \theta)]ds
		\end{pmatrix}.
\end{align}
}

The sampling process at time $t_{k}$ of the Asyn-MTS model is similar to the Syn-MTS model. We first predict the parameters of base distribution of variables $X_{t_k}^1, \cdots, X_{t_k}^D$ conditional on the hidden state $\mathrm{h}_{t_{k-}}$ from marginal learning layer, i.e., {\small$\{\mu_{t_k}^1, \Sigma_{t_k}^1\}, \cdots, \{\mu_{t_k}^D, \Sigma_{t_k}^D\}$}. Then, we sample a given number of points $z_{t_k}^1, \cdots, z_{t_k}^D$ from the base distributions with the predicted parameters, and we transform these points using the conditional CNF model. 

The algorithms for the training and sampling procedures of the Syn-MTS model and the Asyn-MTS model are provided in Appendix \ref{Appendix:Algorithms}.
\section{Experiments}
We conduct one simulation experiment to verify the efficacy of the model specification, followed by three real-world datasets to evaluate the overall performance of the RFNs. In all studies and experiments, we use 70\%/15\%/15\%  of the sample instances for training/validation/testing. 

\subsection{Baseline Models}
     We establish four baseline models specifically to facilitate the construction of the marginal learning layer, 
     namely GRUODE \citep{de2019gru}, GRU-D \citep{che2018recurrent}, ODERNN \citep{rubanova2019latent}, and ODELSTM \citep{lechner2020learning}. These baselines are all designed to model multivariate irregularly sampled time series. Following \citep{de2019gru}, We adapt them by assuming that each variable follows a multivariate normal distribution, and thus their joint learning layer is designed to learn the parameters of these normal distributions. Specifically, the parameters of the normal distribution, $\mu_{t_k}$ and $\sigma_{t_k}$ are predicted based on the hidden state $\mathrm{h}(t_k)$ via a feed-forward neural network directly. 
     
     To compare the vanilla form of a baseline model, e.g., GRU-D, with its RFN counterpart, e.g., termed as RFN-GRU-D, we use the corresponding RFN specifications laid out in Section \ref{sec:Multivariate Synchronous Learning} and \ref{sec:Multivariate Asynchronous Learning} for the multivariate learning. For a fair and robust comparison, we employ the same marginal learning layer across all baseline models in RFN models. The only difference between the baselines and our RFN model is the replacement of the joint learning layer—from the multivariate normal distribution assumption in the baseline models to the conditional CNF in RFN. By doing so, we can isolate and demonstrate the specific performance improvement brought by our proposed joint learning layer.
     
     The alternative construction of the marginal learning layer except GRUODE is listed in Appendix \ref{Appendix:Marginal learning blocks}. Hyperparameters, such as learning rate and batch size, are tuned using a random search for all experiments. The `dopri5' method is employed as the numerical solver for the ODE-based models. 

\subsection{Evaluation Metric}
We use the Continuous Ranked Probability Score (CRPS) \cite{matheson1976scoring} as the evaluation metric to measure the proximity between the learned distribution and the empirical distribution of data. The CRPS metric is the suitable scoring function for the evaluation since it attains the minimum when the learned distribution of data aligns with its empirical distribution. The CRPS is defined as 
\begin{align*}
	\operatorname{CRPS}(\hat{F}, x)=\int_{\mathbb{R}}(\hat{F}(z)-\mathbb{I}\{x \leq z\})^2 \mathrm{~d} z,
\end{align*}
where {\small$\hat{F}$} is the estimated CDF of random variable $X$ and $x$ is the realized observations of $X$. In our dataset, the observations are irregularly sampled, and some values are missing, and thus we only compute $\mathrm{CRPS}$  when the variables can be observed. 

Since $\mathrm{CRPS}$ targets to evaluate the estimated distribution of the univariate random variable, \cite{salinas2019high} suggests using $\mathrm{CRPS}_{\text {sum }}$ as a new proper multivariate-scoring rule to evaluate the estimated distribution of multivariate random variables. Specifically, we have

\begin{align*}
    \mathrm{CRPS}_{\text {sum }}=\mathbb{E}_t\biggl[\operatorname{CRPS}\biggl(\hat{F}_{\text {sum }}(t), \sum_{d=1}^D x_t^d\biggl)\biggl],
\end{align*}

where {\small$\hat{F}_{\text {sum }}$} is computed by adding the estimated CDF of all variables together. In practice, the empirical CDF is always used to represent $\hat{F}$ and {\small$\hat{F}_{\text {sum }}$}. In our experiments, we take 100 samples from trained models to estimate the empirical CDF. For Asyn-MTS datasets, some variables are not observable at an observation time; we thus only add the estimated CDF of observed variables as {\small$\hat{F}_{\text {sum }}$}. 

Apart from $\mathrm{CRPS}$ and $\mathrm{CRPS}_{\text {sum }}$, Kuleshov et al. \cite{kuleshov2018accurate} propose a more interpretive metric, termed confidence score ($\mathrm{CS}$),  to evaluate the calibration performance of the predicted distribution in the time series forecasting. The key idea behind this metric is to measure how the true frequency of observations across all time points matches the quantiles of the predicted distribution. For each variable $d$ at each time point $t$, the model predicts a distribution $F_t^d(\cdot)$. From this, we can compute the $q$-th quantile of the distribution, denoted by $Q_t^d(q)=\inf \{x_t^d \in \mathbb{R} \mid F_t^d(q) \ge x_t^d\}$. For a well-calibrated model, the proportion of actual observations $x_t^d$ smaller than the predicted $Q_t^d(q)$ should closely approximate $q$.
        
        Specifically, it first choose $m$ quantile levels $0<p_1<\cdots<p_m<1$, then for each threshold $p_j$, they compute the empirical frequency 
        \begin{align*}
        		\hat{p}_j^d=\frac{|\{{x_t^d: F_t^d(x_t^d) \le p_j, t=1, \cdots, T}|\}|}{T}.
        	\end{align*}
        Then, the confidence score is defined to describe the quality of forecast calibration $\mathrm{CS}=\frac{1}{mD}\sum_{j=1}^m\sum_{d=1}^D (p_j-\hat{p}_j^d)^2$.

\subsection{Simulation Experiment}\label{subsec:Ablation}
\begin{figure*}[tbh]
	\centering
	\subfigure[]{
		\includegraphics[width=0.99\columnwidth]{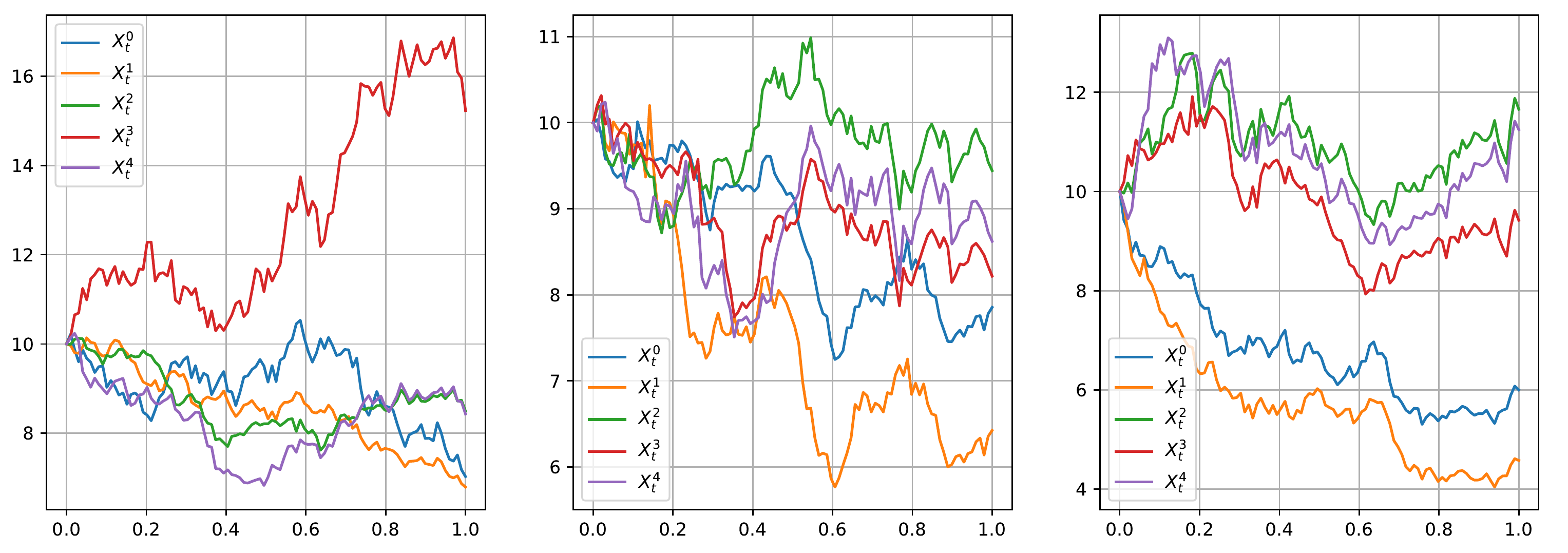}
		\label{fig:instances}
	}%
	\subfigure[]{
		\includegraphics[width=0.99\columnwidth]{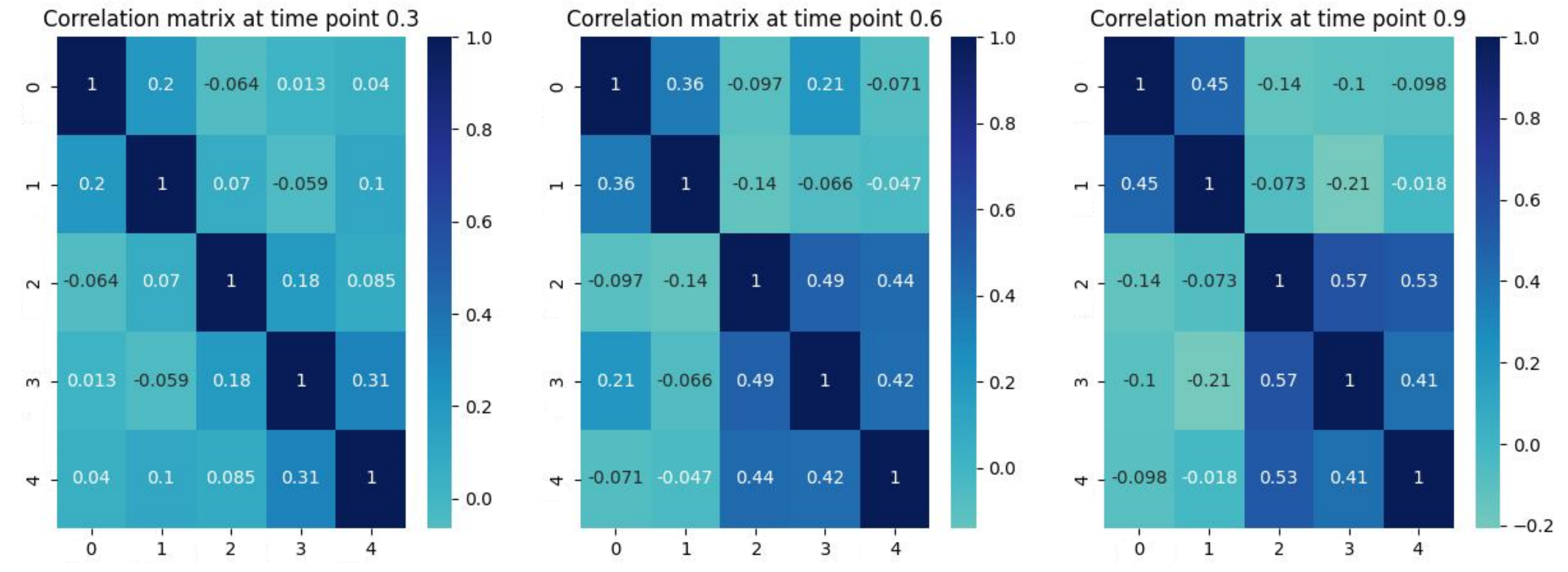}
		\label{fig:correlation}
	}%
	\caption{ (a) Three representative sample paths. (b) The sample correlation matrix at observed time points 0.3, 0.6, 0.9.}
\end{figure*}

\begin{figure*}[tbh]
	\centering
	\subfigure[]{
		\includegraphics[width=0.99\columnwidth]{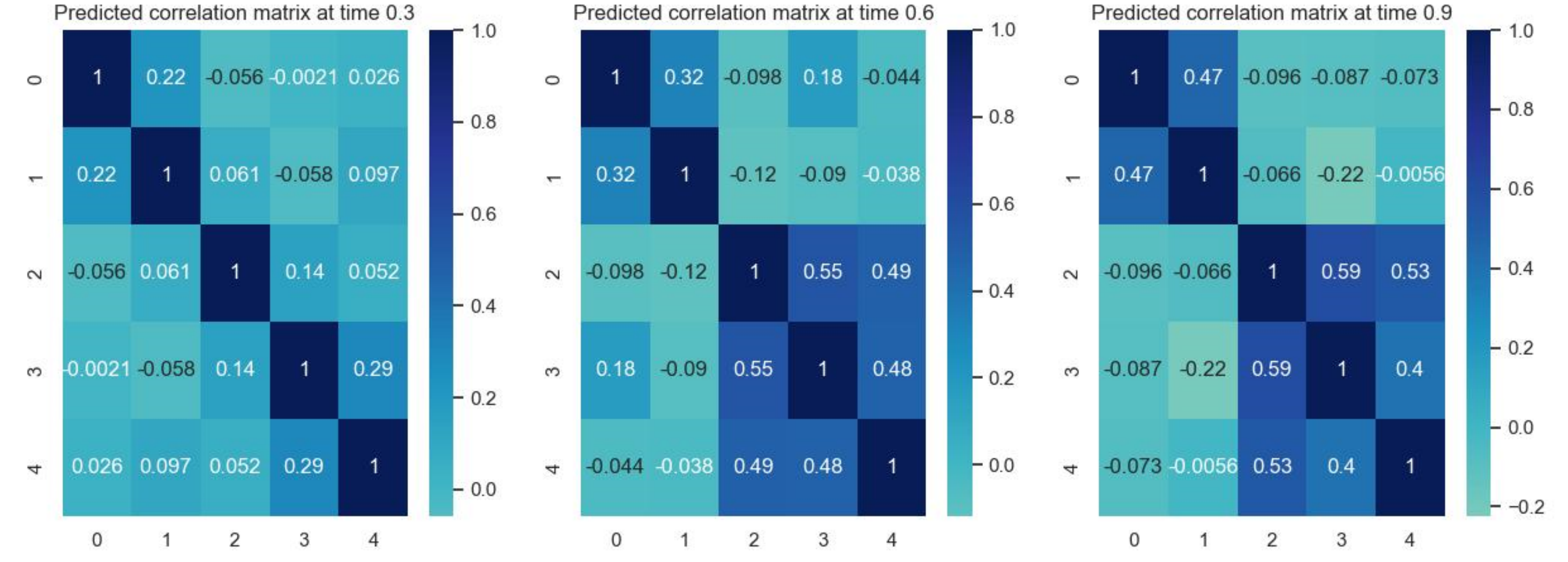}
		\label{fig:corr Syn-MTS}
	}%
	\subfigure[]{
		\includegraphics[width=0.99\columnwidth]{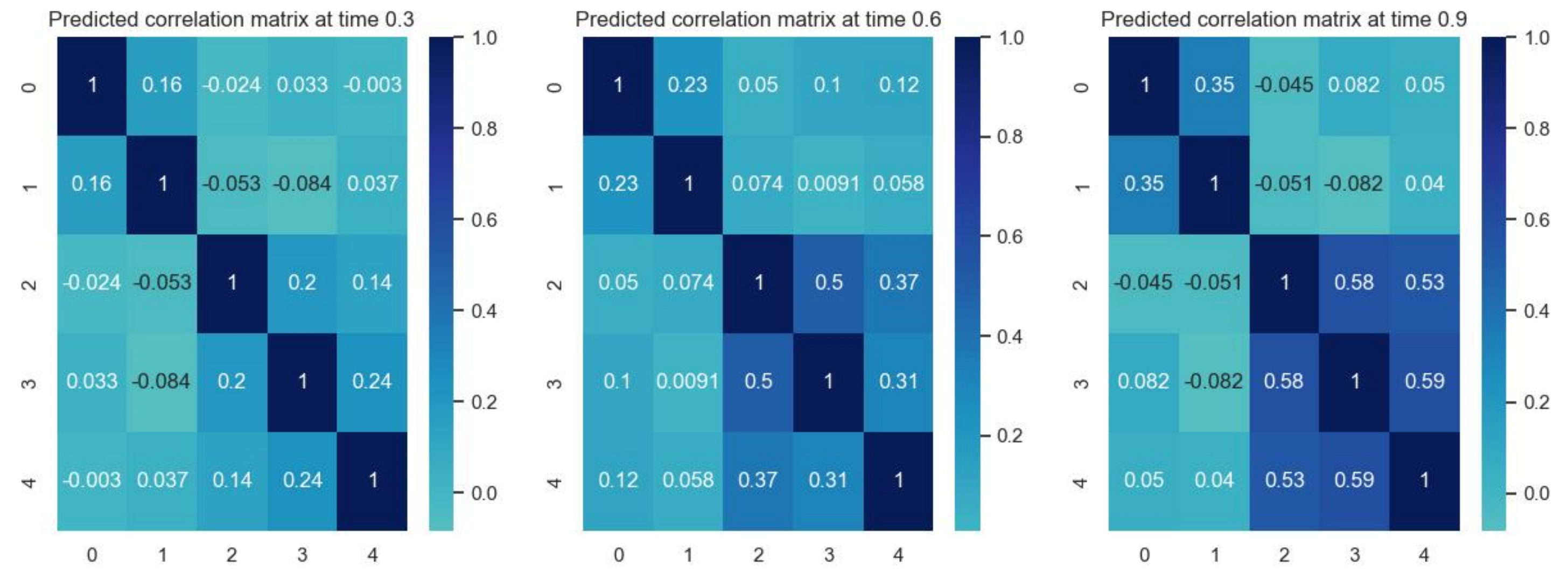}
		\label{fig:corr Asyn-MTS}
	}%
	\caption{The sample correlation matrix recovered from the learned Syn-MTS model (a) and the learned Asyn-MTS model (b) at 0.3, 0.6, and 0.9.}
\end{figure*}

{\textbf{(a) Data Generating Process}}. 
We set the Data Generation Process (DGP) $\mathrm{X_t}=[X^1_t,\cdots,X^5_t]$ as the following continuous-time correlated Geometric Brownian Motion (GBM) \cite{glasserman2004monte} 
\begin{align}
	\label{DGP}
	&d\mathrm{X}_t=\mathrm{\mu} \odot \mathrm{X}_t dt+ \text{diag}(\mathrm{\sigma} \odot \mathrm{X}_t) d\mathrm{W}_t, \, t \geq 0,
\end{align}
where $\mathrm{\mu}=[\mu^1,\cdots,\mu^5]^\top$ and $\sigma=[\sigma^1,\cdots,\sigma^5]^\top$ are the drift term and volatility term of variables, and $\mathrm{W}_t=[W^1_t,\cdots,W^5_t]^\top$ is the multivariate Brownian Motion process such that $ \langle dW_t^i, dW^j_t\rangle =\rho_{ij}(t)dt$ where
\begin{equation*}
	[\rho_{ij}(t)]_{1\leq i,j\leq 5}:=
		\begin{aligned}
			\sin\left(\frac{\pi}{2}t\right)\begin{bmatrix}
				1 & \rho_1 & 0 & 0 & 0 \\
				\rho_1 & 1 & 0 & 0 & 0 \\
				0 & 0 & 1 & \rho_2 & \rho_2 \\
				0 & 0 & \rho_2 & 1 & \rho_2 \\
				0 & 0 & \rho_2 & \rho_2  & 1\\
			\end{bmatrix}.
		\end{aligned}
\end{equation*}
We devise such a correlated GBM as the DGP for several reasons. Firstly, the GBM is a fundamental model in physics and asset prices in financial markets \cite{black1973pricing}. Secondly, the correlated GBM admits log-normal distribution for each variable, distinguishing it from the normal distribution. Thirdly, the correlation among the five variables is not set to static but time-varying, gradually increasing as time progresses from $0$ to $1$ due to the sinusoidal function $\sin(\frac{\pi}{2}t)$. This guarantees that the joint distribution of the five variables is constantly changing. Fig. \ref{fig:instances} showcases three representative sample paths from the dataset. Additionally, Fig. \ref{fig:correlation} illustrates the sample correlation matrix at observed time points 0.3, 0.6, and 0.9, validating the dynamic and gradual increase in correlation.

{\textbf{(b) Simulation Settings}}. We first simulate 1,000 sample paths using \eqref{DGP} with uniform time intervals. The Brownian parameters, i.e., the drift terms and the diffusion terms, of the 1,000 sample paths are set as follows: $\mu^1=\mu^2 \sim \text{Unif}(-0.2, -0.05), \mu^3=\mu^4=\mu^5 \sim \text{Unif}(0.05, 0.2), \sigma^1=\sigma^2 \sim \text{Unif}(0.15, 0.3), \sigma^3=\sigma^4=\sigma^5 \sim \text{Unif}(0.15, 0.3)$, where $\text{Unif}(\cdot,\cdot)$ stands for a uniform distribution.

We create the Syn-MTS and Asyn-MTS datasets through random sampling from the simulated dataset with complete observation. In the Syn-MTS dataset, we randomly choose half of the time points and sample the corresponding observations $\mathrm{X}_{t}$ at the chosen time points. To generate the Asyn-MTS dataset, we randomly eliminate half of the observations for each variable. The elimination is independent between any two variables. Therefore, some variables have observations at an observed time point, and some variables do not.

\begin{table*}[tbh]
	\centering
	\caption{\textbf{Geometric Brownian Motions:} Baseline models in vanilla form (2,3,4 columns) vs. using RFN specification (5,6,7 columns). \\ Numbers underneath $\mathrm{CRPS},\mathrm{CRPS}_{\text {sum }},\mathrm{CS}$ are displayed in the form of mean $\pm$ std, averaged over five experiments.	 \label{Sim}}
	\begin{tabular}{@{}cccccccc@{}}
		\toprule
		model   & $\mathrm{CRPS}$  & $\mathrm{CRPS}_{\text {sum }}$  & $\mathrm{CS}$ &  model        & $\mathrm{CRPS}$    & $\mathrm{CRPS}_{\text {sum }}$ & $\mathrm{CS}$\\ \midrule
		\multicolumn{8}{c}{\textit{Syn-MTS}}                                                                                            \\
		\midrule      
		GRUODE  & 0.2051 $\pm$ 0.0182 & 0.6785 $\pm$ 0.0120 &  0.0010 $\pm$ 0.0003 & RFN-GRUODE  & 0.1870 $\pm$ 0.0020 & 0.5869 $\pm$  0.0088 &  0.0007 $\pm$ 0.0002   \\
		ODELSTM & 0.2049 $\pm$ 0.0190 & 0.6697 $\pm$ 0.0120 &  0.0003 $\pm$ 0.0001 & RFN-ODELSTM & 0.1826 $\pm$ 0.0007 & 0.5795 $\pm$  0.0025 &  0.0002 $\pm$ 0.0001   \\
		ODERNN  & 0.2070 $\pm$ 0.0045 & 0.6777 $\pm$ 0.0157 &  0.0002 $\pm$ 0.0001 & RFN-ODERNN  & 0.1808 $\pm$ 0.0017 & 0.5744 $\pm$  0.0078 &  0.0001 $\pm$ 0.0001   \\
		GRU-D 	& 0.1921 $\pm$ 0.0130 & 0.6372 $\pm$ 0.0077 &  0.0015 $\pm$ 0.0006 & RFN-GRU-D   & 0.1838 $\pm$ 0.0007 & 0.5815 $\pm$  0.0071 &  0.0015 $\pm$ 0.0003   \\
		\midrule
		\multicolumn{8}{c}{\textit{Asyn-MTS}}                                                                                                  \\
		\midrule 
		GRUODE  & 0.2308 $\pm$ 0.0046 & 0.4565 $\pm$ 0.0125 &  0.0006 $\pm$ 0.0001 & RFN-GRUODE  & 0.2045 $\pm$ 0.0042 & 0.3827 $\pm$  0.0117 & 0.0003 $\pm$ 0.0002    \\
		ODELSTM & 0.2261 $\pm$ 0.0061 & 0.4677 $\pm$ 0.0326 &  0.0002 $\pm$ 0.0001 & RFN-ODELSTM & 0.1962 $\pm$ 0.0015 & 0.3650 $\pm$  0.0033 & 0.0001 $\pm$ 0.0001   \\
		ODERNN  & 0.2160 $\pm$ 0.0051 & 0.4283 $\pm$ 0.0144 &  0.0004 $\pm$ 0.0001 & RFN-ODERNN  & 0.1918 $\pm$ 0.0048 & 0.3597 $\pm$  0.0130 & 0.0001 $\pm$ 0.0001     \\
		GRU-D   & 0.2220 $\pm$ 0.0057 & 0.4420 $\pm$ 0.0162 &  0.0010 $\pm$ 0.0002 & RFN-GRU-D   & 0.1858 $\pm$ 0.0019 & 0.3405 $\pm$  0.0041 & 0.0007 $\pm$ 0.0001    \\ \bottomrule
	\end{tabular}
\end{table*}

\begin{figure}[tbh]
	\centering
	\subfigure[Syn-MTS]{
		\includegraphics[width=0.43\linewidth]{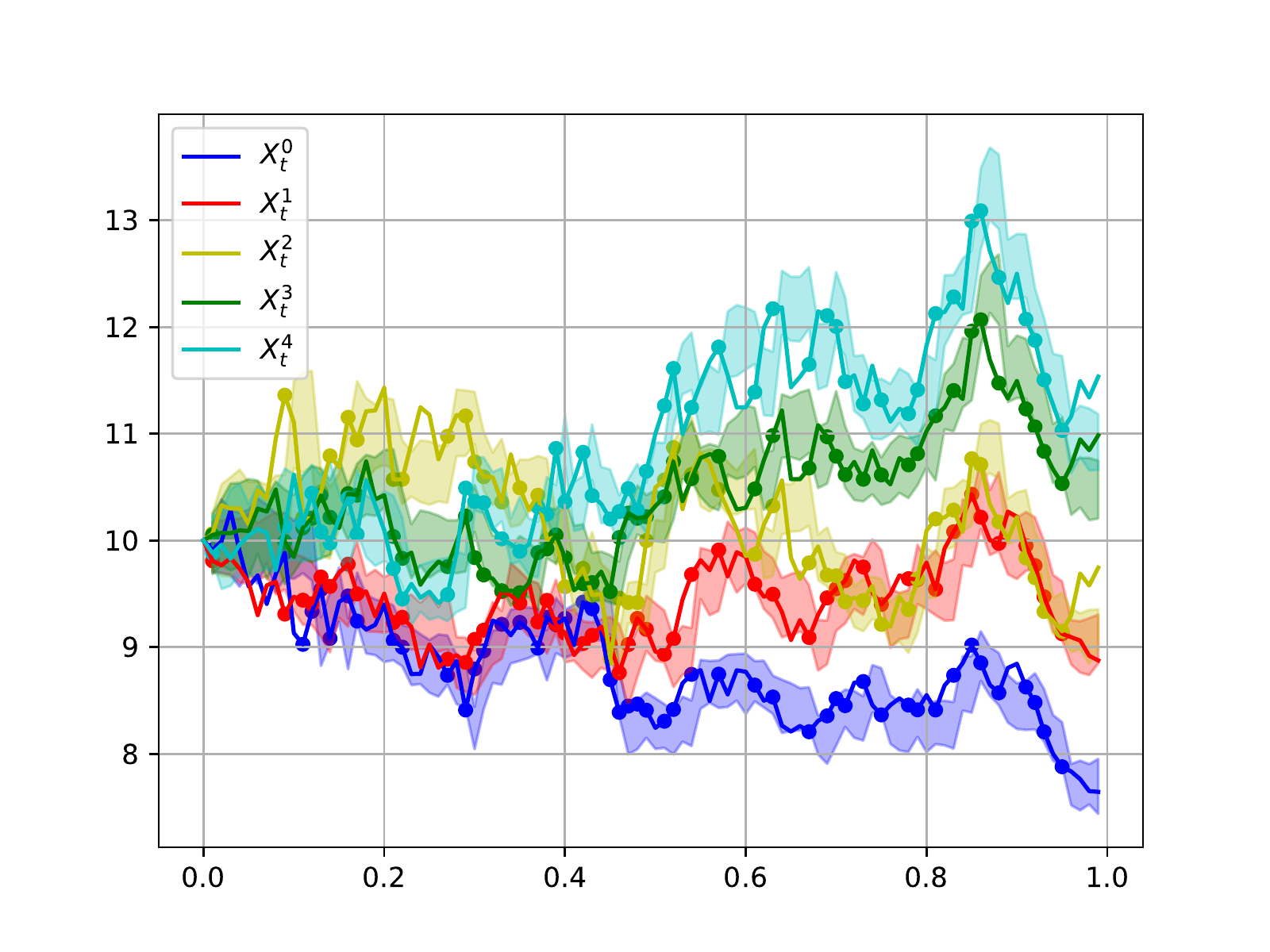}
		\label{fig:Syn-MTS est interval}
	}
	\subfigure[Asyn-MTS]{
		\includegraphics[width=0.43\linewidth]{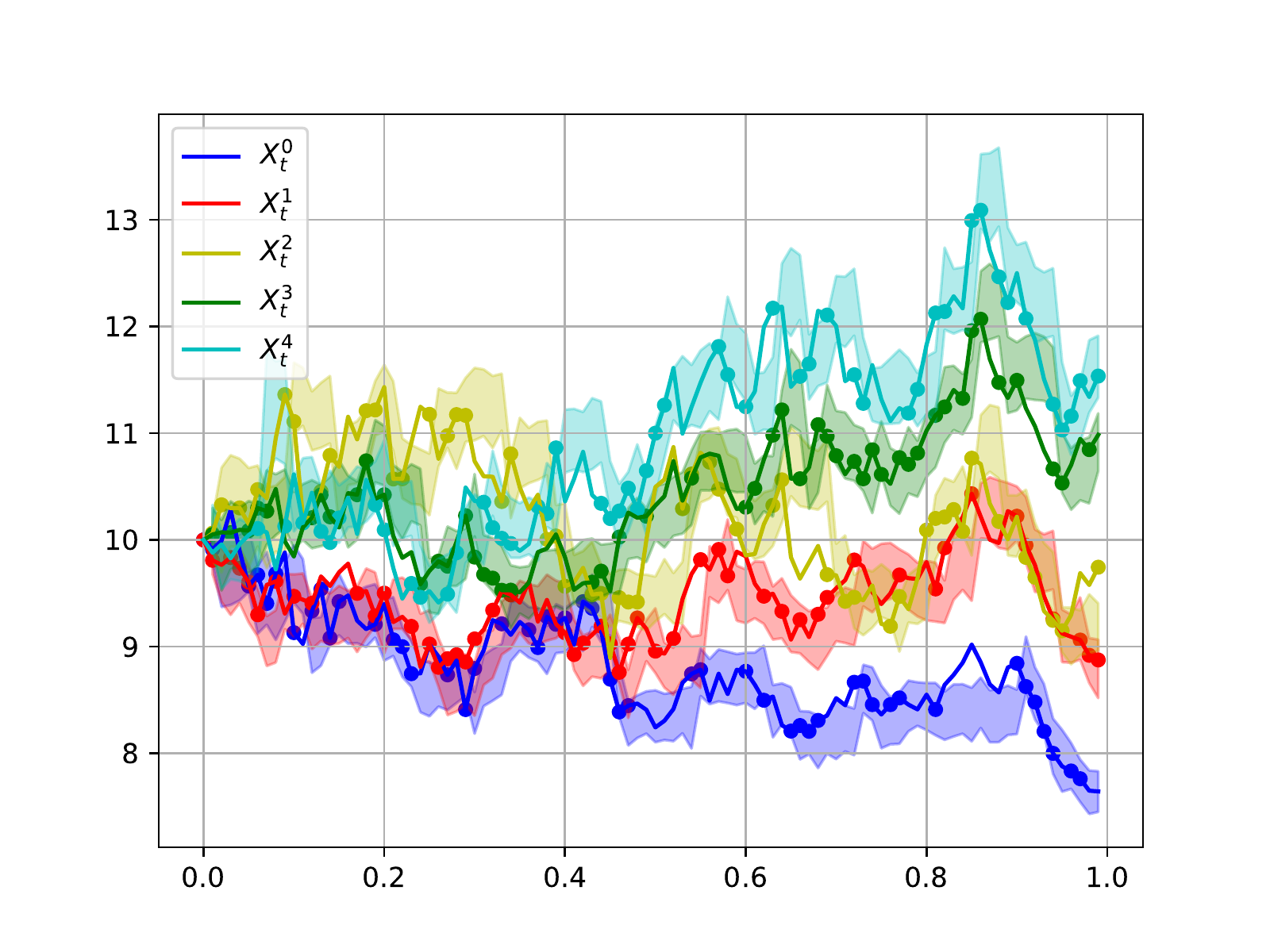}
		\label{fig:Asyn-MTS est interval}
	}
	\caption{Predicted quantile intervals (shaded area) from 20\% to 80\% of Syn-MTS data (a) and Asyn-MTS data (b) using the RFN-GRUODE specification. Solid dots represent observations. Broader quantile ranges are available in Appendix \ref{Appendix:More results}.}  
\end{figure}

\begin{table*}[tbh]
	\centering
	\caption{\textbf{Physical Activities of Human Body (MuJoCo).} (Same report formats as in Table 1) 	\label{Mujoco}}
	\begin{tabular}{@{}cccccccc@{}}
		\toprule
		model   & $\mathrm{CRPS}$  & $\mathrm{CRPS}_{\text {sum }}$  & $\mathrm{CS}$ & model        & $\mathrm{CRPS}$    & $\mathrm{CRPS}_{\text {sum }}$ & $\mathrm{CS}$ \\ \midrule
		\multicolumn{8}{c}{\textit{Syn-MTS}}                                                                                                 \\ \midrule
		GRUODE  & 0.2198 $\pm$ 0.0035 & 1.0763 $\pm$ 0.0352 & 0.0117 $\pm$ 0.0022 & RFN-GRUODE  & 0.1858 $\pm$ 0.0037 & 0.6790 $\pm$ 0.0300 & 0.0097 $\pm$ 0.0033    \\
		ODELSTM & 0.2256 $\pm$ 0.0015 & 1.1402 $\pm$ 0.0298 & 0.0117 $\pm$ 0.0026 & RFN-ODELSTM & 0.1736 $\pm$ 0.0017 & 0.7103 $\pm$ 0.0504 & 0.0060 $\pm$ 0.0010    \\
		ODERNN  & 0.2156 $\pm$ 0.0042 & 1.0450 $\pm$ 0.0393 & 0.0143 $\pm$ 0.0016 & RFN-ODERNN  & 0.1747 $\pm$ 0.0019 & 0.6467 $\pm$ 0.0340 & 0.0087 $\pm$ 0.0025  \\
		GRU-D   & 0.2224 $\pm$ 0.0032 & 1.1118 $\pm$ 0.0177 & 0.0119 $\pm$ 0.0014 & RFN-GRU-D   & 0.1905 $\pm$ 0.0010 & 0.6748 $\pm$ 0.0093 & 0.0084 $\pm$ 0.0007   \\ \midrule
		\multicolumn{8}{c}{\textit{Asyn-MTS}}                                                                                                  \\ \midrule
		GRUODE  & 0.2695 $\pm$ 0.0054 & 1.0661 $\pm$ 0.0249 & 0.0024 $\pm$ 0.0001 & RFN-GRUODE  & 0.1970 $\pm$ 0.0052 & 0.7306 $\pm$ 0.0185 & 0.0007 $\pm$ 0.0002  \\
		ODELSTM & 0.2728 $\pm$ 0.0057 & 1.0834 $\pm$ 0.0292 & 0.0013 $\pm$ 0.0002 & RFN-ODELSTM & 0.1733 $\pm$ 0.0048 & 0.6318 $\pm$ 0.0155 & 0.0007 $\pm$ 0.0002  \\
		ODERNN  & 0.2696 $\pm$ 0.0036 & 1.0748 $\pm$ 0.0194 & 0.0020 $\pm$ 0.0005 & RFN-ODERNN  & 0.1667 $\pm$ 0.0020 & 0.6277 $\pm$ 0.0086 & 0.0012 $\pm$ 0.0004   \\
		GRU-D   & 0.2261 $\pm$ 0.0035 & 0.9078 $\pm$ 0.0255 & 0.0116 $\pm$ 0.0008 & RFN-GRU-D   & 0.1709 $\pm$ 0.0035 & 0.6409 $\pm$ 0.0151 & 0.0058 $\pm$ 0.0009   \\ \bottomrule
	\end{tabular}
\end{table*}

\begin{table*}[tbh]
	\centering
	\caption{\textbf{Climate Records of Weather (USHCN).} (Same report formats as in Table 1) 	\label{USHCN}}
	\begin{tabular}{@{}cccccccc@{}}
		\toprule
		model   & $\mathrm{CRPS}$   & $\mathrm{CRPS}_{\text {sum }}$  & $\mathrm{CS}$ & model       & $\mathrm{CRPS}$   & $\mathrm{CRPS}_{\text {sum }}$  & $\mathrm{CS}$ \\ \midrule
		
		GRUODE  & 0.3012 $\pm$ 0.0082 & 0.5797 $\pm$ 0.0112 & 0.0269 $\pm$ 0.0043 & RFN-GRUODE  & 0.2463 $\pm$ 0.0039 & 0.5183 $\pm$ 0.0072 & 0.0075 $\pm$ 0.0052    \\
		ODELSTM & 0.3113 $\pm$ 0.0049 & 0.5909 $\pm$ 0.0075 & 0.0256 $\pm$ 0.0029 & RFN-ODELSTM & 0.2629 $\pm$ 0.0043 & 0.5510 $\pm$ 0.0090 & 0.0071 $\pm$ 0.0037    \\
		ODERNN  & 0.3112 $\pm$ 0.0088 & 0.5923 $\pm$ 0.0157 & 0.0285 $\pm$ 0.0011 & RFN-ODERNN  & 0.2626 $\pm$ 0.0073 & 0.5499 $\pm$ 0.0146 & 0.0051 $\pm$ 0.0046   \\
		GRU-D   & 0.3016 $\pm$ 0.0143 & 0.5798 $\pm$ 0.0186 & 0.0243 $\pm$ 0.0038 & RFN-GRU-D   & 0.2415 $\pm$ 0.0016 & 0.5093 $\pm$ 0.0025 & 0.0094 $\pm$ 0.0051    \\
		\bottomrule
	\end{tabular}
\end{table*}

\begin{table*}[tbh]
	\centering
	\caption{\textbf{Transaction Records of Eight Biotechnology Stocks (NASDAQ).} (Same report formats as in Table 1) 	\label{TCH}}
	\begin{tabular}{@{}cccccccc@{}}
		\toprule
		model   & $\mathrm{CRPS}$   & $\mathrm{CRPS}_{\text {sum }}$  & $\mathrm{CS}$ & model       & $\mathrm{CRPS}$   & $\mathrm{CRPS}_{\text {sum }}$  &  $\mathrm{CS}$ \\ \midrule
		GRUODE  & 0.0163 $\pm$ 0.0028 & 0.0578 $\pm$ 0.0219 & 0.0021 $\pm$ 0.0009 & RFN-GRUODE  & 0.0149 $\pm$ 0.0008 & 0.0514 $\pm$ 0.0039  & 0.0007 $\pm$ 0.0006   \\
		ODELSTM & 0.0163 $\pm$ 0.0021 & 0.0529 $\pm$ 0.0085 & 0.0013 $\pm$ 0.0010 & RFN-ODELSTM & 0.0150 $\pm$ 0.0009 & 0.0506 $\pm$ 0.0056  & 0.0004 $\pm$ 0.0002  \\
  	ODERNN  & 0.0166 $\pm$ 0.0009 & 0.0564 $\pm$ 0.0047 & 0.0031 $\pm$ 0.0020 & RFN-ODERNN  & 0.0150 $\pm$ 0.0008 & 0.0521 $\pm$ 0.0036  & 0.0021 $\pm$ 0.0009   \\  
		GRU-D   & 0.0068 $\pm$ 0.0007 & 0.0270 $\pm$ 0.0016 & 0.0102 $\pm$ 0.0042 & RFN-GRU-D   & 0.0064 $\pm$ 0.0003 & 0.0244 $\pm$ 0.0005  & 0.0049 $\pm$ 0.0054   \\
 \bottomrule
	\end{tabular}
\end{table*}

\begin{figure*}[tbh]
	\centering
	\subfigure[MuJoCo Asyn-MTS]{
		\includegraphics[width=0.23\linewidth, trim=20 0 45 40, clip]{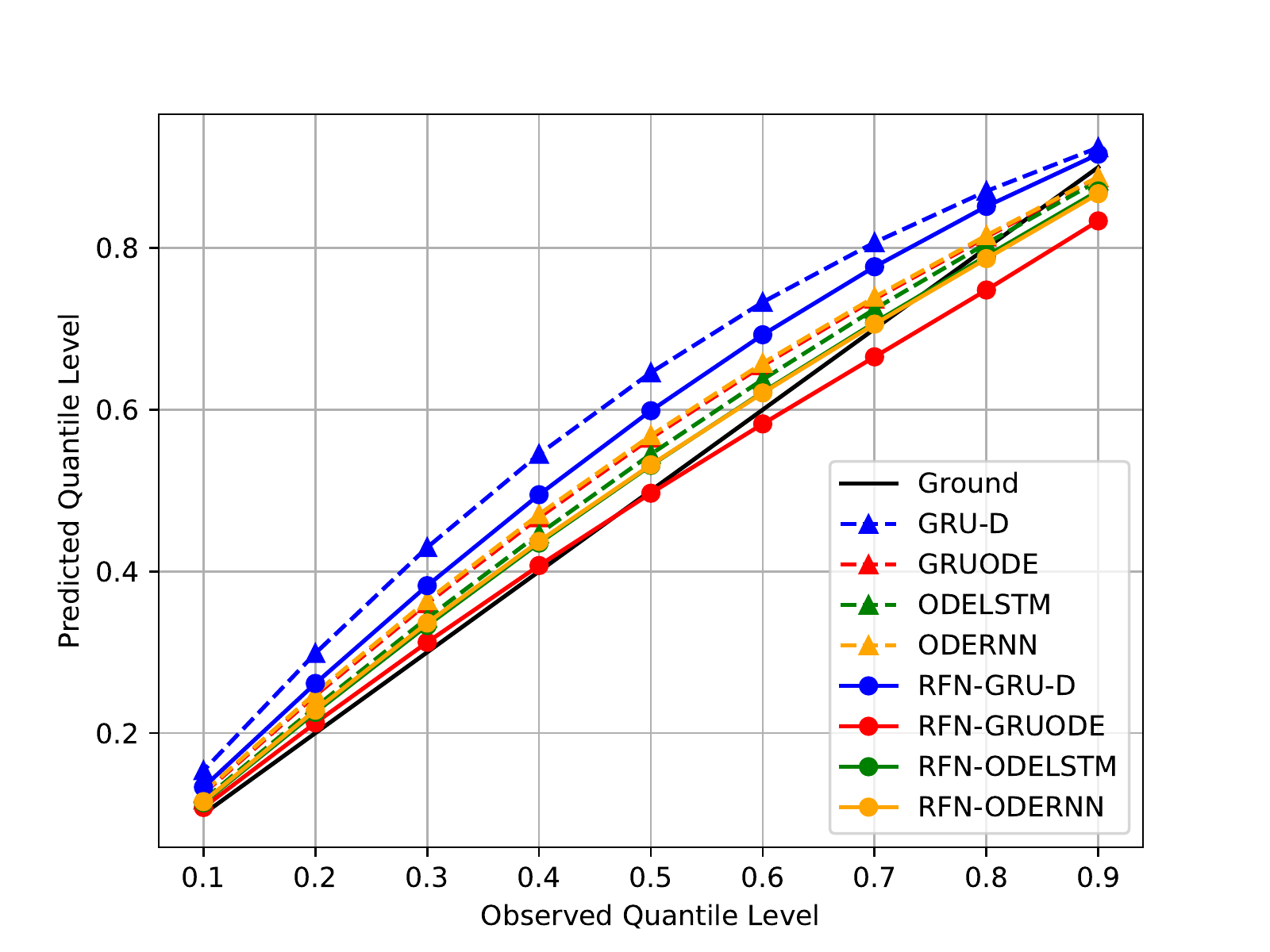}
		\label{fig:Hopper_async_cs}
	}
	\subfigure[MuJoCo Syn-MTS]{
		\includegraphics[width=0.23\linewidth, trim=20 0 45 40, clip]{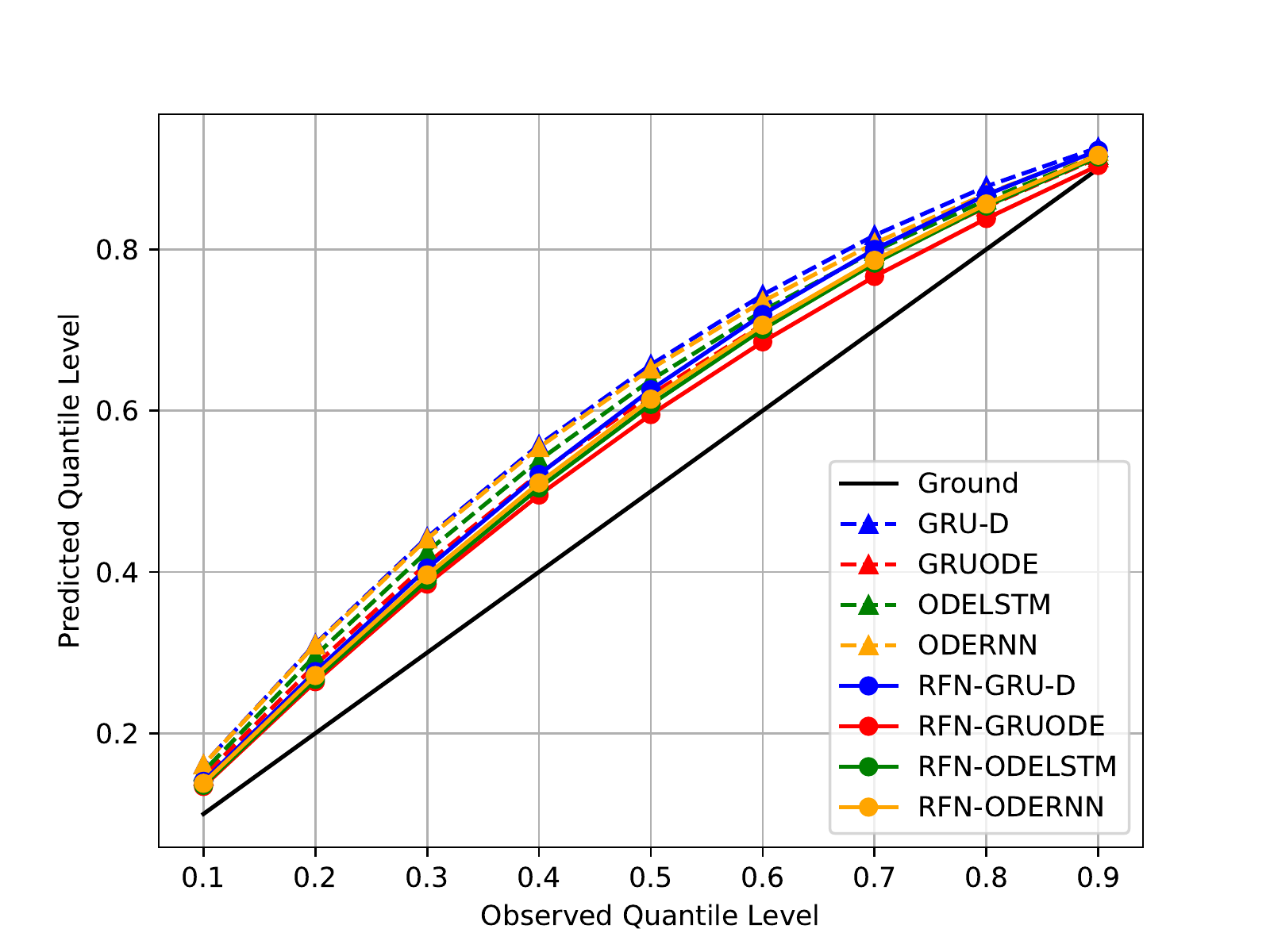}
		\label{fig:Hopper_sync_cs}
	}
    \subfigure[USHCN Asyn-MTS]{
		\includegraphics[width=0.23\linewidth, trim=20 0 45 40, clip]{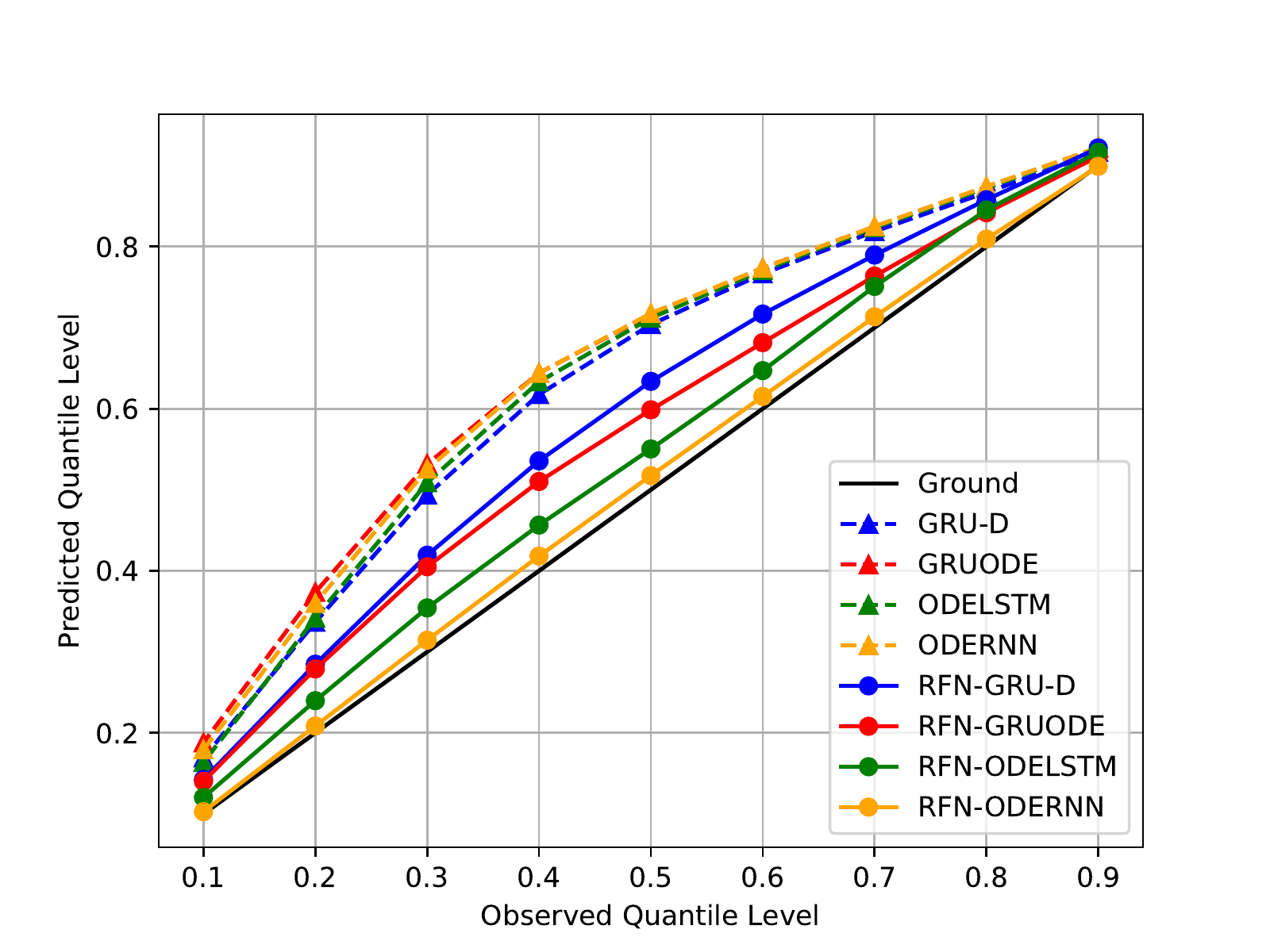}
		\label{fig:USHCN_async_cs}
	}
    \subfigure[NASDAQ Asyn-MTS]{
		\includegraphics[width=0.23\linewidth, trim=20 0 45 40, clip]{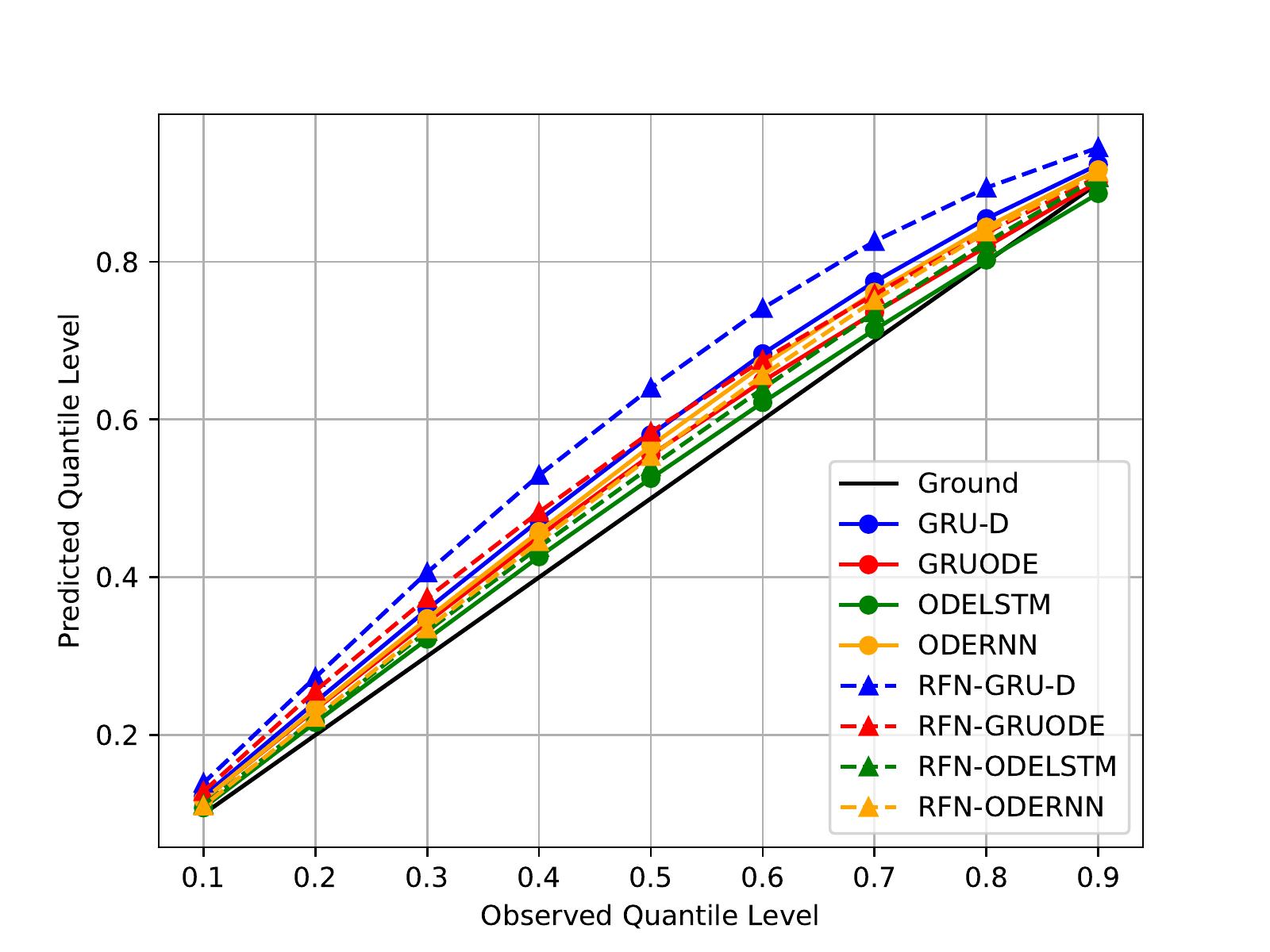}
		\label{fig:NASDAQ_async_cs}
	}
	\caption{The observed quantile level v.s. predicted quantile level of three real-world datasets. The dash lines are the baseline models, and the corresponding solid lines are the RFN counterparts. The theoretical best result is the solid black line, and the closer the model is to it, the better the performance is. \label{fig:cs_plot}}  
\end{figure*}

{\textbf{(c) Data and Model Validation}}. The experiment results in Table \ref{Sim} demonstrate that RFNs outperform the corresponding baselines with smaller $\text{CRPS}$, $\text{CRPS}_{\text{sum}}$, and $\text{CS}$. This indicates the superior ability of the RFNs to capture the joint distribution of multivariate irregular time series at each observed time point and recover its dynamics. Fig. \ref{fig:corr Syn-MTS} and \ref{fig:corr Asyn-MTS} illustrate the correlation matrices sampled from the estimated Syn-MTS and Asyn-MTS models at time points 0.3, 0.6, and 0.9, closely resembling the ground truth in Fig. \ref{fig:correlation}. The correlation coefficients from the recovered matrices gradually transition from weak to strong as time progresses, aligning with the ground truth in Fig. \ref{fig:correlation}. This dynamic capture of the joint distribution validates the effectiveness of the RFN specification. Lastly, we employ the RFNs for interval estimation of correlated GBM processes. Fig. \ref{fig:Syn-MTS est interval} and Fig. \ref{fig:Asyn-MTS est interval} display prediction intervals for Syn-MTS and Asyn-MTS. 

\begin{figure*}[tbh]
	\centering
	\subfigure[t=0.09]{
		\includegraphics[width=.07\linewidth]{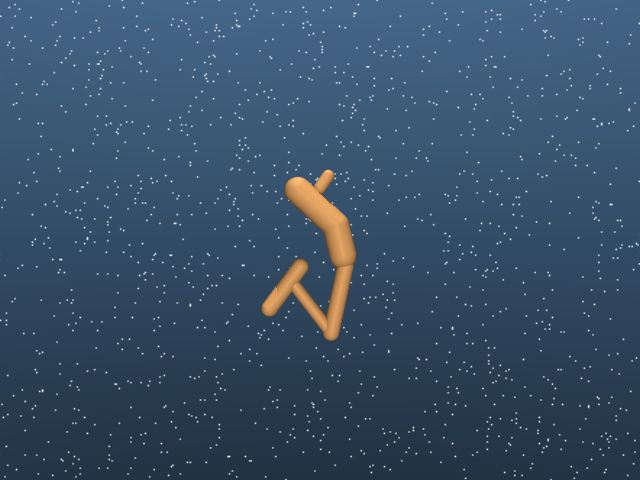}
	}%
	\subfigure[t=0.21]{
		\includegraphics[width=.07\linewidth]{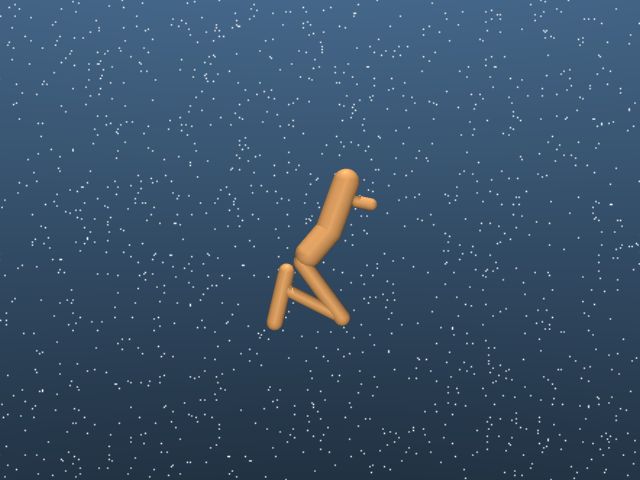}
	}%
	\subfigure[t=0.36]{
		\includegraphics[width=.07\linewidth]{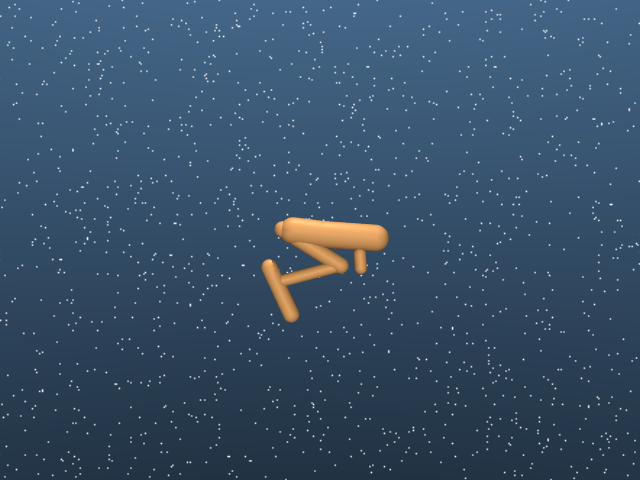}
	}%
	\subfigure[t=0.47]{
		\includegraphics[width=.07\linewidth]{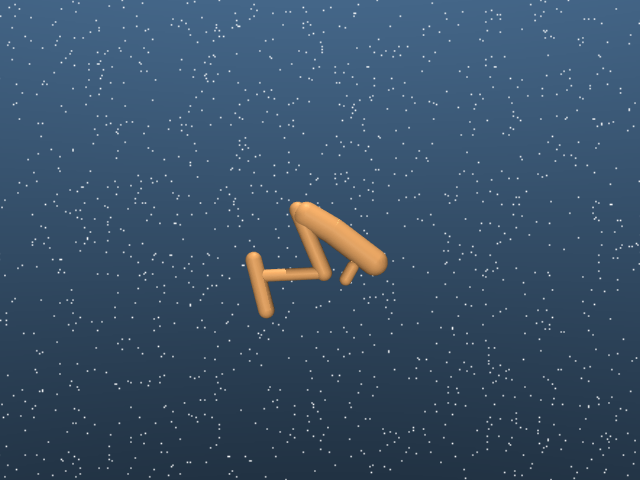}
	}%
	\subfigure[t=0.67]{
		\includegraphics[width=.07\linewidth]{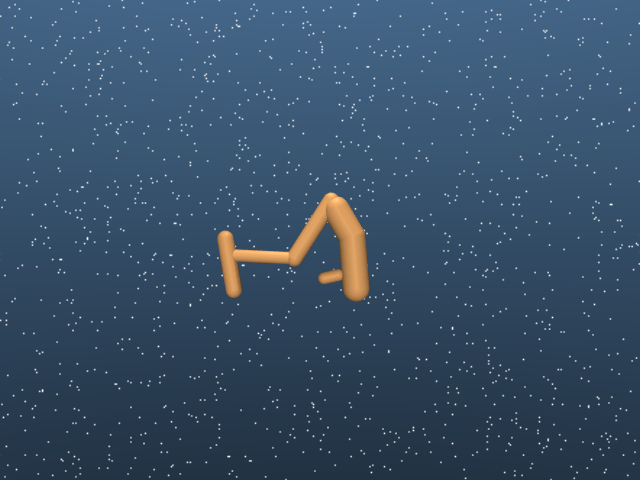}
	}%
	\subfigure[t=0.70]{
		\includegraphics[width=.07\linewidth]{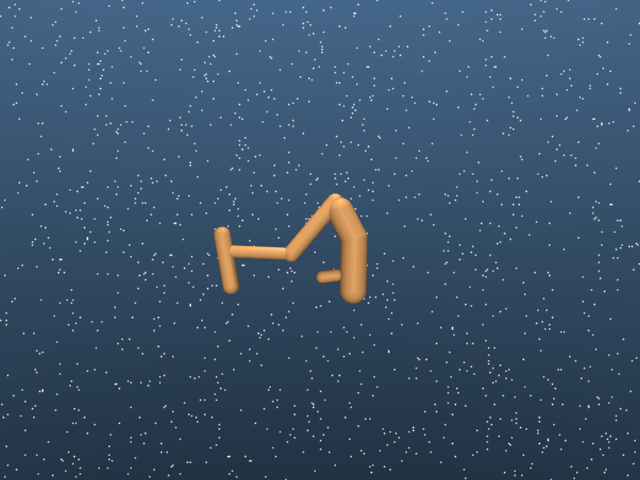}
	}%
	\subfigure[t=0.87]{
		\includegraphics[width=.07\linewidth]{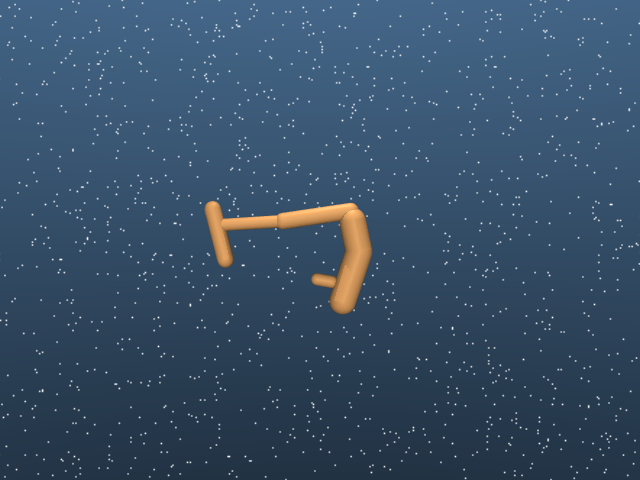}
	}%
	\subfigure[t=0.88]{
		\includegraphics[width=.07\linewidth]{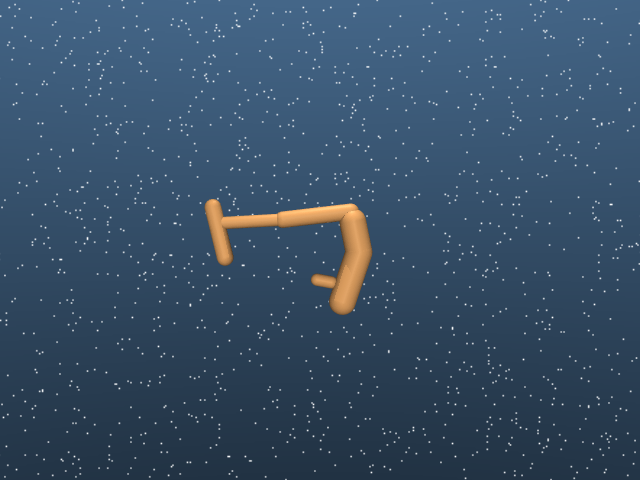}
	}%
	\subfigure[t=1.03]{
		\includegraphics[width=.07\linewidth]{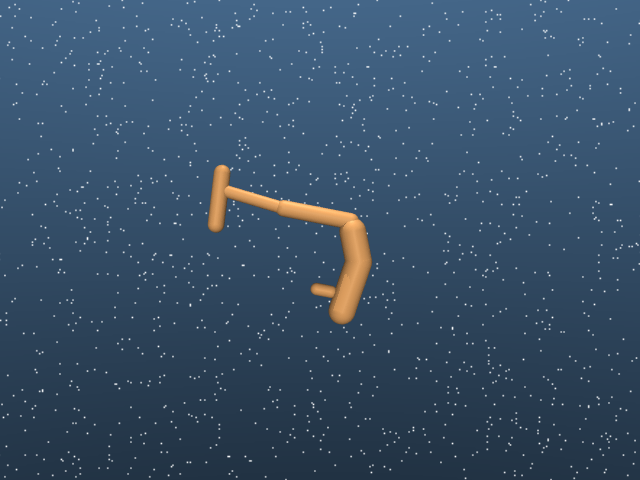}
	}%
	\subfigure[t=1.17]{
		\includegraphics[width=.07\linewidth]{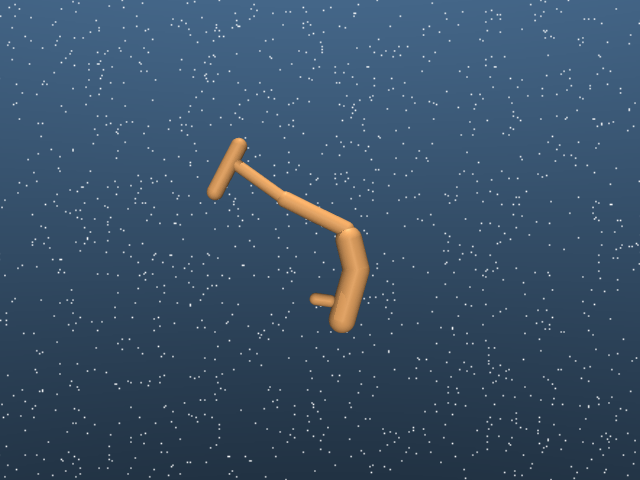}
	}%
	\qquad
	\subfigure[t=1.40]{
		\includegraphics[width=.07\linewidth]{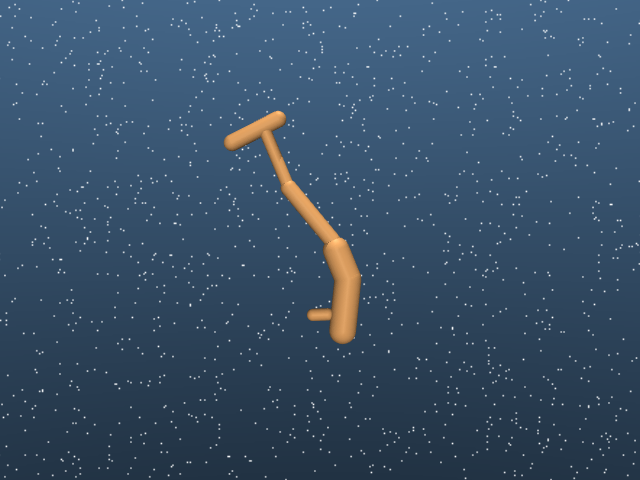}
	}%
	\\
	\subfigure{
		\includegraphics[width=.07\linewidth]{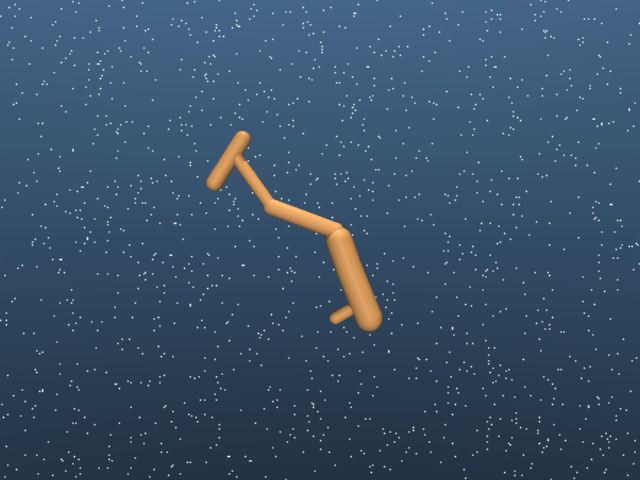}
	}%
	\subfigure{
		\includegraphics[width=.07\linewidth]{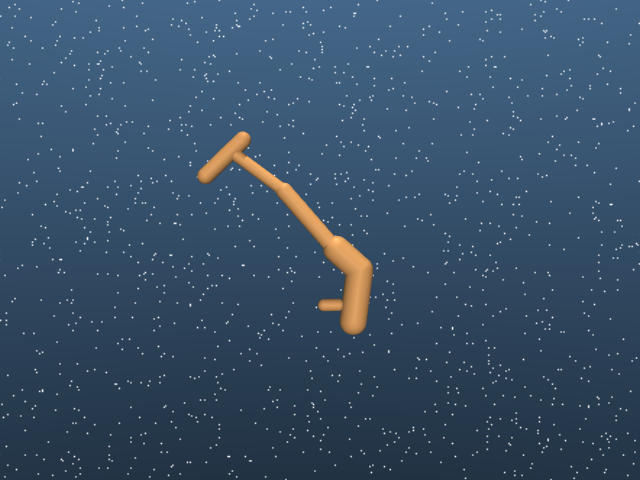}
	}%
	\subfigure{
		\includegraphics[width=.07\linewidth]{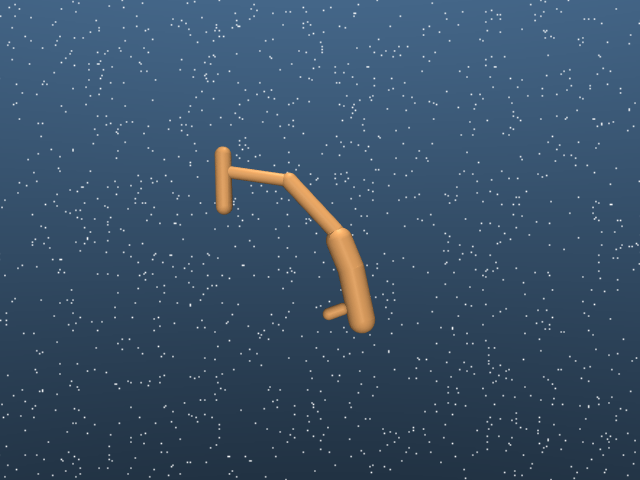}
	}%
	\subfigure{
		\includegraphics[width=.07\linewidth]{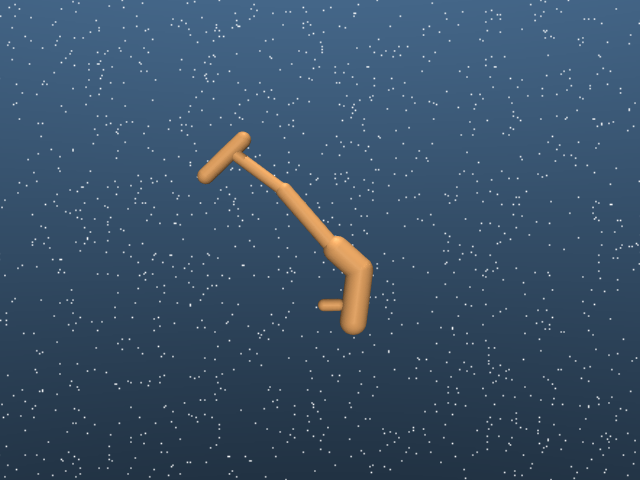}
	}%
	\subfigure{
		\includegraphics[width=.07\linewidth]{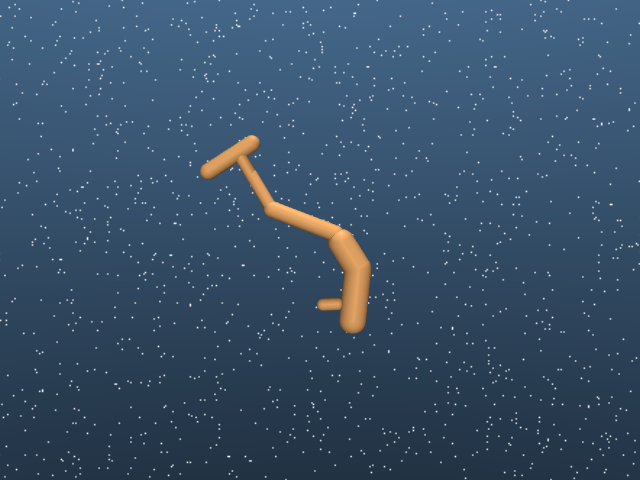}
	}%
	\subfigure{
	\includegraphics[width=.07\linewidth]{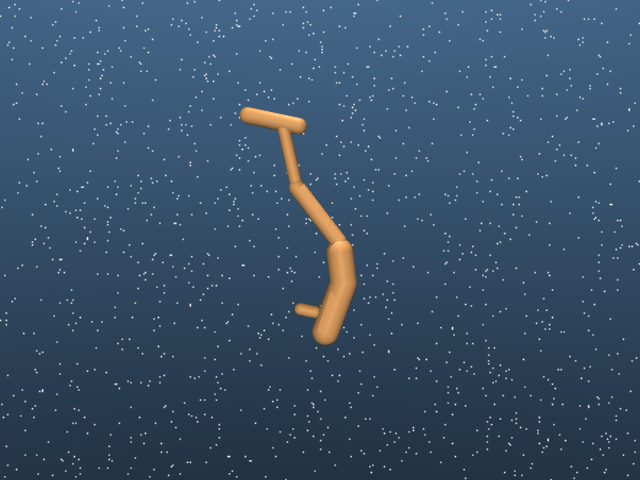}
	}%
	\subfigure{
		\includegraphics[width=.07\linewidth]{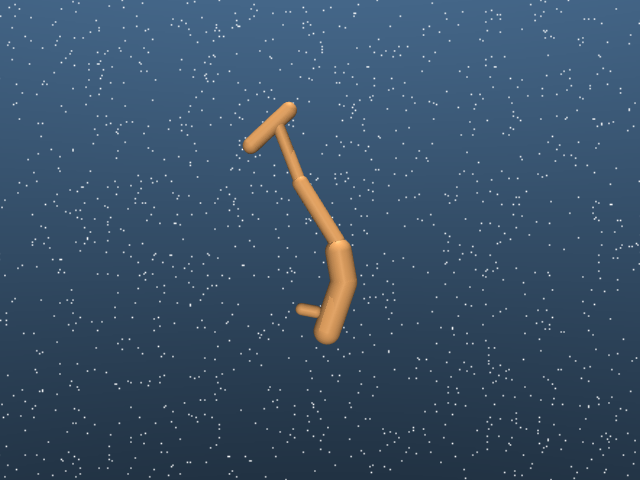}
	}%
	\subfigure{
		\includegraphics[width=.07\linewidth]{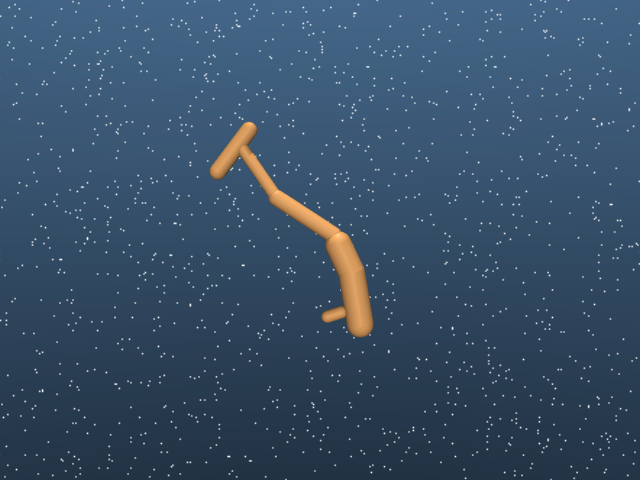}
	}%
	\subfigure{
		\includegraphics[width=.07\linewidth]{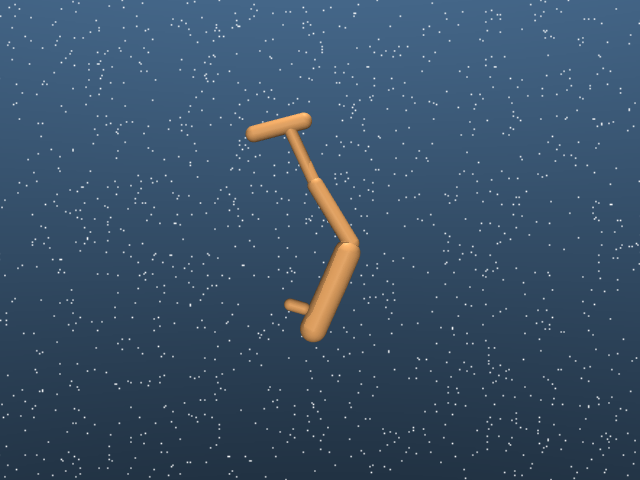}
	}%
	\subfigure{
		\includegraphics[width=.07\linewidth]{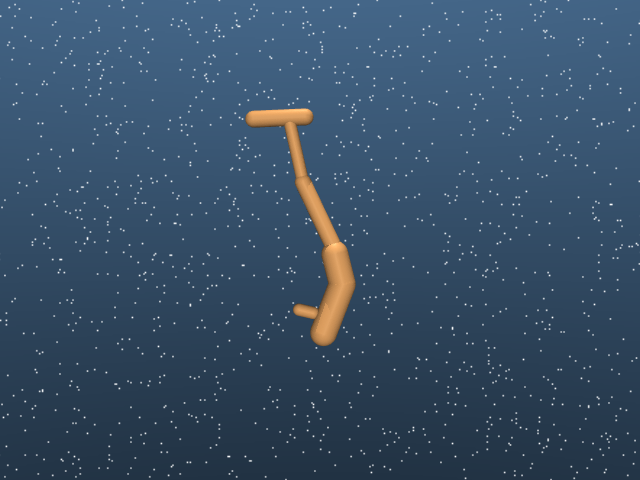}
	}%
	\qquad
	\subfigure{
		\includegraphics[width=.07\linewidth]{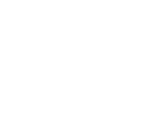}
	}
	\caption{(a)-(j) represent data observed at historical times with unevenly spaced intervals. The objective is to estimate the data distribution at t=1.40. Fig. (k) displays the observed ground truth at t=1.40. We train the Syn-MTS and Asyn-MTS models to forecast the desired distribution at t=1.40. The first and last five figures in the second row display five samples generated from the trained Syn-MTS model (Asyn-MTS model) at t=1.40. \label{fig:MuJoCo}} 
\end{figure*}

\subsection{Dataset of Physical Activities (MuJoCo)}
In applications such as robotics, multivariate time series are commonly used to capture the position, speed, and trajectory of objects, exhibiting a high degree of correlation among variables. However, irregular and asynchronous frequency of measurements frequently arises due to limitations in measurement devices. This experiment tests the performance of using the RFN specification to predict object positions against baseline models in their vanilla forms.

The MuJoCo physics dataset is introduced in \cite{rubanova2019latent} to verify that the ODE-based models can learn an approximation of Newtonian physics. The dataset is created by the ``Hopper'' model from the Deepmind Control Suite \cite{tassa2018deepmind}. It contains 14 variables in total, where the first 7 variables and the last 7 variables control the position and the velocities.

 By following the generation and preprocessing of \cite{rubanova2019latent}, we randomly sample the initial position and velocities of ``Hopper'' such that the hopper rotates in the air and falls on the ground. We generate 5,000 instances with 150 observed time points for each instance, and randomly sample 50\% values to generate the Syn-MTS and Asyn-MTS datasets. 

Table \ref{Mujoco} shows that the RFNs perform robustly better than the baselines in their vanilla form. As a validation, we use the trained model to forecast the future joint distribution at $t=1.40$ based on a given instance's past ten observations (ranging from $t=0.09$ to $t=1.17$). Subsequently, we generate 10 samples from the estimated distribution at $t=1.40$ and compare them with the ground truth at $t=1.40$. Fig. \ref{fig:MuJoCo} displays the 10 samples at $t=1.40$ from both the Asyn-MTS model and the Syn-MTS model. They all resemble the ground truth closely, affirming the RFN's proficiency in predicting the joint distribution.

In Fig. \ref{fig:Hopper_async_cs} and \ref{fig:Hopper_sync_cs}, we present a comparison based on the CS metric to understand the model’s reliability. The $x$-axis represents the observed quantile level, while the $y$-axis shows the predicted quantile level. The solid black line represents the theoretical best performance, where the predicted quantile exactly matches the observed quantile, i.e., the $x$-axis equals the $y$-axis. The closer a model’s performance is to this black line, the better it is calibrated. We plot the results for our model (solid lines) alongside the baseline models (dashed lines).  We observe that almost all the dashed lines fall upper and to the left of the solid lines, indicating that our models outperform the baseline models.

\subsection{Dataset of Climate Records (USHCN)}
Variables in natural phenomena, such as climate data, also display strong serial and cross-correlations \cite{du2019deep}. For instance, temperature values from one season provide valuable information about temperature patterns in the following season. Additionally, precipitation, such as rain and snow, can cause temperature drops and affect humidity levels. Missing observations are also common in climate data due to inclement weather or equipment malfunctions. In this experiment, we evaluate the performance of the RFN specification on the USHCN dataset.

The United States Historical Climatology Network (USHCN) dataset \cite{menne2010long} consists of daily measurements from 1,218 centers across the country. It includes 5 variables: precipitation, snowfall, snow depth, maximum temperature, and minimum temperature. Following the preprocessing approach of \cite{de2019gru}, we select the training data from the first quarter of the last four years (1996-2000). This yields 4,494 instances, with each instance's time period normalized from 0 to 12.5 at intervals of 0.1. The observations for each variable have uneven spacing, making the dataset asynchronous. Consequently, we exclusively employ the Asyn-MTS model to train this dataset. 

The results presented in Table \ref{USHCN} demonstrate the superior performance of the RFN specification over its vanilla counterpart for all baselines when applied to the climate dataset. These findings reinforce that climate data does not conform to a multivariate Gaussian distribution and highlight the importance of capturing the dependence structure among variables for accurate weather forecasting.

\subsection{Dataset of Stock Transactions (NASDAQ)} \label{stock options}
The capability to forecast stock prices holds importance for investors, offering a strategic advantage in the financial markets \cite{zhao2022stock}. A transaction in the stock market is a match between a buy order and a sell order. Since traders send their buy and sell orders to exchange at random times, the time intervals between transactions are inherently random. 

In this section, we utilize the minute-by-minute transaction records of eight biotech stocks from the NASDAQ exchange from July 26, 2016, to April 28, 2017, across 191 trading days \cite{qin2017dual}. Their ticker symbols are `BIIB', `BMRN', `CELG', `REGN', `VRTX', `GILD', `INCY', and `MYL'.  As they are companies from the same sector, their stock prices tend to move together, and these movements are strongly correlated. 

We partition the 191 trading days into 993 instances, each of which is a 75-minute-long multivariate instance.  The dataset contains instances of missing values, making it a multivariate asynchronous time series.  Utilizing the Asyn-MTS model, we simultaneously forecast the joint distribution of the stock prices of these eight stocks. Table \ref{TCH} shows that the RFN specification consistently outperforms its non-RFN counterpart in every one of the models.

\begin{figure*}[htb]
	\centering
	\subfigure[Missing Rate (Syn-MTS)]{
		\includegraphics[width=0.23\linewidth, trim=5 15 15 15, clip]{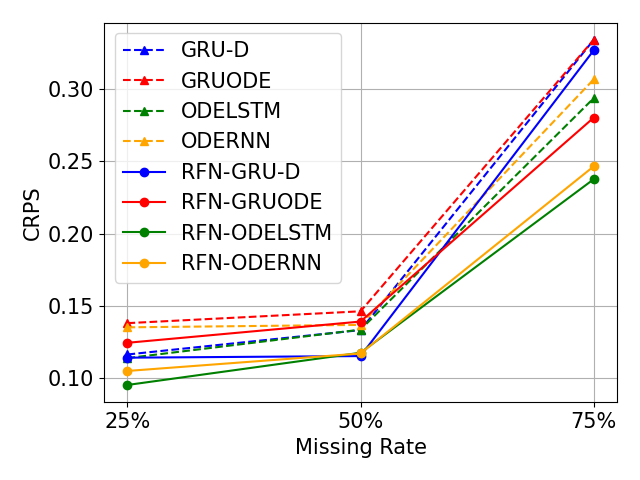}
		\label{fig:missing_rate_sync}
	}
	\subfigure[Missing Rate (Asyn-MTS)]{
		\includegraphics[width=0.235\linewidth, trim=5 15 15 15, clip]{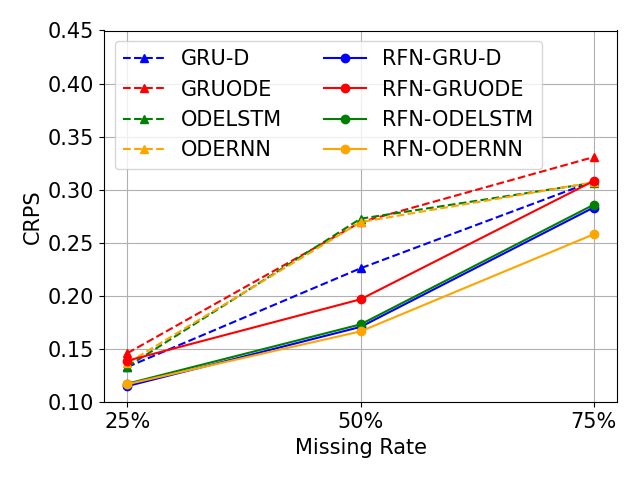}
		\label{fig:missing_rate_async}
	}	\subfigure[Hidden Size of Conditional CNF]{
		\includegraphics[width=0.23\linewidth, trim=5 15 15 15, clip]{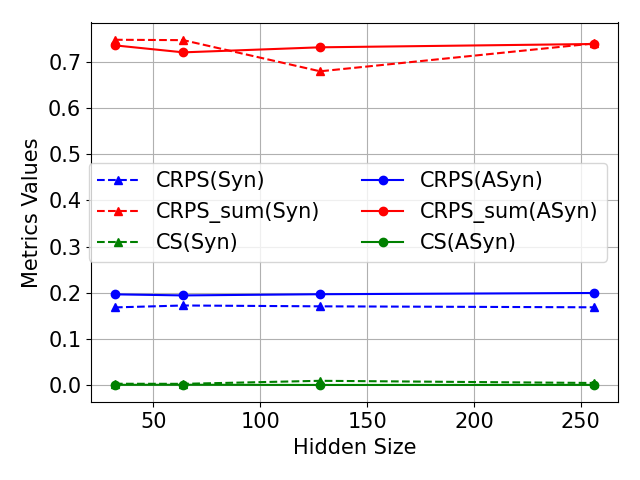}
		\label{fig:hidden}
	}	\subfigure[Memory Size of Sequential Model]{
		\includegraphics[width=0.23\linewidth, trim=5 15 20 15, clip]{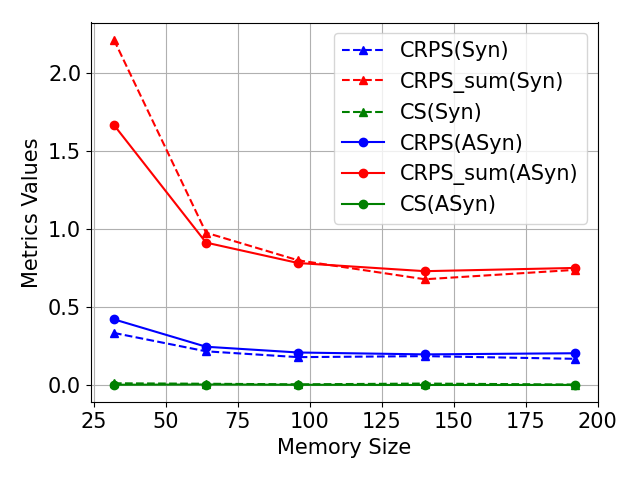}
		\label{fig:memory}
	}
	\caption{The sensitivity analysis across the missing rate,  hidden size of conditional CNF (RFN-GRUODE), the memory size of the sequential model (RFN-GRUODE), on Physical Activities (MuJoCo) dataset.  \label{fig:sensitivity}}  
\end{figure*}

\subsection{Sensitivity Analysis}
In this section, we present additional experiments to analyze our model's performance under various scenarios. In the main paper, we randomly sampled 50\% of the values to generate the Syn-MTS and Asyn-MTS datasets. To further test the robustness of our model, we vary the missing rate of the MuJoCo dataset from 25\% to 75\% to create new Syn-MTS and Asyn-MTS datasets. 

The CRPS performance is illustrated in Fig. \ref{fig:missing_rate_sync} and \ref{fig:missing_rate_async}, where dashed lines represent the baseline and solid lines represent the RFN specifications. The results indicate that RFN outperforms all baseline models across all missing rates. In addition, we observe that, for the Syn-MTS dataset, performance decreases slightly as the missing rate increases from 25\% to 50\%, but declines significantly when the missing rate rises from 50\% to 75\%. This suggests substantial information loss when more than half of the time points are unobserved. In contrast, the Asyn-MTS dataset shows a steady decrease in performance as the missing rate increases, possibly because missing information in one variable can be compensated by observations of other variables.

Furthermore, we explore the impact of varying the hidden size of the conditional CNF from 32 to 256, with performance results shown in Fig. \ref{fig:hidden}.  The performance remains relatively stable across different hidden sizes, indicating that the hidden size of the flow model has minimal influence on the overall performance, demonstrating the robustness of the proposed model. Similarly, we analyze the effect of various memory sizes of the hidden state in the sequential models, as shown in Fig. \ref{fig:memory}.  We observe that increasing the hidden state size results in performance improvements due to the model's enhanced ability to capture long-term dependencies. However, as the hidden size increases further, the rate of improvement diminishes.

\section{Discussions \& Conclusion}

In this paper, we propose an end-to-end learning framework, termed RFN, to address the challenges posed by multivariate irregular time series. The RFN is structured into two core components: a marginal learning layer and a joint learning layer. The marginal learning layer processes the multivariate time series—whether synchronous or asynchronous—by leveraging state-of-the-art sequential models to capture the temporal dynamics at each observed time point. The joint learning layer then models the joint distribution of the variables at each time step using the proposed conditional CNF model.

One of the key innovations of our framework is its ability to overcome the restrictive Gaussian assumption commonly made in time series modeling. By incorporating the dynamic conditional CNF model, our approach facilitates non-parametric learning of the joint distribution at arbitrary continuous-time points, thereby accommodating both temporal and cross-sectional dependencies in complex, irregular time series data.

We extensively evaluate the performance of the RFN framework across several real-world datasets from diverse domains. Our experimental results demonstrate that RFNs consistently outperform state-of-the-art models that assume Gaussian distributions for modeling dependencies in irregularly sampled time series. These results underscore the robustness and effectiveness of the RFN framework, making it a powerful tool for handling the challenges of irregular time series data in various application domains.

However, we acknowledge several challenges in applying RFN to certain types of datasets. Fisrt, if cross-sectional dependencies are weak, the model’s advantage in learning inter-variable relationships diminishes. Second. when applied to datasets with a large proportion of missing values (e.g., greater than 85\%), the model’s ability to learn reliable dependencies is significantly hindered since extremely sparse data may disrupt temporal and cross-sectional relationships.

Our RFN framework also has some limitations. First, the computational cost of training a CNF-based model is higher than that of simpler Gaussian-based alternatives due to the need for solving continuous-time flow dynamics, see Table \ref{efficiency}. Second, the increased flexibility of normalizing flows introduces additional complexity in optimizing the transformation dynamics, which can lead to unstable gradients and longer training times. To mitigate this issue, we incorporate regularization techniques, such as those proposed in \citep{finlay2020train}, to accelerate convergence and improve stability.
        
\section*{Acknowledgment}
Qi WU acknowledges the support from The CityU-JD Digits Joint Laboratory in Financial Technology and  Engineering, The Hong Kong Research Grants Council [General Research Fund 11219420/9043008 ], and The CityU APRC Grant 9610643. The work described in this paper was partially supported by the InnoHK initiative, the  Government of the HKSAR, and the Laboratory for AI-Powered Financial Technologies.


%

\ifCLASSOPTIONcaptionsoff
  \newpage
\fi



\bibliographystyle{IEEEtran}
\bibliography{IEEEexample}
%

%

\begin{IEEEbiography}[{\includegraphics[width=1in,height=1.25in,clip,keepaspectratio]{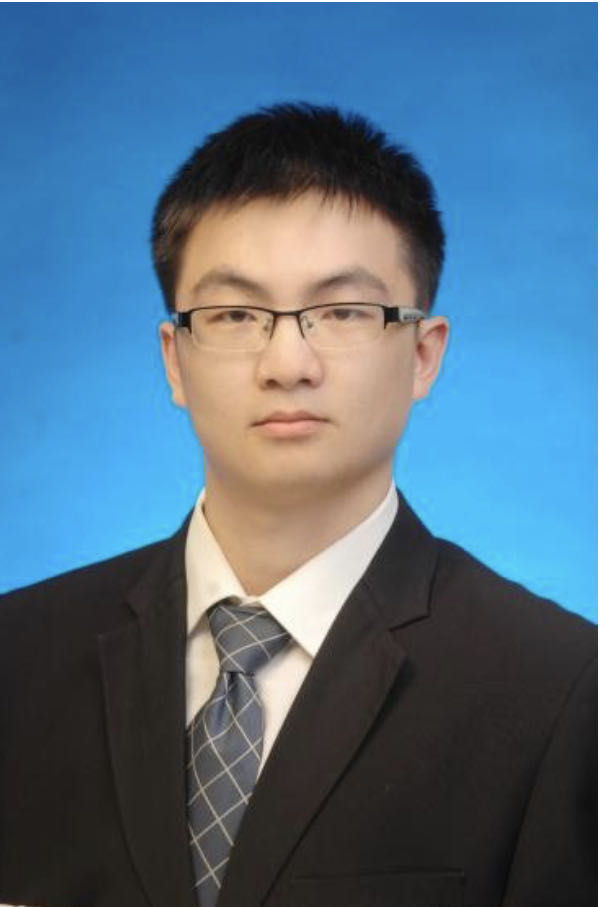}}]{Yijun Li}
received the Ph.D. degree in data science from the City University of Hong Kong, Hong Kong, in 2024. He is currently a postdoctoral researcher at the Laboratory for AI-Powered Financial Technologies Limited, Hong Kong. His research interests include time series forecasting, causal inference, and financial technology. 
\end{IEEEbiography}

\begin{IEEEbiography}[{\includegraphics[width=1in,height=1.25in,clip,keepaspectratio]{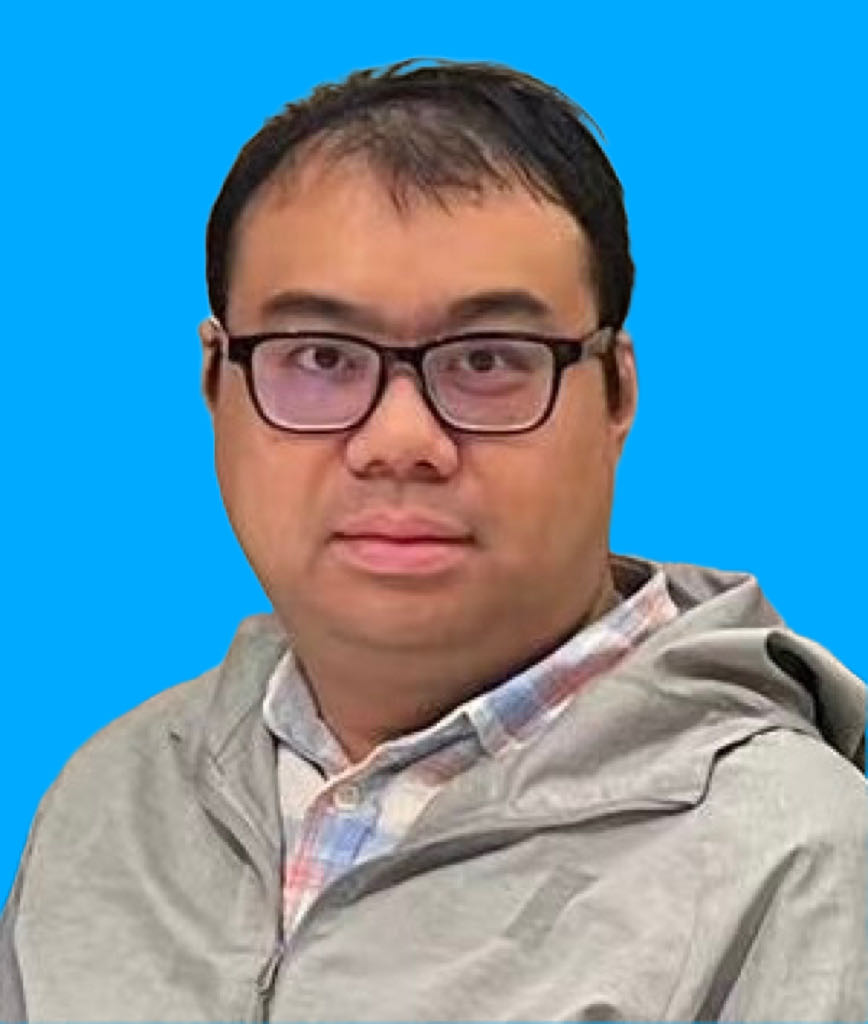}}]{Cheuk Hang Leung}
received the BSc. degree in mathematics from the University of Hong Kong,  in 2009, the MSc degree from the Hong Kong University of Science and Technology, in 2012, and the PhD degree from the Chinese University of Hong Kong, in 2018. He is now working at the City University of Hong Kong JD Digits Joint Laboratory in Financial Technology and Engineering. 
\end{IEEEbiography}

\begin{IEEEbiography}[{\includegraphics[width=1in,height=1.25in,clip,keepaspectratio]{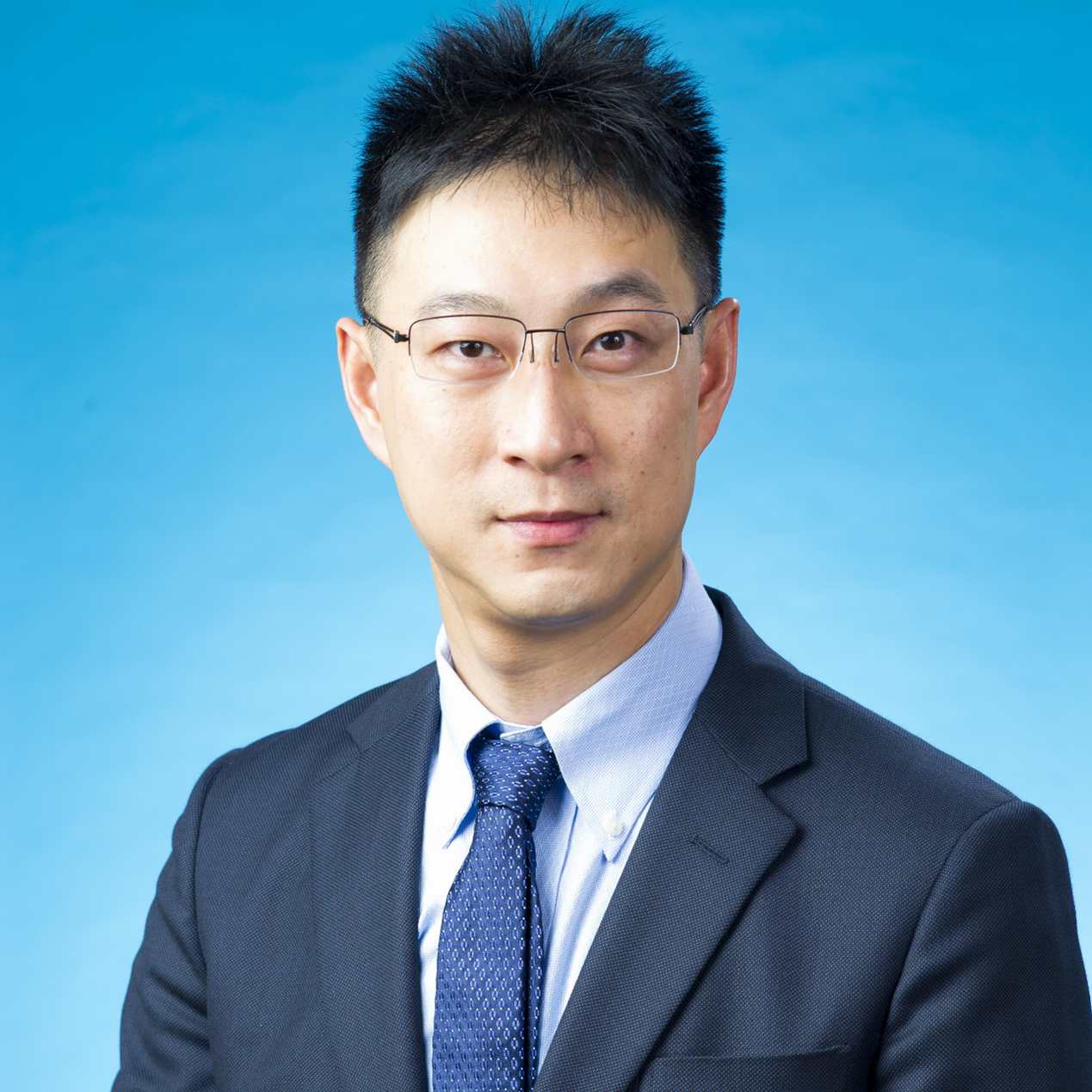}}]{Qi Wu}
    received the Ph.D. degree in applied mathematics from Columbia University, New York, NY, USA, in 2013. He is currently an Associate Professor with the School of Data Science, City University of Hong Kong (CityU), Hong Kong. He is also the Program Leader at Laboratory for AI-Powered Financial Technologies Ltd., Hong Kong. Prior to CityU, he was an Assistant Professor at The Chinese University of Hong Kong and a Finance Practitioner at Lehman Brothers, London, U.K.; UBS, Stamford, CT, USA; and DTCC, New York.
\end{IEEEbiography}





\clearpage

\appendices
\section{Approaches on Handling Temporal Irregularity}
\subsection{Discretization}
        Temporal discretization is a fundamental technique for converting irregularly sampled time series data into a regularly sampled format, as illustrated in Fig. \ref{fig:discretization}. This approach involves first defining a sequence of \( J + 1 \) evenly spaced time points, denoted by \( \tau_0, \cdots, \tau_J \), where each discretization interval \( \tau_j - \tau_{j-1} \) is of equal length for all \( j \in \{1, \cdots, J\} \).
        
        Given a multivariate irregular time series \( \mathbf{x} \), we use temporal discretization to construct a corresponding regular time series \( \mathbf{x}' \), accompanied by a missing data mask \( \mathbf{m}' \). A mapping is then applied to assign observed values from \( \mathbf{x} \) within each discretization window \( [\tau_{j-1}, \tau_j) \) to the vectors \( \mathbf{x}' \) and \( \mathbf{m}' \), accounting for two cases.
        
        In the first case, for dimension \( d \), if there are observed values \( x_t^d \) within the window \( t \in [\tau_{j-1}, \tau_j) \), these observations are aggregated into a single representative value, denoted by \( (x')^d_{t'} \), where \( t'_j = (\tau_{j-1} + \tau_j) / 2 \). A common aggregation approach is to use the mean of these values. We then set \( m^d_{t'_j} = 1 \) to indicate the presence of an observed value within this interval. In the second case, if there are no observations of \( x_t^d \) within the window \( [\tau_{j-1}, \tau_j) \), we set \( (m')^d_{t'_j} = 0 \) to indicate the absence of data.
        
        This discretization method effectively transforms the problem of modeling irregularly sampled time series into that of modeling regularly sampled series with potentially missing values. As the discretization interval \( \tau_j - \tau_{j-1} \) is widened, the amount of missing data typically decreases; however, more values may fall within the same interval, leading to increased aggregation and potential loss of local information that could be critical for certain tasks. Conversely, with narrower discretization intervals, aggregation effects are minimized, yet the length of the time series increases, along with the proportion of missing data. Thus, the window size in temporal discretization is a crucial hyperparameter that must be carefully tuned based on the specific modeling needs. 
        
\begin{figure}[ht]
    \centering
    \includegraphics[width=\linewidth]{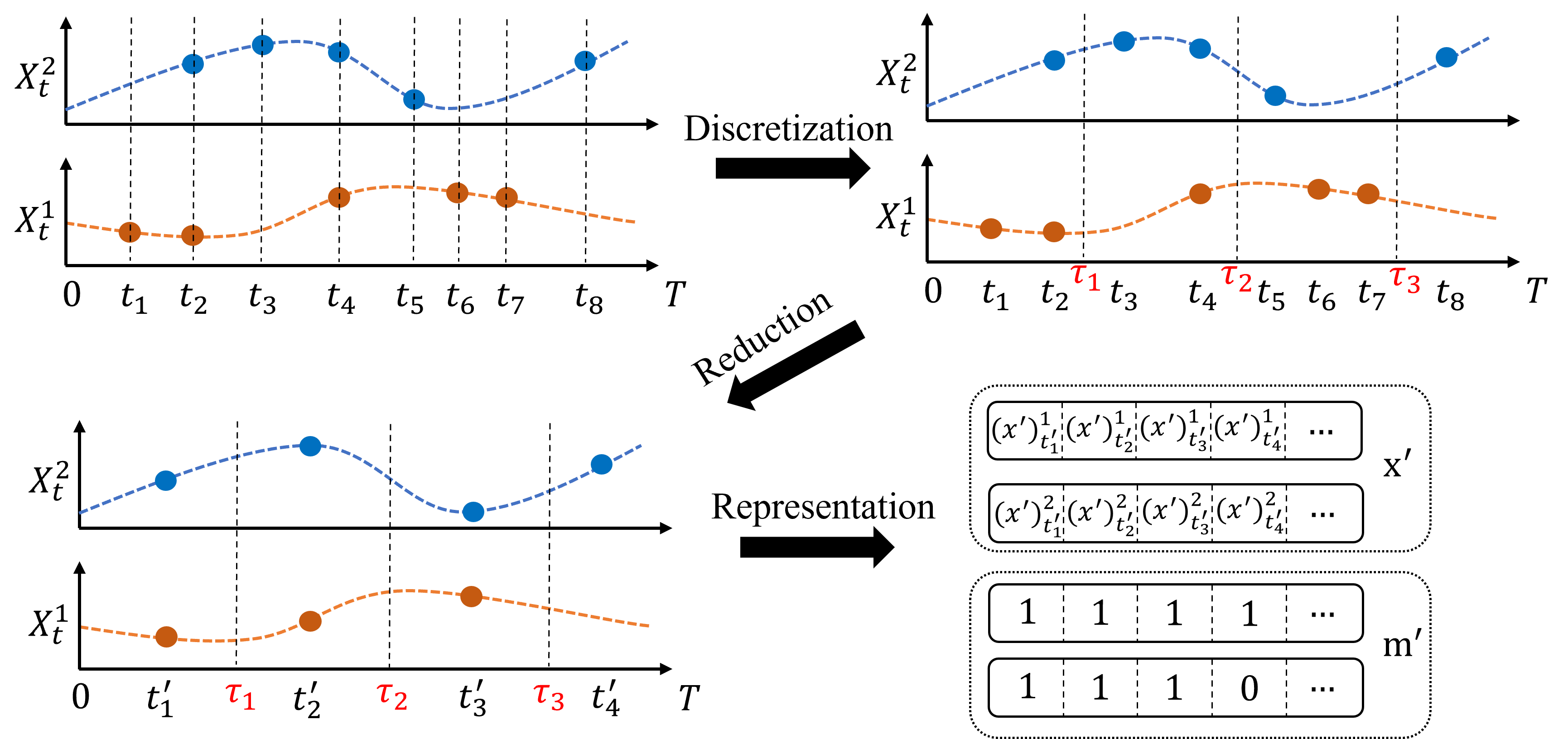}
    \caption{Temporal discretization.}
    \label{fig:discretization}
\end{figure}

\subsection{Imputation}
Imputation is another widely used technique for managing irregular time series. It requires filling in missing values to convert the data into a regularly spaced format. Unlike discretization, which aggregates data and removes local patterns, imputation seeks to retain the original temporal structure. Fig. \ref{fig:imputation} presents some commonly employed imputation techniques. Below are descriptions of these methods:

\begin{itemize}
    \item \textit{Mean Imputation}  \citep{che2018recurrent}: this method imputes the missing values using the means of the observation. It is straightforward and fast, but it sacrifices the variability in the data by filling gaps with constant values. 
    \item \textit{Forward/Backward Imputation} \citep{che2018recurrent}: this method imputes the missing values using the last/next of the observation. It retains local trends by filling gaps with the nearest known value but may fail for long sequences of missing data, exaggerating repeated values.
    \item \textit{Interpolation} \cite{DBLP:conf/iclr/ShuklaM19}: this method imputes the missing values within the range of a discrete set of known data points. Methods include linear interpolation, spline interpolation, and polynomial interpolation, which rely heavily on the assumption of smoothness in the underlying data.
    \item \textit{Gaussian Process Regression (GPR) Imputation} \cite{williams2006gaussian}: this method excels in handling non-linearity by considering uncertainty and correlation within the data. However, it is computationally expensive, especially for large datasets.
    \item \textit{RNN Imputation} \cite{cao2018brits}: this method captures temporal dependencies from the entire sequence by processing the sequence in both forward and backward directions. It jointly learns to impute missing values and perform prediction tasks, optimizing both objectives simultaneously. However, the training can be time-consuming due to the complexity of the bidirectional nature, especially with large datasets.
    \item \textit{GAN Imputation} \cite{luo2018multivariate}: this method generates plausible data to fill gaps by learning the distribution of observed data. It is robust for complex, high-dimensional datasets, though it can be challenging to train.
    \item ...
\end{itemize}
\begin{figure}[ht]
    \centering
    \includegraphics[width=\linewidth]{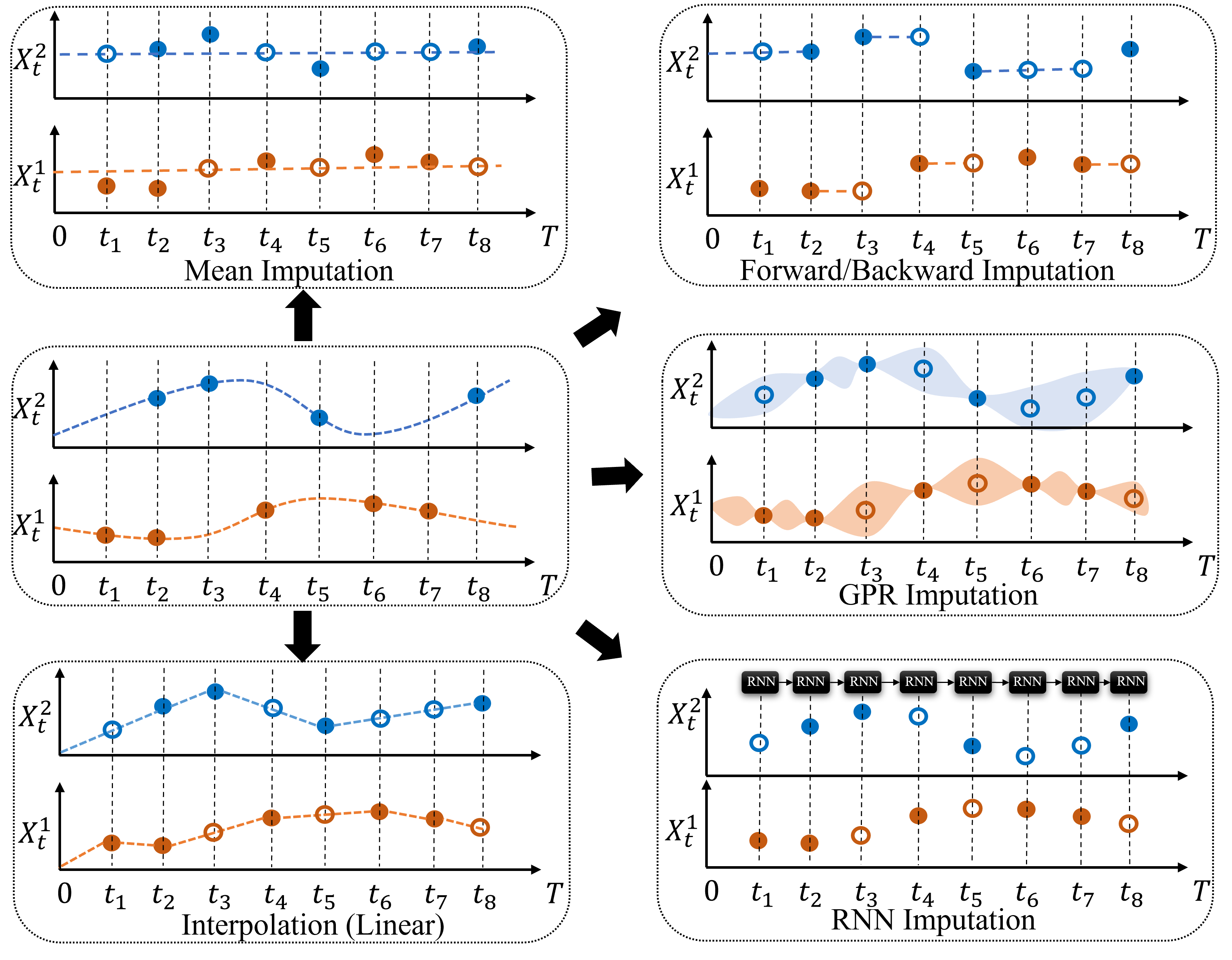}
    \caption{Various temporal imputation approaches.}
    \label{fig:imputation}
\end{figure}

\subsection{End-to-end Frameworks}
       The imputation-then-predict strategy often leads to suboptimal performance because imputation often introduces biases or artificial patterns. To address this, end-to-end frameworks can be used that modify classical RNNs to directly handle irregular time series, eliminating the need for explicit interpolation and preserving the inherent temporal irregularity.

        One approach is to introduce time-aware RNNs that incorporate the irregularity of timestamps into the model's input. These networks include the time gaps between observations as input, allowing the model to adjust its predictions based on how much time has elapsed between data points. This allows the network to make informed predictions that take into account the non-uniform temporal gaps in the data.
        
        A more sophisticated approach involves working within a continuous-time framework. Discrete-time RNNs typically update the hidden state only when observations are available, assuming a constant hidden state between observations. In contrast, continuous-time models allow the hidden state to evolve over time, even when no new data is observed. For example, GRU-D \citep{che2018recurrent}, as illustrated in Fig. \ref{fig:GRU-D}, introduces a decay mechanism where the hidden state diminishes exponentially over time in the absence of observations, better reflecting the natural decay of information. Similarly, models like T-LSTM \citep{baytas2017patient} adapt traditional LSTM architectures by dividing the memory into short-term and long-term components, incorporating decay into the short-term memory to capture diminishing influences from older data. These models learn to ``forget'' information at a rate proportional to the time gap between observations, allowing them to model both irregular time intervals and missingness patterns without requiring explicit interpolation.
        
        While these models are effective, the exponential decay assumption limits their representation capacity. Exponential decay can be seen as a simplified form of an ordinary differential equation (ODE). To address this limitation, models based on Neural ODEs have been proposed, replacing fixed exponential decay with a more flexible differential equation-driven approach that can better capture complex, continuous temporal dynamics. 
        
        For instance, GRUODE (shown in Fig. \ref{fig:GRUODE}) builds on the GRU architecture but replaces the discrete-time updates with a continuous-time ODE solver. GRUODE allows the hidden state to evolve according to learned differential equations, making it more adept at modeling irregular time series. The continuous updating mechanism accounts for time gaps more naturally than traditional RNNs, thus improving predictive performance for complex, irregular time series data.
        
\begin{figure}[ht]
    \centering
    \includegraphics[width=0.7\linewidth]{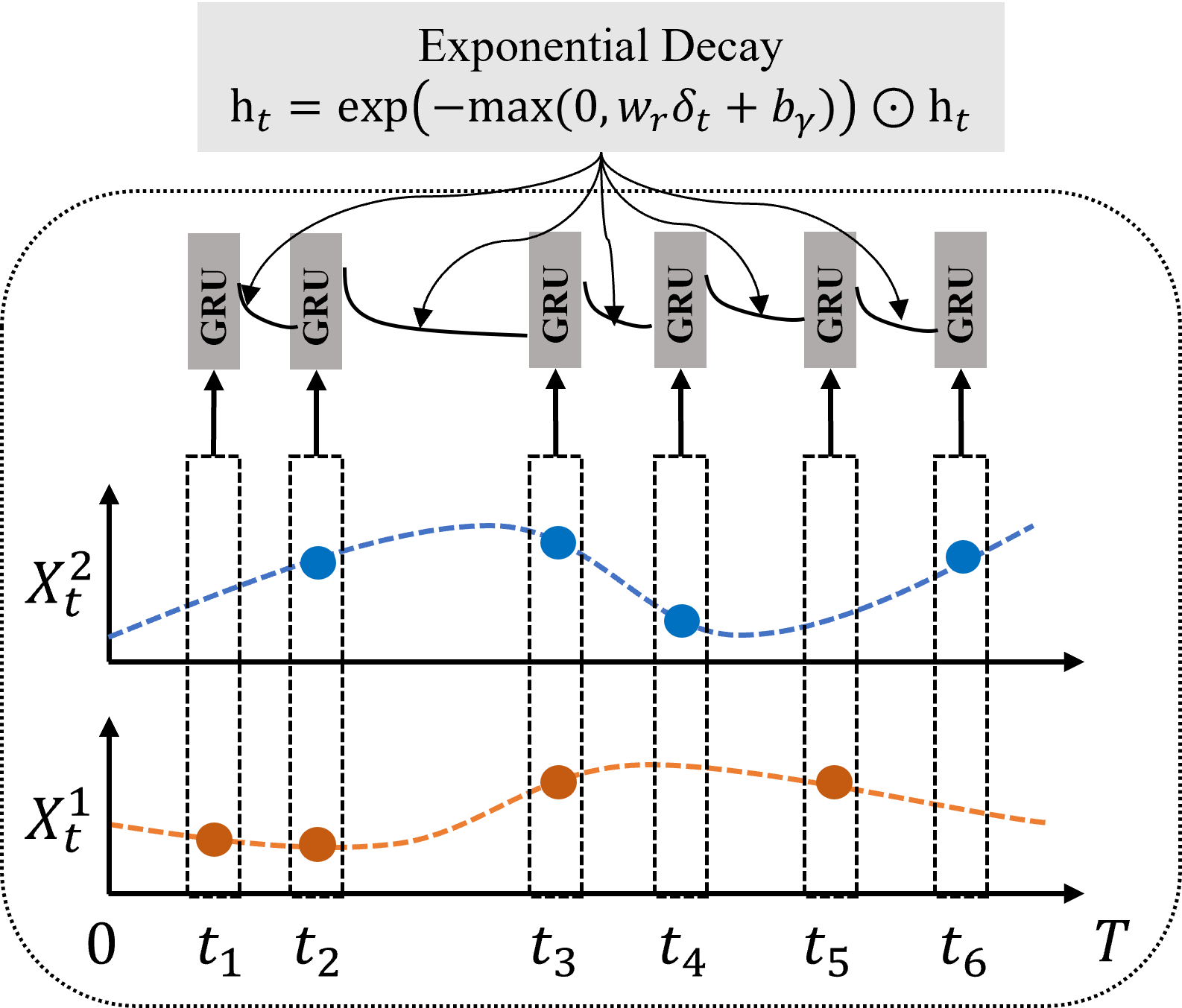}
    \caption{The operation process of GRU-D.\label{fig:GRU-D}}
\end{figure}

\begin{figure}[ht]
    \centering
    \includegraphics[width=\linewidth]{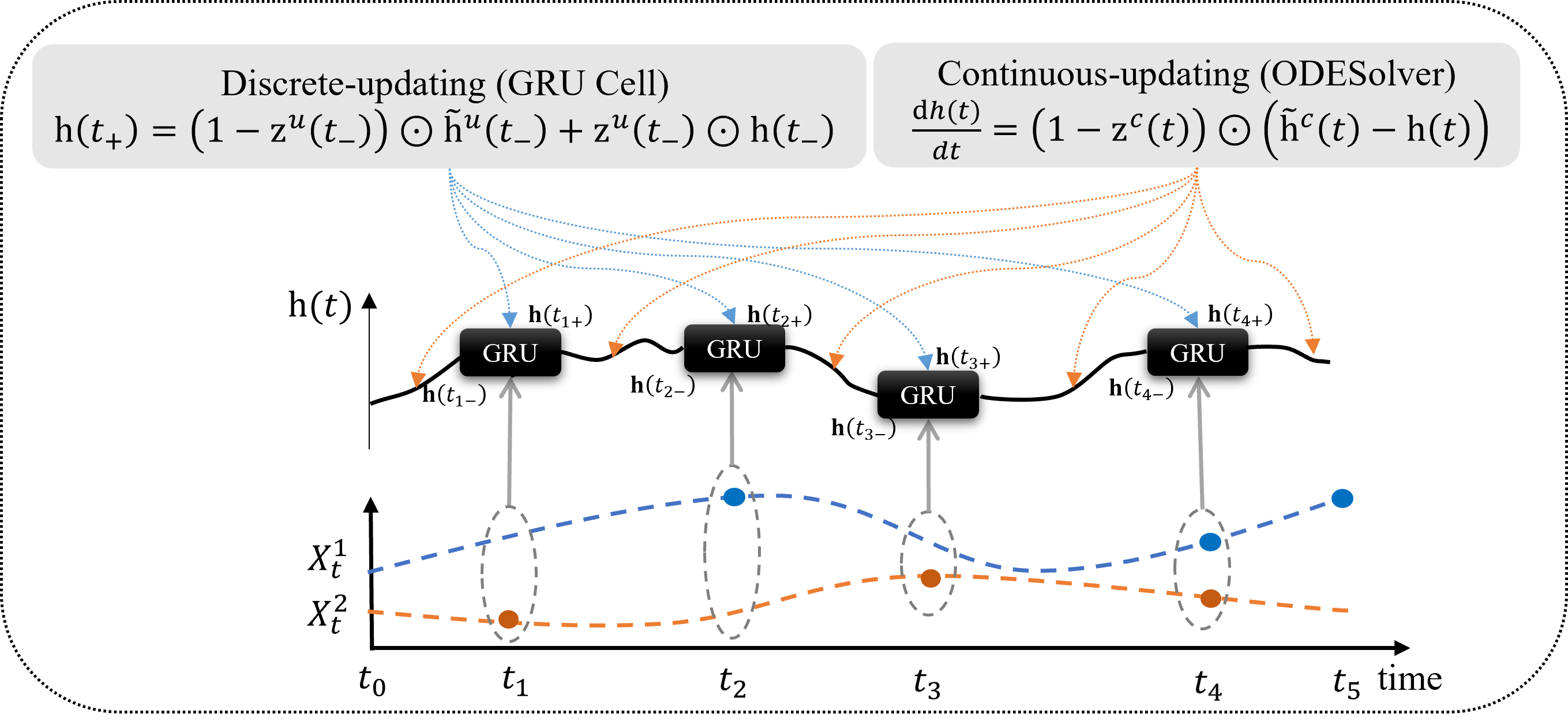}
    \caption{The operation process of GRUODE.\label{fig:GRUODE}}
\end{figure}

\section{Marginal learning layers} \label{Appendix:Marginal learning blocks}
\subsection{GRU-D} \label{subsec:GRU-D}
GRU-D \cite{che2018recurrent} modifies the classical GRU by adding the trainable exponential decays,
\[\gamma_t = \exp\{-\max(\mathrm{0}, \mathbf{w}_{\gamma} \delta_t+\mathrm{b}_{\gamma})\}\]
where $\delta_t$ is the time intervals between two observations, and $\mathbf{w}_{\gamma}$ and $\mathrm{b}_{\gamma}$ are trainable parameters. First, it imputes the missing values with the weighted average between the last observations and the empirical means,
\[\hat{x}_t^d=m_t^d x_t^d+\left(1-m_t^d\right)\left(\gamma_{x_t}^d x_{t^{\prime}}^d+\left(1-\gamma_{x_t}^d\right) \tilde{x}^d\right)\]
where $m_t^d$ is the mask of $d^{\text{th}}$ variable at time $t$, $\gamma_{x_t}^d$ is the trainable decay, $x_{t^{\prime}}^d$ is the previous observation ($t^{\prime}<t$), and $\tilde{x}^d$ is its empirical mean.
Second, it employs hidden state decay to further capture the missing patterns in the hidden state, i.e.,
\[\hat{\mathrm{h}}_{t-1}=\gamma_{h_t} \odot \mathrm{h}_{t-1}.\]
Applying the marginal learning idea to GRU-D, the updating equations are as follows:
\begin{align*}
\mathrm{r}_t&=\sigma\left(\mathbf{W}_{r}\mathrm{\hat{x}}_t+\mathbf{U}_{r}\mathrm{\hat{h}}_{t-1} + \mathbf{V}_{r}\mathrm{m}_t +\mathrm{b}_{r}\right), \nonumber\\
\mathrm{z}_t&=\sigma\left(\mathbf{W}_{z}  \mathrm{\hat{x}}_t+\mathbf{U}_{z}\mathrm{\hat{h}}_{t-1}+\mathbf{V}_{z}\mathrm{m}_t +\mathrm{b}_{z}\right),\nonumber\\
\mathrm{\tilde{h}}_t&=\tanh \left(\mathbf{W}_{h}  \mathrm{\hat{x}}_t+\mathbf{U}_{h}(\mathrm{r}_t \odot \mathrm{\hat{h}}_{t-1})+\mathbf{V}_{h}  \mathrm{m}_t +\mathrm{b}_{h}\right),\nonumber\\
\mathrm{h}_t&=\mathrm{z}_t \odot \mathrm{h}_{t-1} + (1-\mathrm{z}_t) \odot \mathrm{\tilde{h}}_{t},
\end{align*}
where {\small$\mathbf{W}_{\{r,z,h\}}, \mathbf{U}_{\{r,z,h\}}, \mathbf{V}_{\{r,z,h\}}, \mathrm{b}_{\{r,z,h\}}$} are training parameters.

\subsection{ODERNN}\label{subsec:ODERNN}
ODERNN \cite{rubanova2019latent} proposes handling irregularly sampled time series by modeling the hidden state evolution using a Neural ODE in the absence of observations. Compared to GRU-ODE-Bayes \cite{de2019gru}, such a Neural ODE does not admit explicit solutions; instead, it is modeled using a standard MLP network. When there are observations, a classical GRU cell will update the hidden state.

When implementing ODERNN to learn the marginal layer, the evolution formula of ODERNN in the marginal learning layer is
\[\frac{d \mathrm{h}(t)}{d t}=\mathrm{f}_\theta(\mathrm{h}(t), t).\]
Let $\mathrm{h}_{t_-}=\mathrm{h}(t_-)$ and $\mathrm{h}_{t_+}=\mathrm{h}(t_+)$, the updating equations of ODERNN for the marginal learning are 
\begin{align*}
\mathrm{r}_{t_-}&=\sigma\left(\mathbf{W}_{r}  \mathrm{x}_t+\mathbf{U}_{r}  \mathrm{h}_{t_-} +\mathrm{b}_{r}\right), \nonumber\\
\mathrm{z}_{t_-}&=\sigma\left(\mathbf{W}_{z}  \mathrm{x}_t+\mathbf{U}_{z}  \mathrm{h}_{t_-} +\mathrm{b}_{z}\right),\nonumber\\
\mathrm{\tilde{h}}_{t_-}&=\tanh \left(\mathbf{W}_{h}  \mathrm{x}_t+\mathbf{U}_{h}  (\mathrm{r}_t \odot \mathrm{h}_{t_-}) +\mathrm{b}_{h}\right),\nonumber\\
\mathrm{h}_{t_+}&=\mathrm{z}_t \odot \mathrm{h}_{t_-} + (1-\mathrm{z}_t) \odot \mathrm{\tilde{h}}_{t-},
\end{align*}
where  {\small$\mathbf{W}_{\{r,z,h\}}, \mathbf{U}_{\{r,z,h\}}, \mathrm{b}_{\{r,z,h\}}$} are training parameters.

\subsection{ODELSTM}\label{subsec:ODELSTM}
ODELSTM \cite{lechner2020learning} is designed to learn the long-term dependencies in irregularly sampled time series. In the absence of observations, the hidden states are evolved by a Neural ODE, and when there is an observation, the hidden state is updated by the classical LSTM. Unlike the ODERNN and GRU-ODE-Bayes, the cell state of LSTM will contain long-term memory. To cast the ODELSTM in the RFN framework, the evolution equation is
\[\frac{d \mathrm{h}(t)}{d t}=\mathrm{f}_\theta(\mathrm{h}(t), t).\]
Let $\mathrm{h}_{t_-}=\mathrm{h}(t_-)$ and $\mathrm{h}_{t_+}=\mathrm{h}(t_+)$. The updating equations of ODELSTM for marginal learning are 
\begin{align*}
\mathrm{f}_{t_-}&=\sigma\left(\mathbf{W}_{f}  \mathrm{x}_t+\mathbf{U}_{f}  \mathrm{h}_{t_-} +\mathrm{b}_{f}\right), \nonumber\\
\mathrm{i}_{t_-}&=\sigma\left(\mathbf{W}_{i}  \mathrm{x}_t+\mathbf{U}_{i}  \mathrm{h}_{t_-} +\mathrm{b}_{i}\right), \nonumber\\
\mathrm{o}_{t_-}&=\sigma\left(\mathbf{W}_{o}  \mathrm{x}_t+\mathbf{U}_{o}  \mathrm{h}_{t_-} +\mathrm{b}_{o}\right),\nonumber\\
\mathrm{\tilde{c}}_{t_-}&=\tanh \left(\mathbf{W}_{c}  \mathrm{x}_t+\mathbf{U}_{c}  \mathrm{h}_{t_-} +\mathbf{b}_{c}\right),\nonumber\\
\mathrm{c}_{t_+}&=\mathrm{f}_{t_-} \odot \mathrm{c}_{(t-1)_+} + \mathrm{i}_{t_-} \odot \mathrm{\tilde{c}}_{t_-}\\
\mathrm{h}_{t_+}&=\mathrm{o}_{t_-} \odot \tanh (\mathrm{c}_{t_+})
\end{align*}
where {\small$\mathbf{W}_{\{f, i, o, c\}}, \mathbf{U}_{\{f, i, o, c\}}, \mathrm{b}_{\{f, i, o, c\}}$} are training parameters.

\section{Ablation studies on Imputing Unobserved Variables}
When we specify the multivariate asynchronous learning layer in Section \ref{sec:Multivariate Asynchronous Learning}, we posit that there are two ways to handle concurrent ``missing'' values across variables at a given observation time. 

We can either impute the unobserved variables, which completes the asynchronous dataset into a synchronous one and train them using the Syn-MTS model, or utilize our Asyn-MTS model, where we factorize the conditional log-likelihood along the component dimension one more time at any observation time, and only compute the resulting factorized conditional log-likelihoods for variables that have observations. In other words, the objective function \eqref{Eq23} avoids computing any loss terms for unobserved variables. Fig. \ref{fig:ablation_3} indicates clearly that our approach (the left bars in lighter green), without completing the data, is indeed robustly superior to the ``imputing before training'' approach, where the unobserved variables are imputed by their marginal layer forecasts, despite it being more intuitive.

\begin{figure}[tbh]
	\centering
	\subfigure[$\mathrm{CRPS}$]{
		\includegraphics[width=0.44\columnwidth]{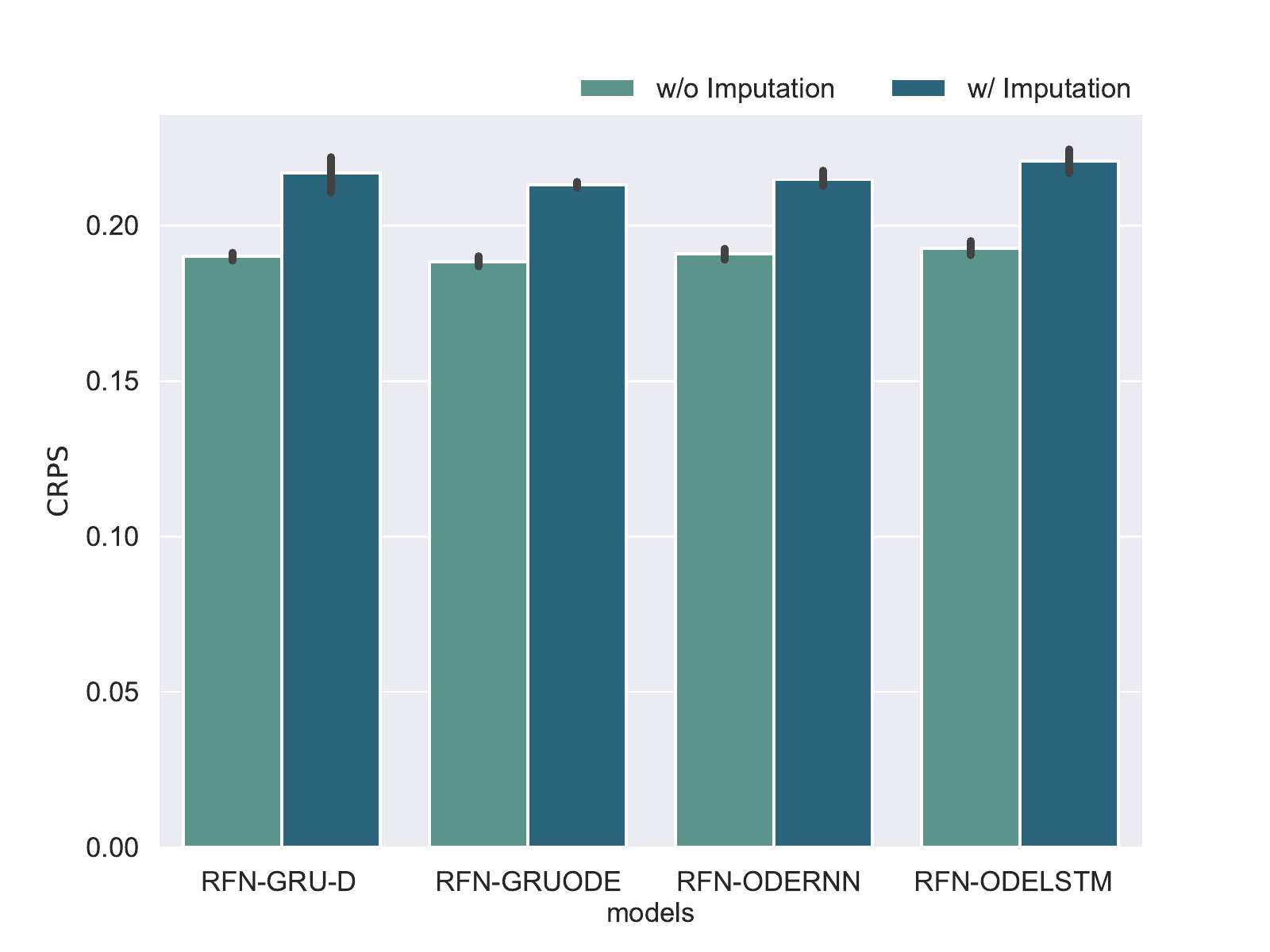}
	}%
	\subfigure[$\mathrm{CRPS}_{\text {sum }}$]{
		\includegraphics[width=0.44\columnwidth]{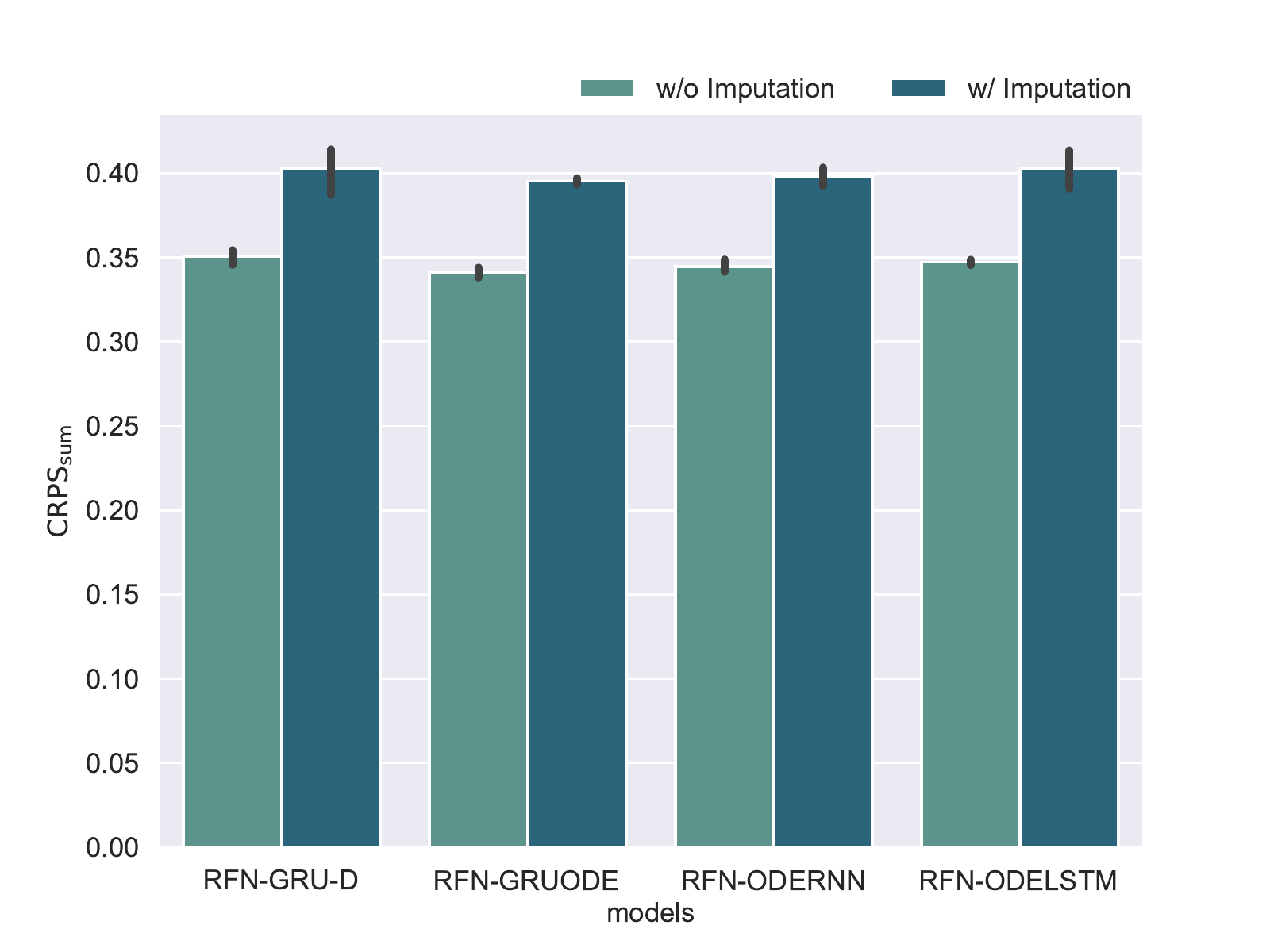}
	}%
	\caption{Asynchronous data: Imputing unobserved variables (right bars in darker green) vs. Not imputing them (left bars in lighter green).\label{fig:ablation_3}}
\end{figure}

\section{Proof of Lemma 5.1} \label{Appendix:Proof of Theorem}
\begin{proof}
Consider the random variables $\mathrm{X}\in \mathbb{R}^D$ and $\mathrm{Y}\in \mathbb{R}^C$. Let $\tilde{\mathrm{Z}}(s)=[\mathrm{Z}(s),\mathrm{Y}(s)]^{\top}$ be a finite continuous random variable. Let the realizations of $\tilde{\mathrm{Z}}(s)$ be $\tilde{\mathrm{z}}(s)=[\mathrm{z}(s),\mathrm{y}(s)]^{\top}$, and the probability density of $\tilde{\mathrm{Z}}(s)$ be $p(\tilde{\mathrm{z}}(s))=p(\mathrm{z}(s),\mathrm{y}(s))$ which depends on the flow time $s$, where $s_{0}\leq s\leq s_{1}$. We assume $\mathrm{z}(s)$ evolves continuously in the real space, starting from a sample point of a pre-defined distribution at $s=s_{0}$ and ending at a sample point $\mathrm{x}$ of $\mathrm{X}$ at $s=s_{1}$. According to \eqref{eqt:ffjord1}, the governing equation of $\tilde{\mathrm{z}}(s)$ can be written as
{\small
\begin{equation*}
\begin{aligned}
\frac{\partial \tilde{\mathrm{z}}(s)}{\partial s}=
\begin{bmatrix}
\frac{\partial \mathrm{z}(s)}{\partial s}  \\
\frac{\partial \mathrm{y}(s)}{\partial s} 
\end{bmatrix}
=\begin{bmatrix}
\mathrm{f}(\mathrm{z}(s), s, \mathrm{y}(s); \theta) \\
0
\end{bmatrix}, \text { where } s_0 \leq s \leq s_1,
\end{aligned}
\end{equation*}
}\noindent
$\mathrm{y}(s)=\mathrm{y}$ and $\mathrm{z}(s_{1})=\mathrm{x}$. From the differential equation $\frac{\partial \mathrm{y}(s)}{\partial s}=0$, we have $\mathrm{y}(s)=\mathrm{y}$. Hence, we are going to find the density $p\left(\tilde{\mathrm{z}}_s\right)=p\left(\mathrm{z}_s, \mathrm{y}\right)$ such that
\begin{equation*}
\begin{aligned}
\frac{\partial \mathrm{z}(s)}{\partial s}=\mathrm{f}(\mathrm{z}(s), s, \mathrm{y} ; \theta).
\end{aligned}
\end{equation*}
%
The proof requires the change of variables in the probability density theorem. We restate the result below.
\begin{result}
Suppose that $G(\cdot)$ is a bijective function and differentiable. Given the variables $(\mathrm{Z},\mathrm{Y})$ and the corresponding density function $p(\mathrm{z},\mathrm{y})$, the density function of $(G(\mathrm{Z}),\mathrm{Y})$ is
\begin{equation}
\begin{aligned}\label{eqt:log density difference}
p(G(\mathrm{z}),\mathrm{y})&=p(\mathrm{z},\mathrm{y})\left|\det\begin{bmatrix}
\partial_\mathrm{z} G(\mathrm{z}) & \mathrm{0}_{D\times C}\\
\mathrm{0}_{C \times D} & \mathbb{I}_{C}
\end{bmatrix}\right|^{-1}\\
\Rightarrow \log \frac{p(G(\mathrm{z}),\mathrm{y})}{p(\mathrm{z},\mathrm{y})}&=-\log \left|\det\begin{bmatrix}
\partial_{\mathrm{z}} G(\mathrm{z})& \mathrm{0}_{D\times C}\\
\mathrm{0}_{C\times D} & \mathbb{I}_{C}
\end{bmatrix}\right|.
\end{aligned}
\end{equation}
\end{result} 
Now, we consider $\frac{\partial \log p(\tilde{\mathrm{z}}(s))}{\partial s}=\frac{\partial \log p(\mathrm{z}(s), \mathrm{y})}{\partial s}$. We write $\tilde{\mathrm{z}}(s+\epsilon)=[\mathrm{z}(s+\epsilon),\mathrm{Y}]^{\top}=[T_{\epsilon}(\mathrm{z}(s)),\mathrm{Y}]^{\top}$. From the first principle of derivatives, we have
\begin{equation*}
		\begin{aligned}
			\partial_s \log p(\tilde{\mathrm{z}}(s))
			&=\underset{\epsilon\rightarrow 0^{+}}{\lim}\frac{\log p(\tilde{\mathrm{z}}(s+\epsilon))-\log p(\tilde{\mathrm{z}}(s))}{\epsilon}\\
			&=\underset{\epsilon\rightarrow 0^{+}}{\lim}\frac{\log p(T_{\epsilon}(\mathrm{z}(s)),\mathrm{y})-\log p(\mathrm{z}(s),\mathrm{y})}{\epsilon}.
		\end{aligned}
\end{equation*}
We simplify the quantity $\underset{\epsilon\rightarrow 0^{+}}{\lim}\frac{\log p(T_{\epsilon}(\mathrm{z}(s)),\mathrm{y})-\log p(\mathrm{z}(s),\mathrm{y})}{\epsilon}$. Applying equation \eqref{eqt:log density difference}, we have
\begin{equation*}
\begin{aligned}
&\partial_s \log p(\tilde{\mathrm{z}}(s))\\
&=\underset{\epsilon\rightarrow 0^{+}}{\lim}\frac{-\log \left|\det\begin{bmatrix}
\partial_{\mathrm{z}(s)} T_{\epsilon}(\mathrm{\mathrm{z}}(s)) & \mathrm{0}_{D\times C}\\
\mathrm{0}_{C\times D} & \mathbb{I}_{C}
\end{bmatrix}\right|}{\epsilon}\\
&\overset{\text{L'H\^{o}pital}}{=}-\underset{\epsilon\rightarrow 0^{+}}{\lim}\partial_\epsilon\log \left|\det\begin{bmatrix}
\partial_{\mathrm{z}(s)} T_{\epsilon}(\mathrm{z}(s)) & \mathrm{0}_{D\times C}\\
\mathrm{0}_{C\times D} & \mathbb{I}_{C}
\end{bmatrix}\right|\\
&=-\underset{\epsilon\rightarrow 0^{+}}{\lim}\frac{\partial_\epsilon\left|\det\begin{bmatrix}
\partial_{\mathrm{z}(s)} T_{\epsilon}(\mathrm{z}(s)) & \mathrm{0}_{D\times C}\\
\mathrm{0}_{C\times D} & \mathbb{I}_{C}
\end{bmatrix}\right|}{\left|\det\begin{bmatrix}
\partial_{\mathrm{z}(s)} T_{\epsilon}(\mathrm{z}(s)) & \mathrm{0}_{D\times C}\\
\mathrm{0}_{C\times D} & \mathbb{I}_{C}
\end{bmatrix}\right|}\\
&=\underbrace{\frac{\underbrace{-\underset{\epsilon\rightarrow 0^{+}}{\lim}\partial_\epsilon\left|\det\begin{bmatrix}
\partial_{\mathrm{z}(s)} T_{\epsilon}(\mathrm{z}(s)) & \mathrm{0}_{D\times C}\\
\mathrm{0}_{C\times D} & \mathbb{I}_{C}
\end{bmatrix}\right|}_{\text{bounded}}}{\underset{\epsilon\rightarrow 0^{+}}{\lim}\left|\det\begin{bmatrix}
\partial_{\mathrm{z}(s)} T_{\epsilon}(\mathrm{z}(s)) & \mathrm{0}_{D\times C}\\
\mathrm{0}_{C\times D} & \mathbb{I}_{C}
\end{bmatrix}\right|}}_{1}\\
&=-\underset{\epsilon\rightarrow 0^{+}}{\lim}\partial_\epsilon\left|\det\begin{bmatrix}
\partial_{\mathrm{z}(s)} T_{\epsilon}(\mathrm{z}(s)) & \mathrm{0}_{D\times C}\\
\mathrm{0}_{C\times D} & \mathbb{I}_{C}
\end{bmatrix}\right|.
\end{aligned}
\end{equation*}
Applying the Jacobi’s formula, we have
\begin{equation*}
\begin{aligned}
&\partial_s \log p(\tilde{\mathrm{z}}(s))\\
&=-\underset{\epsilon\rightarrow 0^{+}}{\lim}\operatorname{Tr}\biggl(\operatorname{adj}\bigg(\begin{bmatrix}
\partial_{\mathrm{z}(s)} T_{\epsilon}(\mathrm{z}(s)) & \mathrm{0}_{D\times C}\\
\mathrm{0}_{C\times D} & \mathbb{I}_{C}
\end{bmatrix}\bigg)\\
& \qquad \times \partial_\epsilon\begin{bmatrix}
&\partial_{\mathrm{z}(s)} T_{\epsilon}(\mathrm{z}(s)) & \mathrm{0}_{D\times C}\\
&\mathrm{0}_{C\times D} & \mathbb{I}_{C}
\end{bmatrix}\biggl)\\
&=-\operatorname{Tr}\biggl(\underbrace{\underset{\epsilon\rightarrow 0^{+}}{\lim}\operatorname{adj}\bigg(\begin{bmatrix}
\partial_{\mathrm{z}(s)} T_{\epsilon}(\mathrm{z}(s)) & \mathrm{0}_{D\times C}\\
\mathrm{0}_{C\times D} & \mathbb{I}_{C}
\end{bmatrix}\bigg)}_{\mathbb{I}_{C}}\\ 
&\qquad \times\underset{\epsilon\rightarrow 0^{+}}{\lim}\frac{\partial}{\partial \epsilon}\begin{bmatrix}
\partial_{\mathrm{z}(s)} T_{\epsilon}(\mathrm{z}(s)) & \mathrm{0}_{D\times C}\\
\mathrm{0}_{C\times D} & \mathbb{I}_{C}
\end{bmatrix}\biggl)\\
&=-\operatorname{Tr}\biggl(\underset{\epsilon\rightarrow 0^{+}}{\lim}\bigg(\partial_\epsilon\partial_{\mathrm{z}(s)} T_{\epsilon}(\mathrm{z}(s))\bigg)\biggl).
\end{aligned}
\end{equation*}
Applying Taylor series expansion on $T_{\epsilon}(\mathrm{z}(s))$ and taking the limit, we have
\begin{equation*}
		\begin{aligned}
			&\partial_s \log p(\tilde{\mathrm{z}}(s))=-\operatorname{Tr}\biggl(\underset{\epsilon\rightarrow 0^{+}}{\lim}\bigg(\partial_\epsilon\partial_{\mathrm{z}(s)} T_{\epsilon}(\mathrm{z}(s))\bigg)\biggl)\\
			&=-\operatorname{Tr}\biggl(\underset{\epsilon\rightarrow 0^{+}}{\lim}\bigg(\partial_\epsilon \partial_{\mathrm{z}(s)}(\mathrm{z}(s)+\partial_s \mathrm{z}(s)\epsilon+O(\epsilon^{2}))\bigg)\biggl)\\
			=&-\operatorname{Tr}\biggl(\underset{\epsilon\rightarrow 0^{+}}{\lim}\biggl(\partial_\epsilon(\mathbb{I}+\partial_{\mathrm{z}(s)} f(\mathrm{z}(s),s,\mathrm{y};\theta)\epsilon+O(\epsilon^{2}))\bigg)\biggl)\\
			&=-\operatorname{Tr}\biggl(\partial_{\mathrm{z}(s)} f(\mathrm{z}(s),s,\mathrm{y};\theta)\biggl).
		\end{aligned}
\end{equation*}
As such, we have
\begin{equation*}
\begin{aligned}
&\int_{s_{0}}^{s_{1}} \partial_s \log p(\tilde{\mathrm{z}}(s))ds = \int_{s_{0}}^{s_{1}} -\operatorname{Tr}\biggl(\partial_{\mathrm{z}(s)} \mathrm{f}(\mathrm{z}(s),s,\mathrm{y};\theta)\biggl) ds\\
&\Rightarrow \log \frac{p(\tilde{\mathrm{z}}(s_{1}))}{p(\tilde{\mathrm{z}}(s_{0}))}=\int_{s_{1}}^{s_{0}} \operatorname{Tr}\biggl(\partial_{\mathrm{z}(s)} \mathrm{f}(\mathrm{z}(s),s,\mathrm{y};\theta)\biggl) ds\\
&\Rightarrow \log \frac{p(\mathrm{z}(s_{1}),\mathrm{y})}{p(\mathrm{z}(s_{0}),\mathrm{y})} = \int_{s_{1}}^{s_{0}} \operatorname{Tr}\biggl(\partial_{\mathrm{z}(s)} \mathrm{f}(\mathrm{z}(s),s,\mathrm{y};\theta)\biggl) ds.
\end{aligned}
\end{equation*}
The proof is completed.
\end{proof}

\section{Performance Measure of Probability Forecasting} \label{Appendix:Performance Measure}
 Probabilistic forecasting provides a full predictive distribution rather than a single point estimate, enabling decision makers to quantify uncertainty and assess risk comprehensively. The performance of such forecasts is typically evaluated using metrics that capture two critical, and often competing, aspects: \textbf{calibration and sharpness} \citep{gneiting2007probabilistic}.

    \textbf{Calibration} reflects the statistical consistency between predicted probabilities and observed outcomes. In a calibrated forecast, the proportion of outcomes falling within a predicted probability interval should match with the nominal quantile level. 
    For instance, given a constructed interval for the forecast, if a 90\% quantile level is considered, approximately 90\% of the observations sampled from the actual distribution should fall within the quantile interval according to the quantile level.  Calibration ensures that the forecasts are reliable and that the probabilities assigned to future events correspond to their long-run frequencies.
    
    \textbf{Sharpness}, on the other hand, is an intrinsic property of the forecast distributions. It refers to the concentration or spread of the predictive distributions, independent of the observed outcomes. A sharper forecast provides a more precise estimate by concentrating probability mass in a narrower region. 
    However, sharpness must be interpreted in conjunction with calibration: forecasts that are extremely sharp but poorly calibrated can be misleading. Thus, the ideal probabilistic forecast is both sharp (i.e., informative) and well-calibrated (i.e., statistically consistent with the observed data).

    The Continuous Ranked Probability Score ($CRPS$) is a proper scoring rule that quantifies the overall discrepancy between the predicted cumulative distribution function (CDF) and the step function corresponding to the observed outcome \citep{gneiting2007probabilistic}. By integrating the squared differences over the entire support of the random variable, the $CRPS$ captures both the calibration and sharpness of the predictive distribution. A lower $CRPS$ indicates that the predicted distribution is closer to the empirical distribution.

    In multivariate settings, direct application of the univariate $CRPS$ is infeasible. Hence, $CRPS_{sum}$ is introduced \citep{salinas2019high}, which evaluates the forecast performance for a multivariate random variable by considering the sum of its components. Essentially, it aggregates the estimated CDFs of individual variables and compares the resulting cumulative probability to the sum of the observed values, providing a proper scoring rule for the overall multivariate forecast and maintaining consistency with the univariate $CRPS$ framework.

    In contrast, the Confidence Score ($CS$) is a metric that focuses solely on calibration, measuring the statistical consistency between the predicted quantiles and the empirical frequencies  \citep{kuleshov2018accurate}. For a well-calibrated forecast, the proportion of observed values falling below the $q$-th quantile of the predicted distribution should be approximately $q$. The $CS$ computes the squared difference between these empirical frequencies and the nominal quantile levels over a range of thresholds, averaged over all variables and quantile levels. A smaller $CS$ indicates that the forecasted quantiles are well aligned with the observed outcomes, thereby implying superior calibration.

\section{More Experiment Results and Discussions} \label{Appendix:More results}
Firstly, in the simulation experiment, the prediction intervals of Fig. \ref{fig:Syn-MTS est interval} and \ref{fig:Asyn-MTS est interval} are restricted to quantiles ranging from 0.2 to 0.8. Here, we present results with a broader quantile range in Fig. \ref{fig:more res}. 

Secondly, we only report the experiment results of the MuJoCo dataset when the missing rate is 50\%, as illustrated in Table \ref{Mujoco}. Also, we give a plot to illustrate the model performance when the missing rate varies in Fig. \ref{fig:missing_rate_sync} and \ref{fig:missing_rate_async}. Here, we provide the detailed experiment results in Table \ref{Sim_25} and \ref{Sim_75}. 

Thirdly,  we report the results of the efficiency experiment in Table \ref{efficiency}, which details the training time and memory usage for all baseline and RFN models. From the results, it's clear that memory usage is similar across all models. However, the training time for ODE-based models is significantly longer than for standard RNN variants like GRU-D. This is reasonable, as ODE-based models must solve differential equations in a continuous-time manner. Additionally, RFNs require longer training times than baselines due to the complexity of their joint learning layer, which involves learning flow mappings after the marginal learning step. This increase in training time is one of the limitations of RFN models, although it enables a richer and more flexible modeling of the data.

 Moreover, as illustrated in Figure \ref{fig:strucure}, in the Syn-MTS setting, missingness occurs uniformly across all variables at the same time steps, meaning that the effective temporal resolution is reduced by half when the missing rate is 50\%. This creates large gaps where no information is available, making it difficult for the model to learn temporal dependencies effectively. For instance, when forecasting at time \(t_7\), all information from \(t_4\) to \(t_6\) is completely lost, which significantly weakens the model's ability to capture serial dependencies.  

    Conversely, in the Asyn-MTS setting, missingness is introduced independently for each variable, meaning that even though each variable is missing 50\% of its values, most of the temporal grid is preserved. As a result, at each time step, even when some variable values are missing, the model can leverage cross-sectional dependencies from other observed variables to recover some of the missing information. This leads to a better-calibrated probabilistic forecasting model, reducing errors in $CS$ evaluation.  
    
    The performance gap between Asyn-MTS and Syn-MTS is particularly pronounced in MUJOCO, as the dataset represents the movement dynamics of a robotic system, where serial dependencies and cross-sectional dependencies are strong. These structured and interdependent time series describe different physical components of the system (e.g., positions, velocities, and joint angles), meaning that observed variables contain valuable information about missing ones. In contrast, the data follows a Geometric Brownian Motion (GBM) in the simulation experiment, and the advantage of Asyn-MTS is less pronounced as GBM exhibits weak serial dependencies due to its Markovian properties. As a result, even though Asyn-MTS retains more time steps with at least some observed variables, the additional information does not significantly improve forecasting performance.  
    
    Despite its advantage in calibration, Asyn-MTS introduces higher predictive uncertainty due to its independent missing structure. Since different variables are missing at different time steps, the model must account for more complex cross-sectional dependencies, increasing variance in the predicted distributions. This leads to broader predictive distributions with reduced sharpness. Because $CRPS$ and $CRPS_{sum}$ assess both calibration and sharpness, their values do not increase significantly as $CS$.

    \begin{figure}[htb]
	\centering
	\subfigure[Syn-MTS data structure]{
		\includegraphics[width=0.9\linewidth]{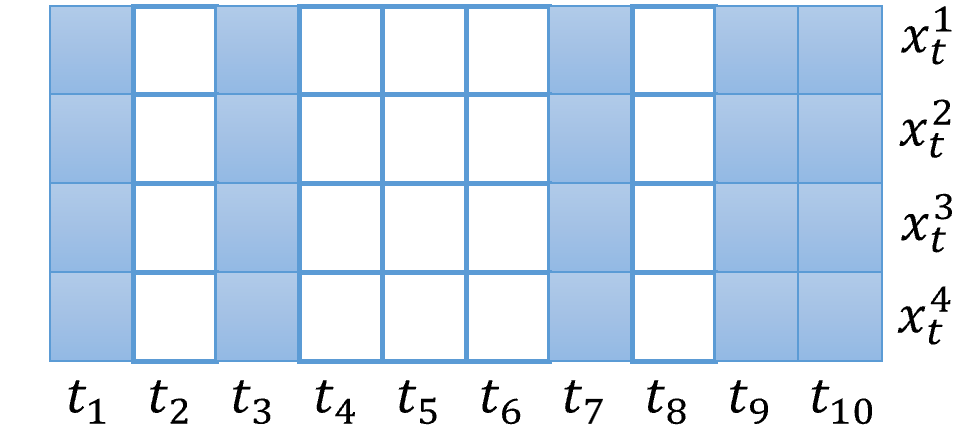}
        \label{fig:sync-app}
	}
	\subfigure[Asyn-MTS data structure]{
		\includegraphics[width=0.9\linewidth]{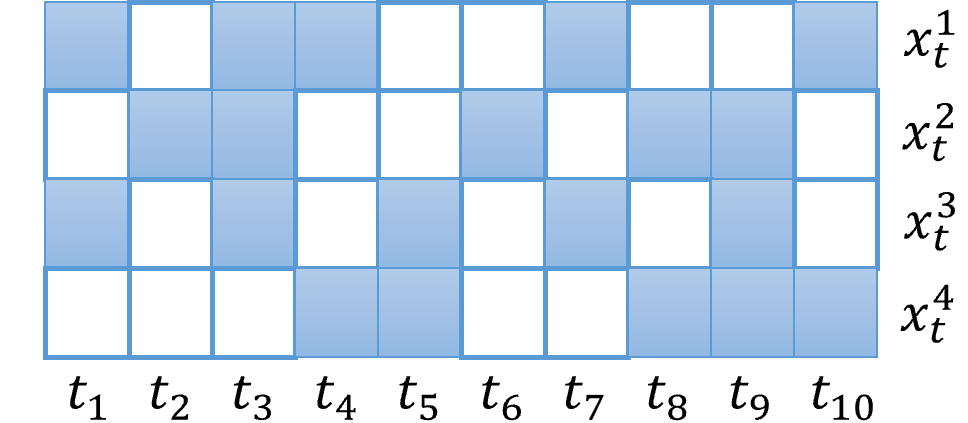}
        \label{fig:async-app}
	}	
	\caption{The data structure of Syn-MTS and Asyn-MTS dataset. \label{fig:strucure}}  
\end{figure}

\begin{table*}[tbh]
	\centering
	\caption{\textbf{Physical Activities of Human Body (MuJoCo): 25\% missingness} \label{Sim_25}}
	\begin{tabular}{@{}cccccccc@{}}
		\toprule
		model   & $\mathrm{CRPS}$  & $\mathrm{CRPS}_{\text {sum }}$  & $\mathrm{CS}$ &  model        & $\mathrm{CRPS}$    & $\mathrm{CRPS}_{\text {sum }}$ & $\mathrm{CS}$\\ \midrule
		\multicolumn{8}{c}{\textit{Syn-MTS}}                                                                                            \\
		\midrule      
		GRUODE  & 0.1380 $\pm$ 0.0213 & 0.7012 $\pm$ 0.0729 & 0.0125 $\pm$ 0.0009 & RFN-GRUODE  & 0.1244 $\pm$ 0.0239 & 0.5288 $\pm$ 0.0313 & 	0.0055 $\pm$ 0.0014   \\
		ODELSTM & 0.1138 $\pm$ 0.0049 & 0.5989 $\pm$ 0.0490 & 0.0085 $\pm$ 0.0013 & RFN-ODELSTM & 0.0952 $\pm$ 0.0010 & 0.4360 $\pm$ 0.0091 & 0.0055 $\pm$ 0.0010   \\
		ODERNN  & 0.1351 $\pm$ 0.0272 & 0.6981 $\pm$ 0.0416 & 0.0120 $\pm$ 0.0015 & RFN-ODERNN  & 0.1049 $\pm$ 0.0052 & 0.4627 $\pm$ 0.0173 & 0.0058 $\pm$ 0.0008   \\
		GRU-D 	& 0.1163 $\pm$ 0.0064 & 0.6078 $\pm$ 0.0509 & 0.0111 $\pm$ 0.0007 & RFN-GRU-D   & 0.1141 $\pm$ 0.0034 & 0.5094 $\pm$ 0.0186 & 0.0055 $\pm$ 0.0012   \\
		\midrule
		\multicolumn{8}{c}{\textit{Asyn-MTS}}                                                                                                  \\
		\midrule 
		GRUODE  & 0.1462 $\pm$ 0.0022 &	0.6845 $\pm$ 0.0200	& 0.0009 $\pm$ 0.0001 & RFN-GRUODE  & 0.1391 $\pm$ 0.0033 & 0.6337 $\pm$ 0.0173 & 0.0008 $\pm$ 0.0001    \\
		ODELSTM & 0.1335 $\pm$ 0.0021 &	0.6131 $\pm$ 0.0168	& 0.0010 $\pm$ 0.0003 & RFN-ODELSTM & 0.1175 $\pm$ 0.0032 & 0.5230 $\pm$ 0.0121 & 0.0007 $\pm$ 0.0001   \\
		ODERNN  & 0.1368 $\pm$ 0.0029 & 0.6404 $\pm$ 0.0249	& 0.0012 $\pm$ 0.0004 & RFN-ODERNN  & 0.1170 $\pm$ 0.0014 & 0.5283 $\pm$ 0.0068 & 0.0007 $\pm$ 0.0001     \\
		GRU-D   & 0.1332 $\pm$ 0.0122 & 0.6232 $\pm$ 0.0116	& 0.0009 $\pm$ 0.0001 & RFN-GRU-D   & 0.1152 $\pm$ 0.0050 & 0.5175 $\pm$ 0.0295 & 0.0007 $\pm$ 0.0001    \\ \bottomrule
	\end{tabular}
\end{table*}

\begin{table*}[tbh]
	\centering
	\caption{\textbf{Physical Activities of Human Body (MuJoCo). 50\% missingness} 	\label{Sim_50}}
	\begin{tabular}{@{}cccccccc@{}}
		\toprule
		model   & $\mathrm{CRPS}$  & $\mathrm{CRPS}_{\text {sum }}$  & $\mathrm{CS}$ & model        & $\mathrm{CRPS}$    & $\mathrm{CRPS}_{\text {sum }}$ & $\mathrm{CS}$ \\ \midrule
		\multicolumn{8}{c}{\textit{Syn-MTS}}                                                                                                 \\ \midrule
		GRUODE  & 0.2198 $\pm$ 0.0035 & 1.0763 $\pm$ 0.0352 & 0.0117 $\pm$ 0.0022 & RFN-GRUODE  & 0.1858 $\pm$ 0.0037 & 0.6790 $\pm$ 0.0300 & 0.0097 $\pm$ 0.0033    \\
		ODELSTM & 0.2256 $\pm$ 0.0015 & 1.1402 $\pm$ 0.0298 & 0.0117 $\pm$ 0.0026 & RFN-ODELSTM & 0.1736 $\pm$ 0.0017 & 0.7103 $\pm$ 0.0504 & 0.0060 $\pm$ 0.0010    \\
		ODERNN  & 0.2156 $\pm$ 0.0042 & 1.0450 $\pm$ 0.0393 & 0.0143 $\pm$ 0.0016 & RFN-ODERNN  & 0.1747 $\pm$ 0.0019 & 0.6467 $\pm$ 0.0340 & 0.0087 $\pm$ 0.0025  \\
		GRU-D   & 0.2224 $\pm$ 0.0032 & 1.1118 $\pm$ 0.0177 & 0.0119 $\pm$ 0.0014 & RFN-GRU-D   & 0.1905 $\pm$ 0.0010 & 0.6748 $\pm$ 0.0093 & 0.0084 $\pm$ 0.0007   \\ \midrule
		\multicolumn{8}{c}{\textit{Asyn-MTS}}                                                                                                  \\ \midrule
		GRUODE  & 0.2695 $\pm$ 0.0054 & 1.0661 $\pm$ 0.0249 & 0.0024 $\pm$ 0.0001 & RFN-GRUODE  & 0.1970 $\pm$ 0.0052 & 0.7306 $\pm$ 0.0185 & 0.0007 $\pm$ 0.0002  \\
		ODELSTM & 0.2728 $\pm$ 0.0057 & 1.0834 $\pm$ 0.0292 & 0.0013 $\pm$ 0.0002 & RFN-ODELSTM & 0.1733 $\pm$ 0.0048 & 0.6318 $\pm$ 0.0155 & 0.0007 $\pm$ 0.0002  \\
		ODERNN  & 0.2696 $\pm$ 0.0036 & 1.0748 $\pm$ 0.0194 & 0.0020 $\pm$ 0.0005 & RFN-ODERNN  & 0.1667 $\pm$ 0.0020 & 0.6277 $\pm$ 0.0086 & 0.0012 $\pm$ 0.0004   \\
		GRU-D   & 0.2261 $\pm$ 0.0035 & 0.9078 $\pm$ 0.0255 & 0.0116 $\pm$ 0.0008 & RFN-GRU-D   & 0.1709 $\pm$ 0.0035 & 0.6409 $\pm$ 0.0151 & 0.0058 $\pm$ 0.0009   \\ \bottomrule
	\end{tabular}
\end{table*}

\begin{table*}[tbh]
	\centering
	\caption{\textbf{Physical Activities of Human Body (MuJoCo): 75\% missingness} \label{Sim_75}}
	\begin{tabular}{@{}cccccccc@{}}
		\toprule
		model   & $\mathrm{CRPS}$  & $\mathrm{CRPS}_{\text {sum }}$  & $\mathrm{CS}$ &  model        & $\mathrm{CRPS}$    & $\mathrm{CRPS}_{\text {sum }}$ & $\mathrm{CS}$\\ \midrule
		\multicolumn{8}{c}{\textit{Syn-MTS}}                                                                                            \\
		\midrule      
		GRUODE  & 0.3342 $\pm$ 0.0160 & 1.7209 $\pm$ 0.0884 & 0.0029 $\pm$ 0.0005 & RFN-GRUODE  & 0.2802 $\pm$ 0.0081 & 1.2127 $\pm$ 0.0388 & 0.0110 $\pm$ 0.0002   \\
		ODELSTM & 0.2937 $\pm$ 0.0038 & 1.5077 $\pm$ 0.0489 & 0.0025 $\pm$ 0.0004 & RFN-ODELSTM & 0.2376 $\pm$ 0.0024 & 0.9905 $\pm$ 0.0783 & 0.0100 $\pm$ 0.0004   \\
		ODERNN  & 0.3067 $\pm$ 0.0133 & 1.6496 $\pm$ 0.0233 & 0.0030 $\pm$ 0.0002 & RFN-ODERNN  & 0.2465 $\pm$ 0.0046 & 1.0736 $\pm$ 0.0240 & 0.0140 $\pm$ 0.0005   \\
		GRU-D 	& 0.3341 $\pm$ 0.0024 & 1.5466 $\pm$ 0.0153 & 0.0038 $\pm$ 0.0002 & RFN-GRU-D   & 0.3268 $\pm$ 0.0071 & 1.3790 $\pm$ 0.0183 & 0.0150 $\pm$ 0.0020   \\
		\midrule
		\multicolumn{8}{c}{\textit{Asyn-MTS}}                                                                                                  \\
		\midrule 
		GRUODE  & 0.3308 $\pm$ 0.0045 & 0.8151 $\pm$ 0.0138 & 0.0038 $\pm$ 0.0005 & RFN-GRUODE  & 0.3083 $\pm$ 0.0094 & 0.7629 $\pm$ 0.0230 & 0.0024 $\pm$ 0.0004    \\
		ODELSTM & 0.3066 $\pm$ 0.0053 & 0.7556 $\pm$ 0.0137 & 0.0014 $\pm$ 0.0006 & RFN-ODELSTM & 0.2859 $\pm$ 0.0109 & 0.6965 $\pm$ 0.0227 & 0.0009 $\pm$ 0.0001   \\
		ODERNN  & 0.3071 $\pm$ 0.0042 & 0.7617 $\pm$ 0.0099 & 0.0020 $\pm$ 0.0003 & RFN-ODERNN  & 0.2581 $\pm$ 0.0072 & 0.6401 $\pm$ 0.0184 & 0.0008 $\pm$ 0.0003     \\
		GRU-D   & 0.3082 $\pm$ 0.0054 & 0.8184 $\pm$ 0.0150 & 0.0040 $\pm$ 0.0005 & RFN-GRU-D   & 0.2833 $\pm$ 0.0237 & 0.7086 $\pm$ 0.0607 & 0.0028 $\pm$ 0.0006    \\ \bottomrule
	\end{tabular}
\end{table*}

\begin{table*}[tbh]
\centering
\caption{The efficiency experiment results of baselines and RFN across all datasets.}
\label{efficiency}
\resizebox{\textwidth}{!}{
\begin{tabular}{@{}cccccccccc@{}}
\toprule
\textbf{Model} & \textbf{\begin{tabular}[c]{@{}c@{}}Training Time \\      (sec/epoch)\end{tabular}} & \textbf{\begin{tabular}[c]{@{}c@{}}Memory \\      (MB)\end{tabular}} & \textbf{\begin{tabular}[c]{@{}c@{}}Training Time \\      (sec/epoch)\end{tabular}} & \textbf{\begin{tabular}[c]{@{}c@{}}Memory \\      (MB)\end{tabular}} & \textbf{Model} & \textbf{\begin{tabular}[c]{@{}c@{}}Training Time \\      (sec/epoch)\end{tabular}} & \textbf{\begin{tabular}[c]{@{}c@{}}Memory\\      (MB)\end{tabular}} & \textbf{\begin{tabular}[c]{@{}c@{}}Training Time\\       (sec/epoch)\end{tabular}} & \textbf{\begin{tabular}[c]{@{}c@{}}Memory\\      (MB)\end{tabular}} \\ \midrule
 & \multicolumn{9}{c}{\textbf{Simulation Dataset   (Geometric Brownian Motions)}} \\ \midrule 
 & \multicolumn{2}{c|}{Syn-MTS} & \multicolumn{2}{c}{Asyn-MTS} &  & \multicolumn{2}{c|}{Syn-MTS} & \multicolumn{2}{c}{Asyn-MTS} \\ \cmidrule(lr){1-5} \cmidrule(l){6-10} 
GRUODE & 20.14 & 1972.82 & 20.87 & 1981.69 & RFN-GRUODE & 29.22 & 1916.84 & 28.72 & 1922.43 \\
ODELSTM & 18.30 & 1956.00 & 17.82 & 1956.69 & RFN-ODELSTM & 26.50 & 1918.72 & 25.95 & 1922.25 \\
ODERNN & 17.86 & 1954.82 & 18.27 & 1959.47 & RFN-ODERNN & 26.39 & 1917.24 & 26.14 & 1921.19 \\
GRU-D & 11.19 & 1933.66 & 11.66 & 1938.66 & RFN-GRU-D & 12.86 & 1933.66 & 12.66 & 1940.12 \\ \midrule 
 & \multicolumn{9}{c}{\textbf{Physical Activities   Dataset (MuJoCo)}} \\ \midrule 
 & \multicolumn{2}{c|}{Syn-MTS} & \multicolumn{2}{c}{Asyn-MTS} &  & \multicolumn{2}{c|}{Syn-MTS} & \multicolumn{2}{c}{Asyn-MTS} \\
 \cmidrule(lr){1-5} \cmidrule(l){6-10} 
GRUODE & 67.54 & 2453.46 & 69.65 & 2546.10 & RFN-GRUODE & 101.26 & 2357.99 & 109.00 & 2458.82 \\
ODELSTM & 55.09 & 2418.16 & 58.55 & 2517.65 & RFN-ODELSTM & 89.55 & 2358.58 & 101.98 & 2461.67 \\
ODERNN & 54.39 & 2421.75 & 54.08 & 2511.42 & RFN-ODERNN & 87.35 & 2355.94 & 91.86 & 2459.33 \\
GRU-D & 24.77 & 2369.32 & 24.60 & 2462.26 & RFN-GRU-D & 27.62 & 2374.85 & 33.54 & 2473.43 \\ \midrule 
 & \multicolumn{9}{c}{\textbf{Climate Records Dataset (USHCN)}} \\ \midrule 
 & \multicolumn{2}{c|}{Syn-MTS} & \multicolumn{2}{c}{Asyn-MTS} &  & \multicolumn{2}{c|}{Syn-MTS} & \multicolumn{2}{c}{Asyn-MTS} \\
 \cmidrule(lr){1-5} \cmidrule(l){6-10} 
GRUODE & / & / & 39.84 & 2051.58 & RFN-GRUODE & / & / & 62.67 & 1981.54 \\
ODELSTM & / & / & 35.24 & 2032.52 & RFN-ODELSTM & / & / & 58.39 & 1982.79 \\
ODERNN & / & / & 34.87 & 2029.41 & RFN-ODERNN & / & / & 50.80 & 1979.88 \\
GRU-D & / & / & 17.66 & 2000.36 & RFN-GRU-D & / & / & 20.39 & 1999.87 \\ \midrule 
 & \multicolumn{9}{c}{\textbf{Stock Transactions   Dataset (NASDAQ)}} \\ \midrule 
 & \multicolumn{2}{c|}{Syn-MTS} & \multicolumn{2}{c}{Asyn-MTS} &  & \multicolumn{2}{c|}{Syn-MTS} & \multicolumn{2}{c}{Asyn-MTS} \\
 \cmidrule(lr){1-5} \cmidrule(l){6-10} 
GRUODE & / & / & 25.23 & 1960.33 & RFN-GRUODE & / & / & 39.15 & 1918.46 \\
ODELSTM & / & / & 23.16 & 1946.80 & RFN-ODELSTM & / & / & 34.60 & 1917.58 \\
ODERNN & / & / & 22.69 & 1945.63 & RFN-ODERNN & / & / & 32.45 & 1918.58 \\
GRU-D & / & / & 12.84 & 1933.28 & RFN-GRU-D & / & / & 15.20 & 1936.01 \\ \bottomrule
\end{tabular}
}
\end{table*}

\section{Algorithms of Training and Sampling of Syn-MTS and Asyn-MTS} \label{Appendix:Algorithms}

\section{Computation Infrastructure}
All studies and experiments are run on Dell Precision 7920 Workstations with Intel(R) Xeon(R) Gold 6256 CPU at 3.60GHz, and three sets of NVIDIA Quadro GV100 GPUs. All models are implemented in Python 3.8. The versions of the main packages of our code are: Pytorch 1.8.1+cu102, torchdiffeq: 0.2.2, Sklearn: 0.23.2, Numpy: 1.19.2, Pandas: 1.1.3, Matplotlib: 3.3.2. 

\onecolumn

\begin{figure*}[t]
	\centering
	\subfigure[Syn-MTS with $0.1 \le q \le 0.9$]{
		\includegraphics[width=0.23\linewidth]{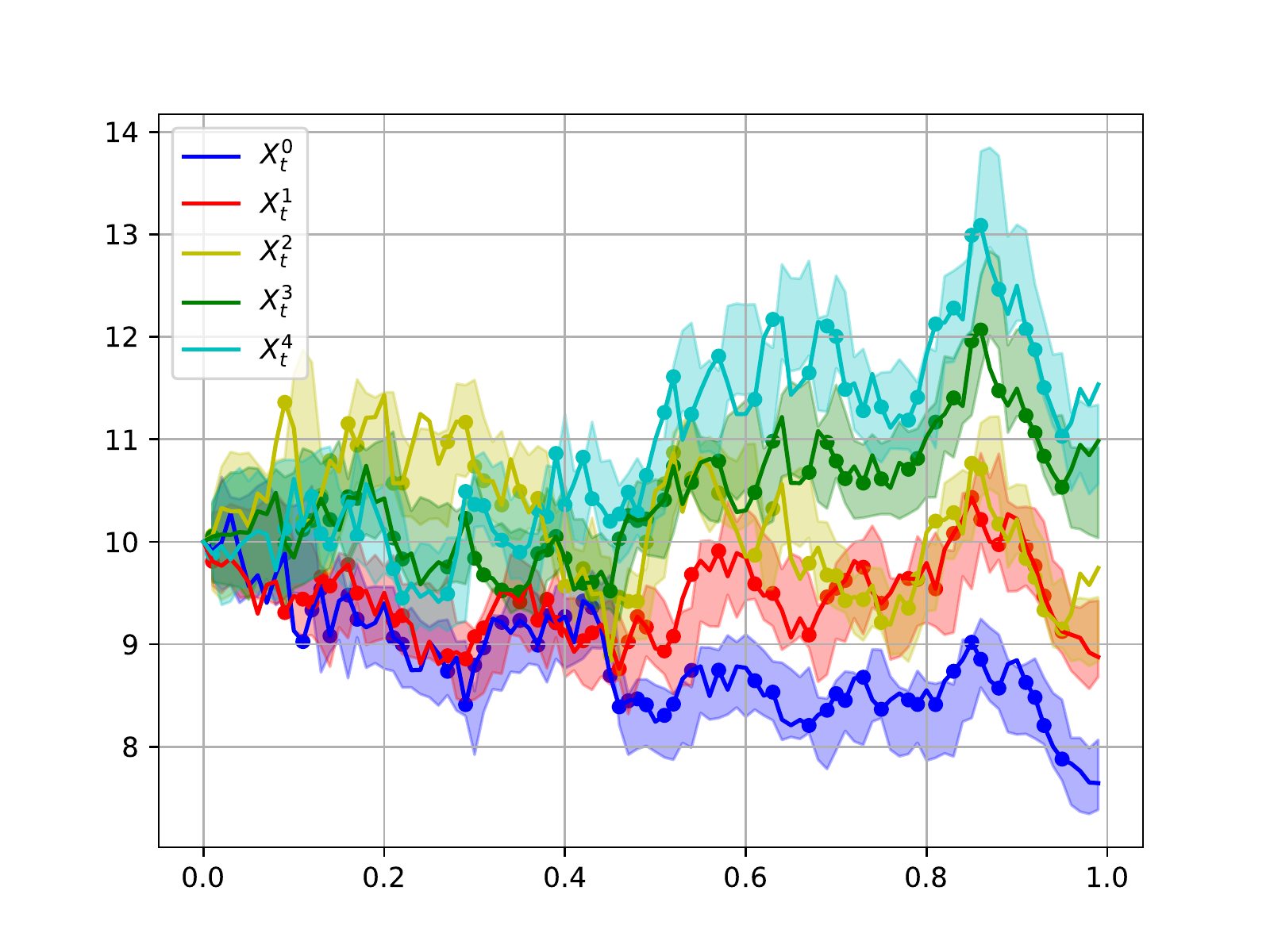}
	}
	\subfigure[Asyn-MTS with $0.1 \le q \le 0.9$]{
		\includegraphics[width=0.23\linewidth]{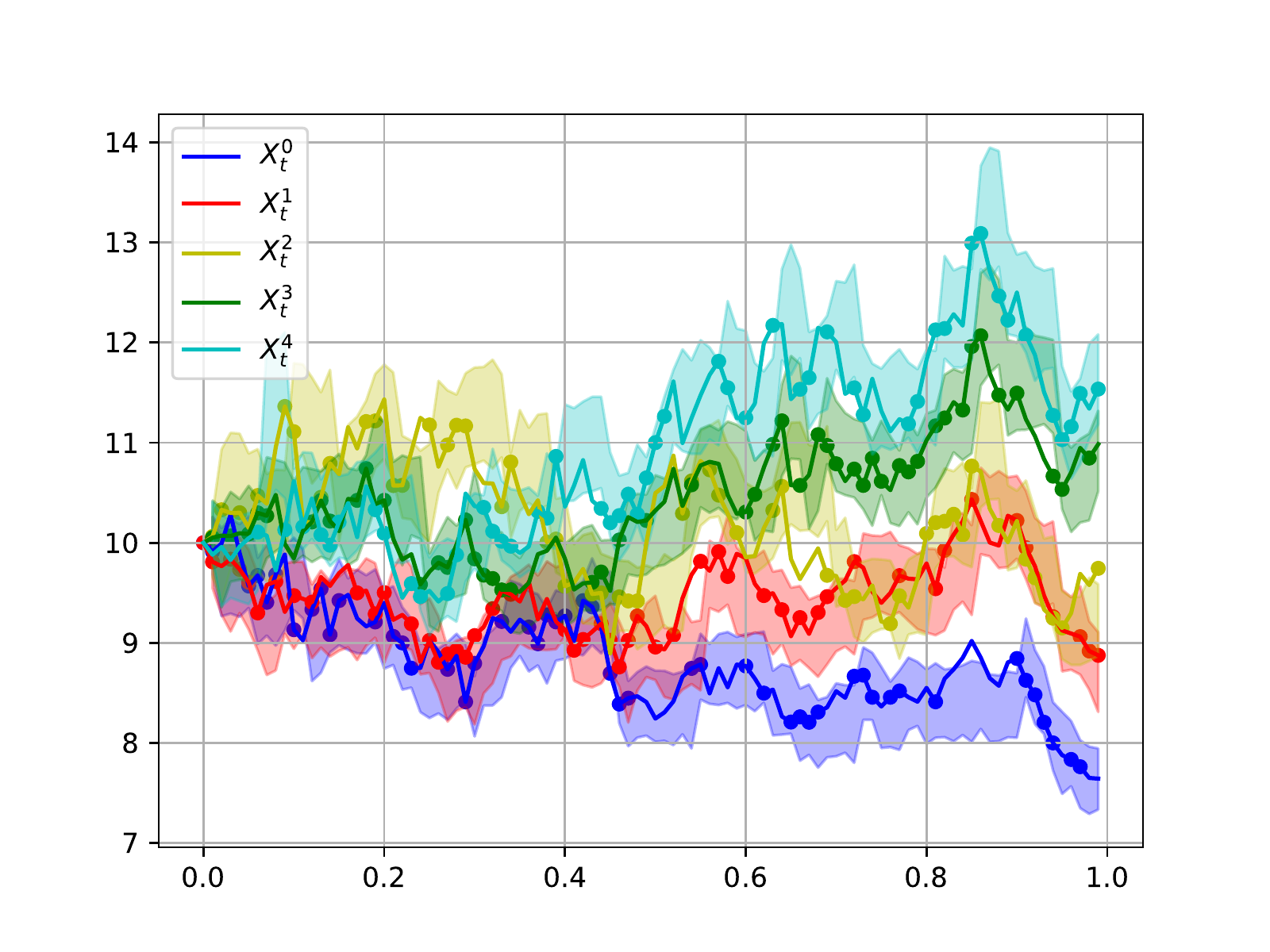}
	}
	\subfigure[Syn-MTS with $0.2 \le q \le 0.8$]{
		\includegraphics[width=0.23\linewidth]{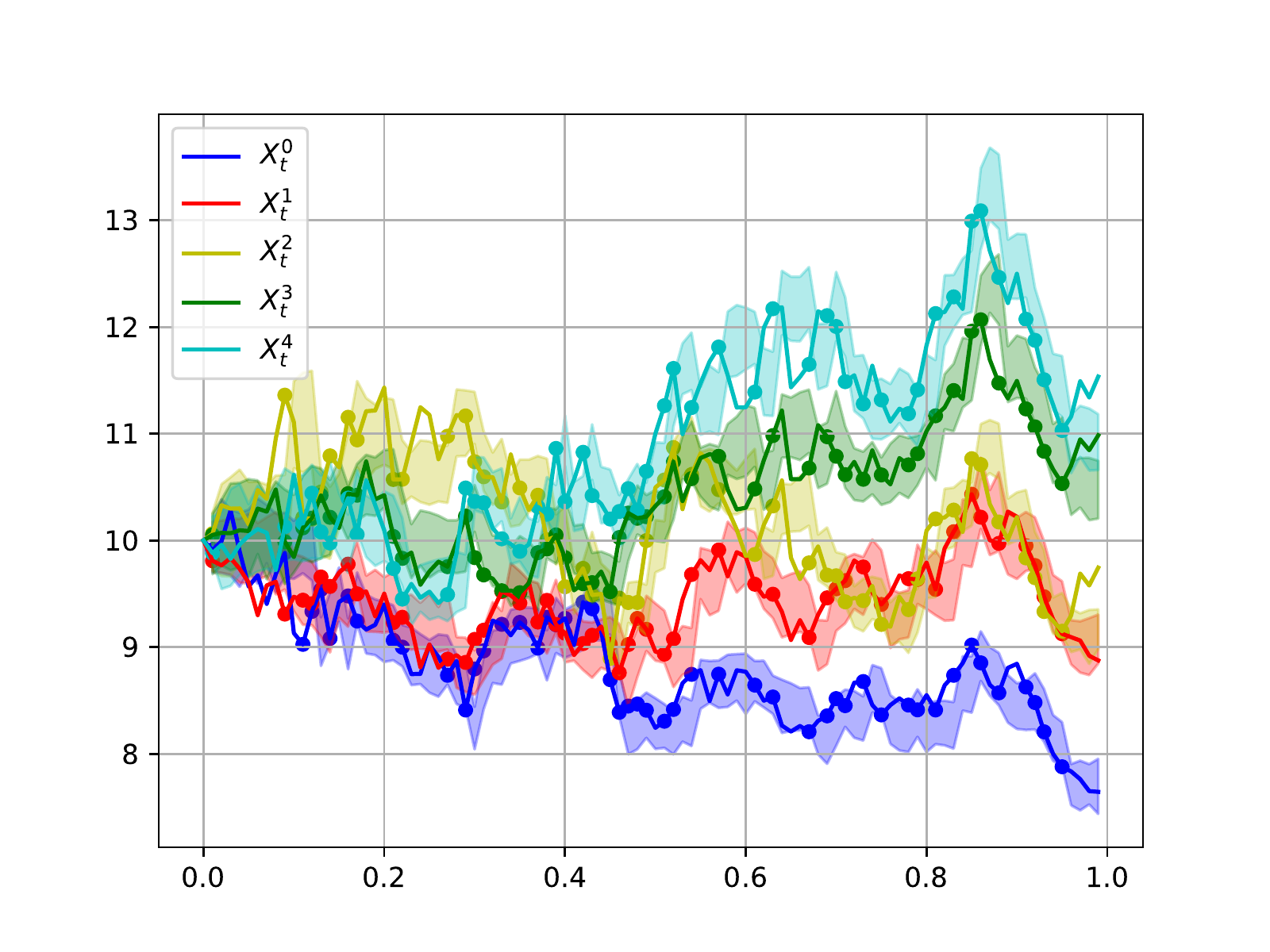}
	}
	\subfigure[Asyn-MTS with $0.2 \le q \le 0.8$]{
		\includegraphics[width=0.23\linewidth]{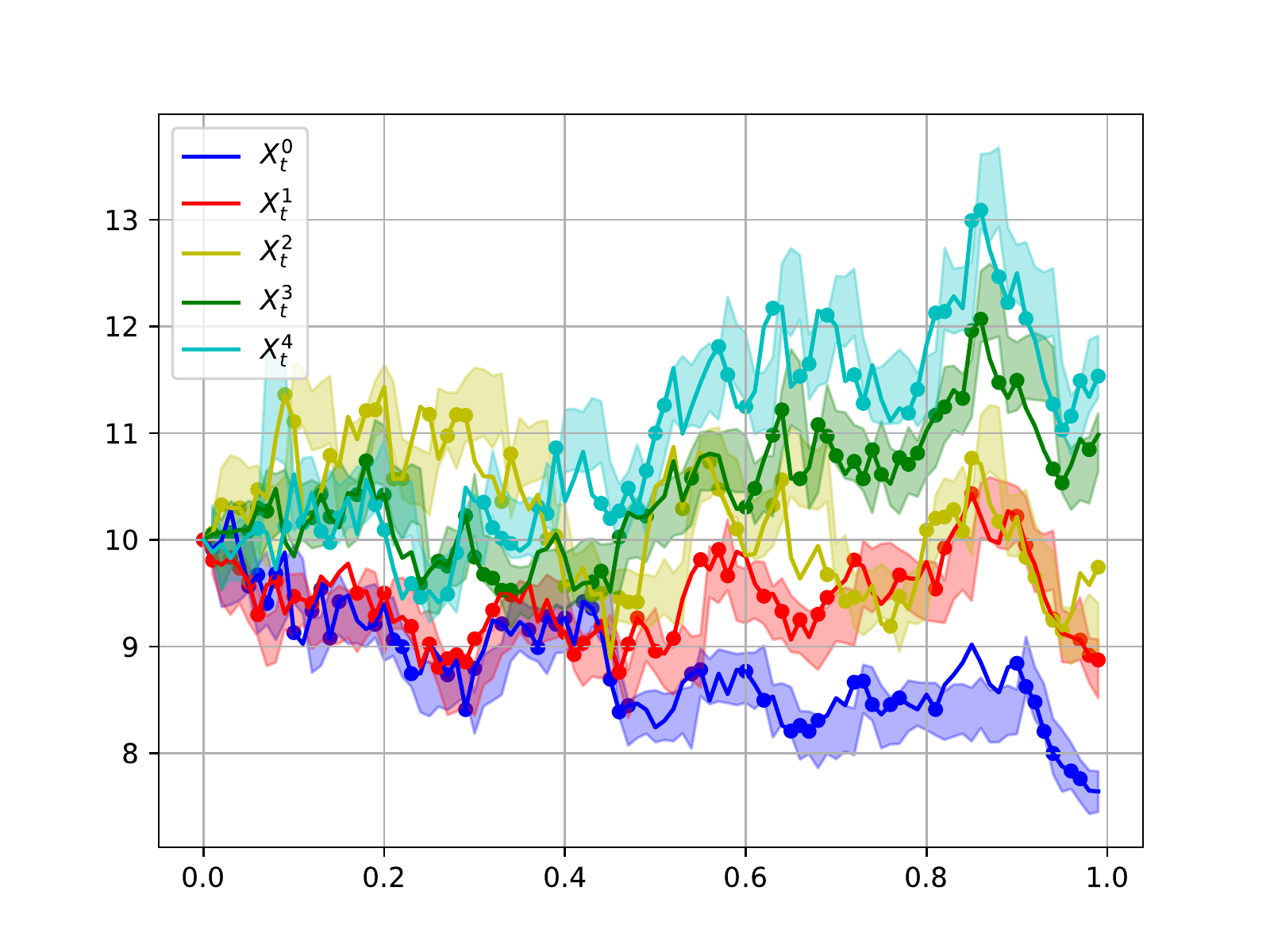}
	}
	\subfigure[Syn-MTS with $0.3 \le q \le 0.7$]{
		\includegraphics[width=0.23\linewidth]{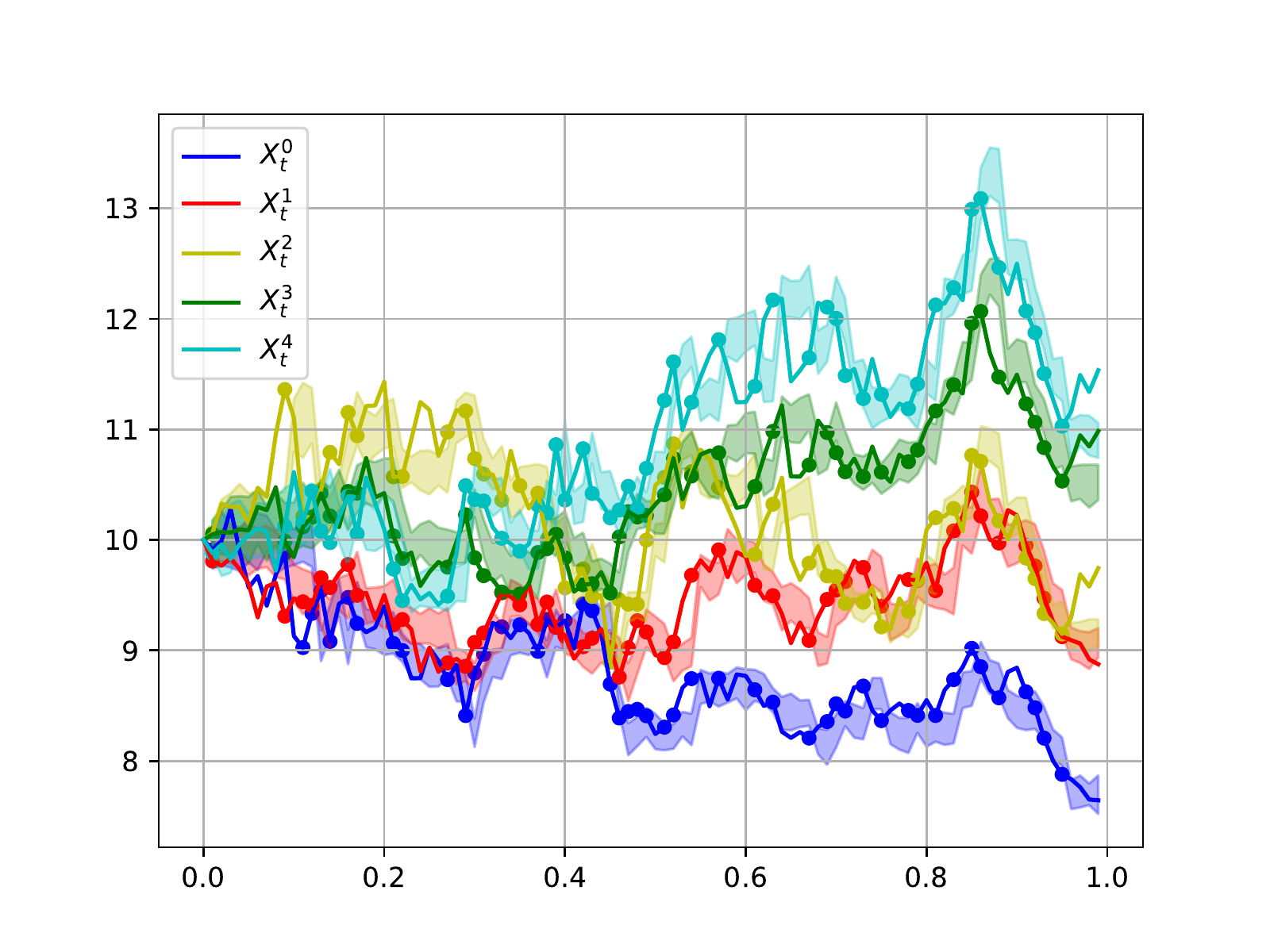}
	}
	\subfigure[Asyn-MTS with $0.3 \le q \le 0.7$]{
		\includegraphics[width=0.23\linewidth]{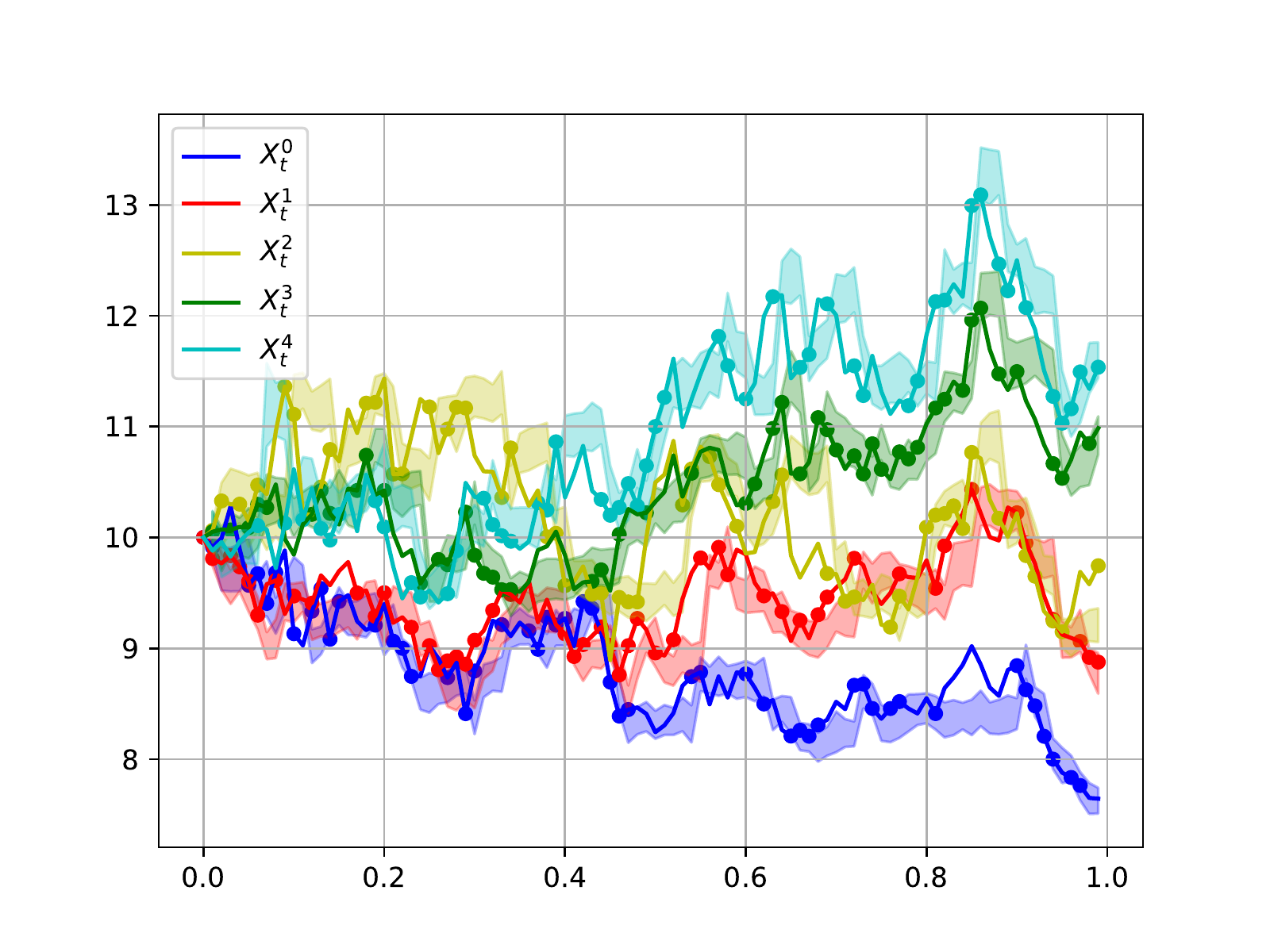}
	}
	\subfigure[Syn-MTS with $0.4 \le q \le 0.6$]{
		\includegraphics[width=0.23\linewidth]{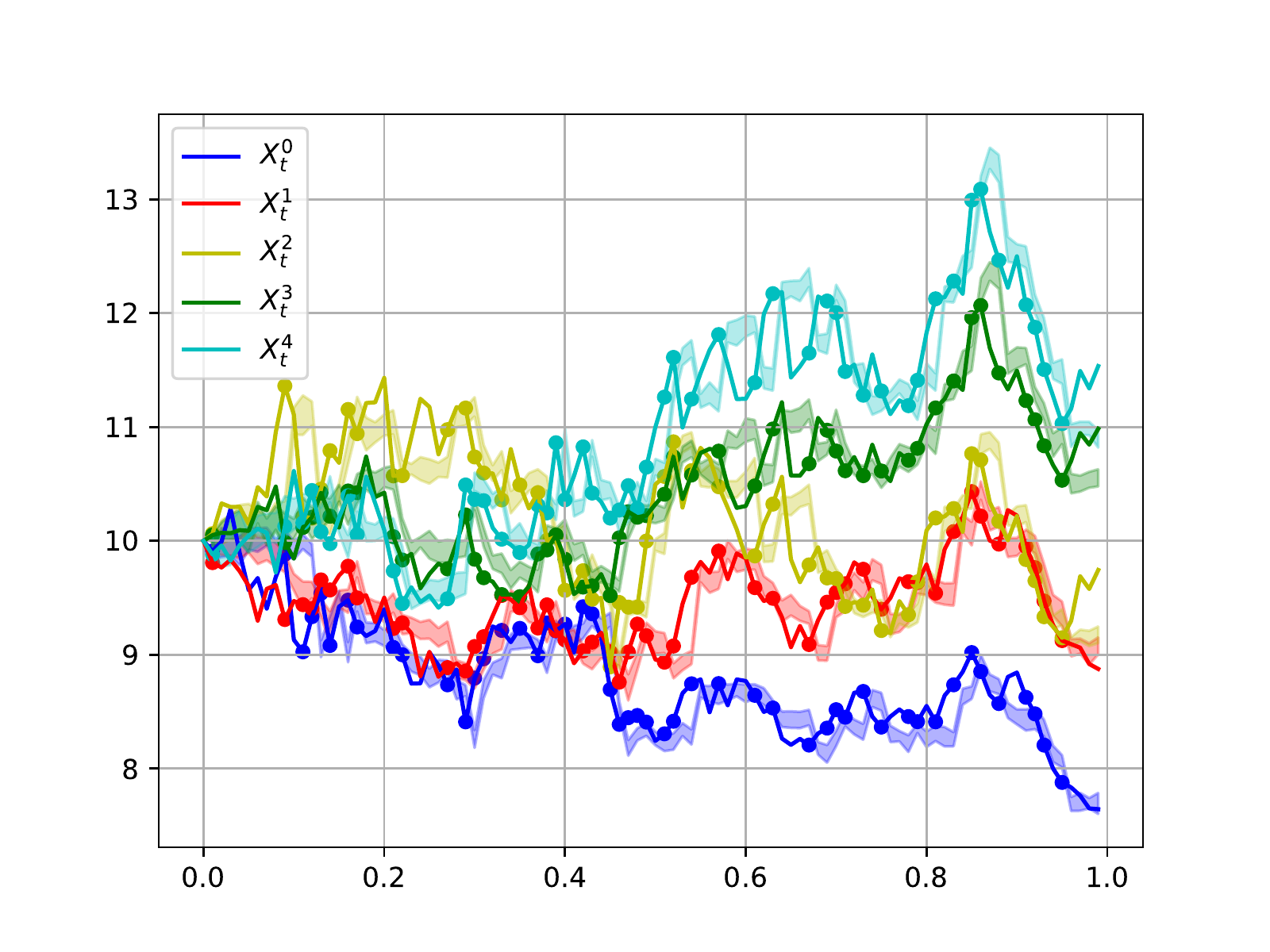}
	}
	\subfigure[Asyn-MTS with $0.4 \le q \le 0.6$]{
		\includegraphics[width=0.23\linewidth]{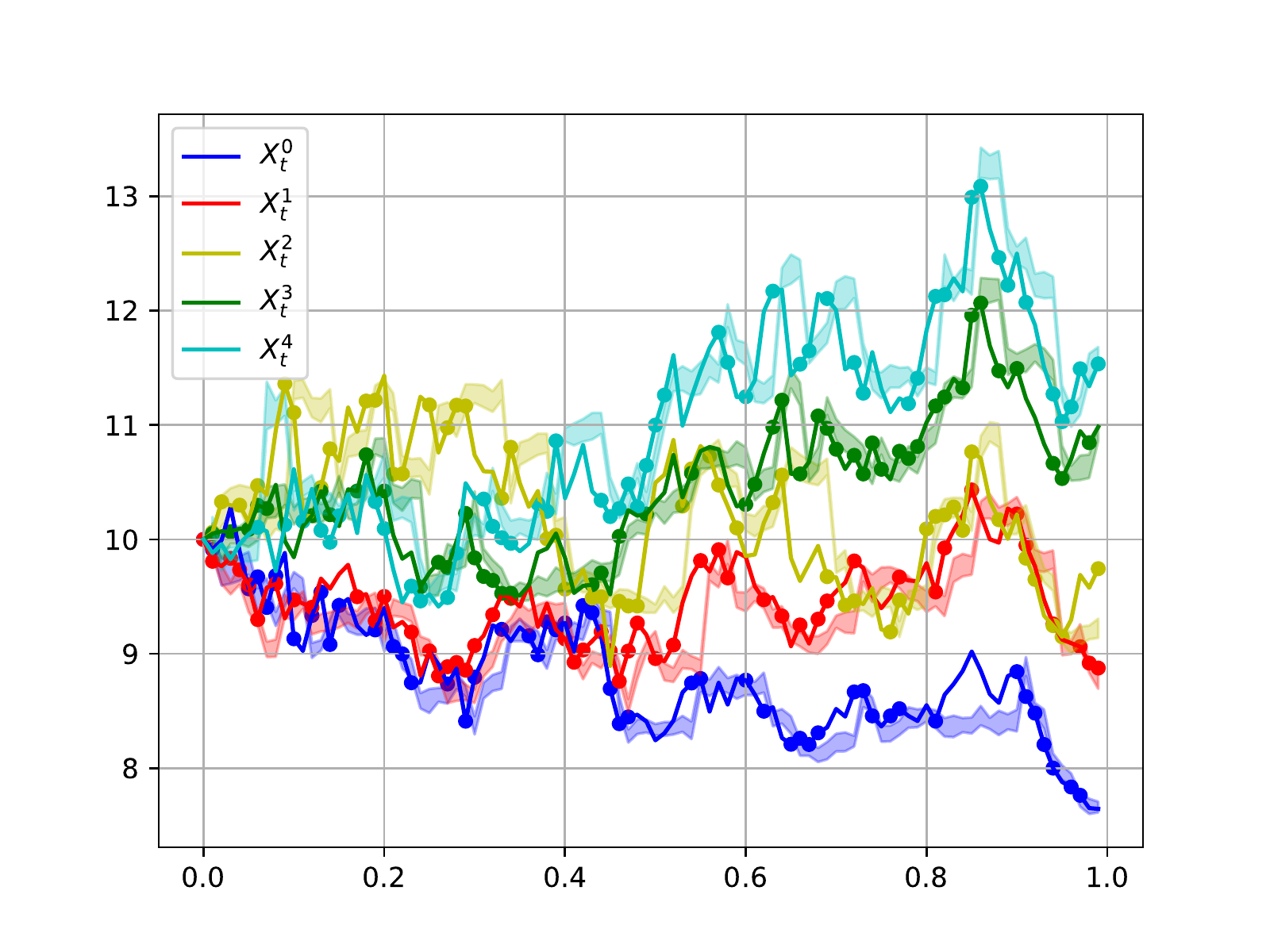}
	}
	
	\caption{Predicted ranges of quantile (shaded areas) of the simulated Syn-MTS and Asyn-MTS at various probability intervals $q$.\label{fig:more res}}
\end{figure*}

	\begin{algorithm}[t]
		\centering
		\caption{Algorithm for training and sampling process of Syn-MTS (example of RFN-GRUODE)}\label{algorithm syn-MTS}
		\small
		\begin{algorithmic}[1]
			\State \textbf{Training (Syn-MTS):} 
			\State \textbf{Input}: Observations: $\{\mathbf{x}_{i}\}_{i=1}^N$; Masks: $\{\mathbf{m}_{i}\}_{i=1}^N$;  Observed time points: $\{\mathrm{t}_i=[t_1,\cdots,t_{K_i}]\}_{i=1}^N$; Flow time interval $[s_1, s_0]$
			\State \textbf{Initialize:} $\text{time}=0$, $\mathrm{h}_0$, and all trainable parameters $\zeta^c, \zeta^u, \psi, \theta$
			\State \textbf{for} $i=1$ \textbf{to} $N$ \textbf{do} \textcolor{blue}{\;\% the index $i$ is omitted in the loop for notation simplicity.}
			\State \indent \textbf{for} $k=1$ \textbf{to} $K_i$ \textbf{do}
			\State \indent \indent $\mathrm{h}(t_{k-})=\text{Continuous Updating}(\mathrm{h}(t_{(k-1)+}), \text{time}, t_k; \zeta^c)$ \textcolor{blue}{\;\% marginal hidden state evolves to $t_k$}
			\State \indent \indent $\mu_{t_k}, \Sigma_{t_k}=\textbf{MLP}(\mathrm{h}_{t_{k-}}; \psi)$ \textcolor{blue}{\;\%  predict and store the parameters of base distribution at $t_k$}
			\State \indent \indent $\text{time}=t_{k}$
			\State \indent \indent $\mathrm{h}(t_{k+})=\text{Discrete Updating}(\mathrm{h}(t_{k-}), \mathrm{x}_{t_k}, \mathrm{m}_{t_k}; \zeta^u)$ 	\textcolor{blue}{\;\% marginal memory updates at $t_k$}
			\State \indent \textbf{end for}	
			\State \indent \textcolor{blue}{\% concatenate observations and parameters at all time points and transform them at the same time by the flow model }
			\State \indent Set $t\in \{t_1, \cdots, t_{K_i}: \mathrm{m}_{t_k}=1, k\in\{1, \cdots, K_i\}\}$ \textcolor{blue}{\;\% all the observed time points}
			\State \indent $\mathrm{z}_t=\mathrm{x}_{t}+\int_{s_1}^{s_0}\mathrm{f}(\mathrm{z}(s), s, \mathrm{h}_{t-}; \theta)ds$ \textcolor{blue}{\;	\% transform observations from data distribution to base distribution (from $s_1$ to $s_0$)}
			\State \indent $\log p(\mathrm{z}_t|\mathrm{h}_{t-})=\log p(\mathrm{z}_t; \mu_{t}, \Sigma_{t})$ 
			\textcolor{blue}{\; \% compute the log-likelihood of transformed observations in base distribution}
			\State \indent$\log p(\mathrm{x}_t|\mathrm{h}_{t-})=\log p(\mathrm{z}_t|\mathrm{h}_{t-})+\int_{s_1}^{s_0}\operatorname{Tr}[\partial_{\mathrm{z}(s)} \mathrm{f}]ds$ \textcolor{blue}{\;	\% compute the log-likelihood of observed data points}

			\State \indent $\mathcal{L}^{i}_{\text{Syn-MTS}}=-\log p(\mathrm{x}_t|\mathrm{h}_{t_{-}})$ 
			\textcolor{blue}{\; \% compute the loss of sample $i$ by equation \eqref{Eq17}}
			\State \textbf{end for}	
			\State $\mathcal{L}_{\text{Syn-MTS}}=\frac{1}{N}\sum_{i=1}^{N} \mathcal{L}^i_{\text{Syn-MTS}}$ \textcolor{blue}{\;\% compute the total loss by averaging the loss of all the samples}
			\State $\zeta^c, \zeta^u, \psi, \theta \leftarrow \arg \min_{\zeta^c, \zeta^u, \psi, \theta} \mathcal{L}_{\text{Syn-MTS}}$ \textcolor{blue}{\;\% optimize the training parameters via stochastic gradient descent algorithm}\\
			\hrulefill
			\State \textbf{Sampling (Syn-MTS):}
			\State {\bfseries Input:} Observations $\mathbf{x}_{i}$; Masks $\mathbf{m}_{i}$;  Observed time points $\mathrm{t}_i=[t_1,\cdots,t_{K_i}]$; Flow time interval $[s_0, s_1]$; Trained model $\mathrm{f}$
			\State {\bfseries Initialize:} $\text{time}=0$, $\mathrm{h}_0$
			\State \textbf{for} $k=1$ \textbf{to} $K_i$ \textbf{do}	
			\State \indent $\mathrm{h}(t_{k-})=\text{Continuous Updating}(\mathrm{h}(t_{(k-1)+}), \text{time}, t_k)$ \textcolor{blue}{\;\% marginal hidden state evolves to $t_k$} 	
			\State \indent {\small $\mu_{t_k}, \Sigma_{t_k}=\textbf{MLP}(\mathrm{h}_{t_{k-}};\psi)$} \textcolor{blue}{\;\%  predict the parameters of base distribution at $t_{k}$}
			\State \indent $\mathrm{z}_{t_k} \sim \mathcal{N}(\mu_{t_k}, \Sigma_{t_k})$ \textcolor{blue}{\;\% sampling samples in predicted base distribution}
			\State \indent $\mathrm{x}_{t_k}=\mathrm{z}_{t_k}+ \int_{s_0}^{s_1}\mathrm{f}(\mathrm{z}(s), s, \mathrm{h}_{t_{k-}};\theta)ds$ 
			\textcolor{blue}{\% transform the samples from base distribution to data distribution (from $s_0$ to $s_1$)}
			\State \indent $\text{time}=t_k$
			\State \indent $\mathrm{h}(t_{k+})=\text{Discrete Updating}(\mathrm{h}(t_{k-}), \mathrm{x}_{t_k}, \mathrm{m}_{t_k})$	\textcolor{blue}{\;\% marginal memory updates at $t_k$}
			\State \textbf{end for}
			\State \textbf{return} $\mathrm{x}_{t}$
		\end{algorithmic}
\end{algorithm}

\begin{algorithm}[tbh]
\centering
\caption{Algorithm for training and sampling process of Asyn-MTS (example of RFN-GRUODE)}\label{algorithm Asyn-MTS}
\small
\begin{algorithmic}[1]
	\State \textbf{Training (Asyn-MTS):} 
	\State \textbf{Input:} Observations: $\{\mathbf{x}_{i}\}_{i=1}^N$; Masks: $\{\mathbf{m}_{i}\}_{i=1}^N$;  Observed time points: $\{\mathrm{t}_i=[t_1,\cdots,t_{K_i}]\}_{i=1}^N$; Flow time interval $[s_1, s_0]$
	\State \textbf{Initialize:} $\text{time}=0$, $\mathrm{h}_0$, and all trainable parameters $\zeta^c, \zeta^u, \psi, \theta, \eta$
	\State \textbf{for} $i=1$ \textbf{to} $N$ \textbf{do}  \textcolor{blue}{\;\% the index $i$ is omitted in the loop for notation simplicity.}
	\State \indent \textbf{for} $k=1$ \textbf{to} $K_i$ \textbf{do}
	\State \indent \indent $\mathrm{h}(t_{k-})=\text{Continuous Updating}(\mathrm{h}(t_{(k-1)+}), \text{time}, t_{k}; \zeta^c)$ \textcolor{blue}{\;\% marginal hidden state evolves to $t_k$}
	\State \indent \indent $\mu_{t_k}^{d}, \Sigma_{t_k}^{d}=\textbf{MLP}(\mathrm{h}_{t_{k-}}; \psi_m)$ \textcolor{blue}{\;\%  predict the parameters of individual base distribution at $t_k$}
	\State \indent \indent $\text{time}=t_{k}$        
	\State \indent \indent $\mathrm{h}(t_{k+})=\text{Discrete Updating}(\mathrm{h}(t_{k-}), \mathrm{x}_{t_k}, \mathrm{m}_{i,t_k}; \zeta^u)$ \textcolor{blue}{\;\% marginal memory updates at $t_k$}
	\State \indent \textbf{end for}
	\State \indent \textcolor{blue}{\% concatenate observations and parameters at all time points and transform them at the same time by the flow model }
	\State \indent Set $t\in \{t_1, \cdots, t_{K_i}: \mathrm{m}_{t_k}=1, k\in\{1, \cdots, K_i\}\}$ \textcolor{blue}{\;\% all the observed time points}
	\State \indent $\mathrm{z}_t=\mathrm{x}_{t}+\int_{s_1}^{s_0}\mathrm{f}(\mathrm{z}(s), s, \mathrm{h}_{t_{k-}}, \mathrm{m}_{t_k}; \theta)ds$ \textcolor{blue}{\% transform data distribution to base distribution by equation \eqref{Eq21}}
	\State \indent $\log p(z_t^d|\mathrm{h}_{t_-})=\log p(z_t^d; \mu_{t}^d, \Sigma_{t}^d)$ \textcolor{blue}{\% compute the log-likelihood of observations of individual variable $d$ in base distribution}
	\State  \indent $\log p(\mathrm{x}_t|\mathrm{h}_{t_-})=\sum_{d=1}^D m_{t_k}^d\log p(z_t^d| \mathrm{h}_{t_-})+\int_{s_1}^{s_0}\operatorname{Tr}[\partial_{z^d(s)} \mathrm{f}]ds$ \textcolor{blue}{\;	\% compute the log-likelihood of observed data points}

	\State 	\indent $\mathcal{L}_{\text{Asyn-MTS}}^{i}=-\log p(\mathrm{x}_t|\mathrm{h}_{t_-})$  \textcolor{blue}{\% only compute the negative log-likelihood loss for the observed variables}

	\State \textbf{end for}
	\State $\mathcal{L}_{\text{Asyn-MTS}}=\frac{1}{N}\sum_{i=1}^{N} \mathcal{L}^i_{\text{Asyn-MTS}}$ \textcolor{blue}{\;\% compute the total loss by averaging the loss of all the samples}
	\State $\zeta^c, \zeta^u, \psi, \theta \leftarrow \arg \min_{\zeta^c, \zeta^u, \psi_m, \theta, \eta} \mathcal{L}_{\text{Asyn-MTS}}$ \textcolor{blue}{\;\% optimize the training parameters via stochastic gradient descent algorithm}\\
	\hrulefill
	\State \textbf{Sampling (Asyn-MTS):}
	\State {\bfseries Input:} Observations: $\mathbf{x}_{i}$; Masks: $\mathbf{m}_i$;  Observed time points: $\mathrm{t}_i=[t_1,\cdots,t_{K_i}]$; Flow time interval $[s_0,s_1]$; Trained model $\mathrm{f}$
	\State {\bfseries Initialize:} $\text{time}=0$, $\mathrm{h}_0$
	\State \textbf{for} $k=1$ \textbf{to} $K_i$ \textbf{do}
	\State \indent $\mathrm{h}(t_{k-})=\text{Continuous Updating}(\mathrm{h}(t_{(k-1)+}), \text{time}, t_k; \zeta^c)$ \textcolor{blue}{\;\% marginal hidden state evolves to $t_k$} 		

	\State  \qquad $\mu_{t_k}^d, \Sigma_{t_k}^d=\textbf{MLP}(\mathrm{h}_{t_{k-}})$ \textcolor{blue}{\;\%  predict the parameters of base distribution for variable $d$ at $t_k$}
	\State  \qquad $z_{t_k}^{d} \sim \mathcal{N}(\mu_{t_k}^{d}, \Sigma_{t_k}^{d})$
	\textcolor{blue}{\;\% sampling samples in predicted base distribution}
	\State  \qquad $\mathrm{x}_{t_k}=\mathrm{z}_{t_k}+\int_{s_0}^{s_1}\mathrm{f}(\mathrm{z}_{t_k}(s), s, \mathrm{h}_{t_{k-}}; \theta)ds$ 
	\textcolor{blue}{\;\% transform the samples from base distributions to data distributions}
	\State \indent$\text{time}=t_k$
	\State \indent $\mathrm{h}(t_{k+})=\text{Discrete Updating}(\mathrm{h}(t_{k-}), \mathrm{x}_{t_k}, \mathrm{m}_{t_k}; \zeta^u)$ \textcolor{blue}{\;\% marginal memory updates at $t_k$}
	\State \textbf{end for}
	\State \textbf{return} $\mathrm{x}_{t}$
\end{algorithmic}
\end{algorithm}



\end{document}